\documentclass[11pt,a4paper]{article}

\newif\ifdraft
\draftfalse

\usepackage{amsmath}
\usepackage{amssymb}
\usepackage{amsthm}
\usepackage{wrapfig}
\usepackage{subcaption}
\usepackage{fullpage}

\usepackage{pst-tree}

\usepackage{indentfirst}

\usepackage{clrscode3e}

\ifdraft
\newcommand{\Znote}[1]{{[\bf Zhenming's Note: #1]}}
\newcommand{\Knote}[1]{{[\bf Kai-Min's Note: #1]}}
\newcommand{\Mnote}[1]{{[\bf Michael's Note: #1]}}
\newcommand{\Hnote}[1]{{[\bf Henry's Note: #1]}}
\else
\newcommand{\Znote}[1]{{}}
\newcommand{\Knote}[1]{{}}
\newcommand{\Mnote}[1]{{}}
\newcommand{\Hnote}[1]{{}}
\fi

\newtheorem{fact}{Fact}[section]

\newtheorem{theorem}{Theorem}[section]

\newtheorem{lemma}[theorem]{Lemma}
\newtheorem{proposition}[theorem]{Proposition}
\newtheorem{corollary}[theorem]{Corollary}
\newtheorem{definition}[theorem]{Definition}

\newenvironment{remark}[1][Remark]{\begin{trivlist}
\item[\hskip \labelsep {\bfseries #1}]}{\end{trivlist}}

\newcommand{\cald}{\mathcal D}

\newcommand{\diag}{\mathrm{diag}}

\newcommand{\cala}{\mathcal{A}}
\newcommand{\calb}{\mathcal{B}}

\newcommand{\calx}{\mathcal{X}}
\newcommand{\calk}{\mathcal{K}}
\newcommand{\calh}{\mathcal{H}}
\newcommand{\calm}{\mathcal{M}}
\newcommand{\calt}{\mathcal{T}}

\newcommand{\calc}{\mathcal{C}}
\newcommand{\proj}{\mathrm{Proj}}

\newcommand{\E}{\mathrm{E}}

\newcommand{\bfx}{\mathbf{X}}
\newcommand{\bfy}{\mathbf{Y}}

\newcommand{\calpa}{\mathcal P_{\mathbf A}}

\newcommand{\calp}{\mathcal P}

\newcommand{\Ball}{\mathbf{Ball}}

\newcommand{\hsn}{\mathrm{HS}}

\providecommand{\ie}{\emph{i.e.,} }
\providecommand{\eg}{\emph{e.g.,} }

\providecommand{\mypara}[1]{\smallskip\noindent\emph{#1}}
\providecommand{\myparab}[1]{\smallskip\noindent\textbf{#1}}

\usepackage{setspace}
\usepackage{xcolor}
\usepackage{xspace}
\usepackage{url} 
\usepackage{nicefrac}

\newcommand{\ednote}[3]{\ifdraft {\color{#2} #1: #3} \fi}

\newcommand{\vknote}[1]{\ednote{VK}{blue}{#1}}
\newcommand{\marginnote}[3]{\ifdraft $^{\textcolor{#2}{\mathbf{\dagger}}}$\marginpar{\setstretch{0.43}\textcolor{#2}{\bf\tiny#1: #3}} \fi}
\newcommand{\vkmnote}[1]{\marginnote{VK}{blue}{#1}}

\newcommand{\ignore}[1]{}
\definecolor{darkgrey}{rgb}{0.5, 0.5, 0.5}
\definecolor{DarkGreen}{rgb}{0, 0.5, 0}

\newenvironment{ignoretext}{\color{darkgrey}}{\ignorespacesafterend}

\usepackage{tikz}
\usetikzlibrary{trees}

\newcommand{\transpose}{\mathsf{T}}

\newcommand{\sfrac}[2]{\small #1 / #2}
\renewcommand\sfrac\nicefrac

\renewcommand{\E}{\mathbb{E}}
\newcommand{\reals}{\mathbb{R}}
\newcommand{\naturals}{\mathbb{N}}

\newcommand{\HS}{\mathsf{HS}}

\newcommand{\X}{\mathcal{X}}

\renewcommand{\hat}{\widehat}

\renewcommand{\th}{^{\textit{\scriptsize{th}}}}

\begin{document}

\title{From which world is your graph?}
\author{
	Cheng Li \\
	College of William \& Mary \\
	\and
	Felix M. F. Wong \\ 
	\and 
	Zhenming Liu \\
	College of William \& Mary \\
	\and 
	Varun Kanade \\
	\small{University of Oxford and The Alan Turing Instituite}
}

\maketitle

\begin{abstract}
Discovering statistical structure from links is a fundamental problem in the
	analysis of social networks. Choosing a misspecified model, or equivalently,
	an incorrect inference algorithm will result in an invalid analysis or even
	falsely uncover patterns that are in fact artifacts of the model. This work
	focuses on unifying two of the most widely used link-formation models: the
	stochastic blockmodel (SBM) and the small world (or latent space) model
	(SWM). Integrating techniques from kernel learning, spectral graph theory,
	and nonlinear dimensionality reduction, we develop the first statistically
	sound polynomial-time algorithm to discover latent patterns in sparse graphs
	for both models. When the network comes from an SBM, the algorithm outputs a
	block structure. When it is from an SWM, the algorithm outputs estimates of
	each node's latent position. 
\end{abstract}

\section{Introduction}\label{sec:intro}

Discovering statistical structures from links is a fundamental problem in
the analysis of social networks. Connections between entities are
typically formed based on underlying feature-based similarities; however
these features themselves are partially or entirely hidden. A question of
great interest is to what extent can these latent features be inferred
from the observable links in the network. This work focuses on the
so-called assortative setting, the principle that \emph{similar
individuals are more likely to interact with each other}. Most stochastic
models of social networks rely on this assumption,
including the two most famous ones -- the stochastic
blockmodel~\cite{HLL:1983} and the small-world
model~\cite{WS:1998,K:2000}, described below.

\myparab{Stochastic Blockmodel (SBM)}.~In a stochastic
blockmodel~\cite{YunP16,MNS:2015,AS:pp15,AS:FOCS2015,Mas:STOC2014,MNS:pp2013,BC:PNAS2009,LLDM:WWW2008,NG:PRE2004,NWS:PNAS2002},
nodes are grouped into disjoint ``communities'' and links are added
randomly between nodes, with a higher probability if nodes are in the
same community. In its simplest incarnation, an edge is added between
nodes within the same community with probability $p$, and between nodes
in different communities with probability $q$, for $p > q$. 
Despite arguably na\"{i}ve modelling choices, such as the independence of
edges, algorithms designed with SBM
work well in practice~\cite{McS:FOCS2001,LLM:WWW2010}.

\myparab{Small-World Model (SWM)}.~In a small-world model, each node is
associated with a latent variable $x_i$, \eg the geographic location of
an individual. The probability that there is a link between two nodes
is proportional to an inverse polynomial of some notion of distance,
$\mathrm{dist}(x_i, x_j)$, between them. 
The presence of a small number of
``long-range'' connections is essential to some of the most intriguing
properties of these networks, such as small diameter and fast
decentralized routing algorithms~\cite{K:2000}. In general, the latent
position may reflect geographic location as well as more abstract
concepts, \eg position on a political ideology spectrum.%
\vkmnote{It'd be interesting to see if there is any empirical study on twitter showing what fraction of the people you follow are actually your ``friends'', and what fraction are celebrities you admire, politically or otherwise. It would also be good to cite some standard political ideology literature.}

\myparab{The Inference Problem}.~
Without observing the latent positions, or knowing which model generates
the underlying graph, the adjacency matrix of a social graph typically
looks like the one shown in Fig.~\ref{fig:adjmats}(a)
(App.~\ref{app:introstuff}). However, if the model generating the graph
is known, it is then possible to run a suitable ``clustering
algorithm''~\cite{McS:FOCS2001,AbrahamCKS13} that reveals the hidden
structure. When the vertices are ordered suitably, the SBM's adjacency
matrix looks like the one shown in Fig.~\ref{fig:adjmats}(b)
(App.~\ref{app:introstuff}) and that of the SWM looks like the one shown
in Fig.~\ref{fig:adjmats}(c) (App.~\ref{app:introstuff}). Existing
algorithms typically depend on knowing the ``true'' model and are
tailored to graphs generated according to one of these models,
\eg~\cite{McS:FOCS2001,AbrahamCKS13,Barbera12,BJN15}.

\myparab{Our Contributions}.~We consider a latent space model that is
general enough to include both these models as special cases. In our
model, an edge is added between two nodes with a probability that is a
decreasing function of the distance between their latent positions. This
model is a fairly natural one, and it is quite likely that a variant has
already been studied; however, to the best of our knowledge there is no
known statistically sound and computationally efficient algorithm for
latent-position inference on a model as general as the one we consider. 

\mypara{1. A unified model.} We propose a model that is a natural
generalization of both the stochastic blockmodel and the small-world
model that captures some of the key properties of real-world social
networks, such as small out-degrees for ordinary users and large
in-degrees for celebrities. We focus on a simplified model where we
have a modest degree graph only on ``celebrities''; the supplementary
material contains an analysis of the more realistic model using somewhat
technical machinery.

\mypara{2. A provable algorithm.} We present statistically sound and
polynomial-time algorithms for inferring latent positions in our
model(s). Our algorithm approximately infers the latent positions of
almost all ``celebrities'' ($1-o(1)$-fraction), and approximately infers
a constant fraction of the latent positions of ordinary users. We show
that it is statistically impossible to err on at most $o(1)$ fraction of
ordinary users by using standard lower bound arguments.

\mypara{3. Proof-of-concept experiments.} We report several experiments
on synthetic and real-world data collected on Twitter from Oct 1 and Nov
30, 2016. Our experiments demonstrate that our model and inference
algorithms perform well on real-world data and reveal interesting
structures in networks.

\myparab{Additional Related Work}. We briefly review the relevant
published literature.  \emph{1. Graphon-based techniques}. Studies using
graphons to model networks have focused on the statistical properties of
the
estimators~\cite{Hoff01latentspace,Airoldi2008,Rohe11,ACC13,WolfeO13,tang2013,WolfeChoi14,KanadeMS16,rohe2016co},
with limited attention paid to computational efficiency. The
``USVT'' technique developed recently~\cite{chatterjee2015} estimates the
kernel well when the graph is dense. Xu et~al.~\cite{xu14} consider a
polynomial time algorithm for a sparse model similar to ours, but focus
on edge classification rather than latent position estimation. \emph{2.
Correspondence analysis in political science}. Estimating the ideology
scores of politicians is an important research topic in political
science~\cite{poole85,laver03,clinton04,gerrish12,gerrish2,grimmer13,Barbera12,BJN15}.
High accuracy heuristics developed to analyze dense graphs
include~\cite{Barbera12,BJN15}.

\myparab{Organization}. Section~\ref{sec:prelim} describes background,
our model and results. Section~\ref{sec:ouralgo} describes our algorithm
and an gives an overview of its analysis. Section~\ref{sec:shortexp}
contains the experiments. 

\section{Preliminaries and Summary of Results}\label{sec:prelim}

\myparab{Basic Notation}. We use $c_0, c_1$,~etc. to denote constants which may
be different in each case. We use whp to denote with high probability, by which
we mean with probability larger $1 - \frac{1}{n^c}$ for any $c$. All notation is summarized in Appendix~\ref{app:not} for quick reference.

%
\myparab{Stochastic Blockmodel}. Let $n$ be the number of nodes in the graph
with each node assigned a label from the set $\{1, \ldots, k \}$ uniformly at
random. An edge is added between two nodes with the same label with probability
$p$ and between the nodes with different labels with probability $q$, with $p >
q$ (assortative case). In this work, we focus on the $k = 2$ case, where $p, q
= \Omega\left((\log n)^c/n\right)$ and the community sizes are exactly the
same. (Many studies of the regimes where recovery is possible have been
published~\cite{Abbe15,MNS:pp2013,MNS:2015,Mas:STOC2014}.) 

Let $A$ be the adjacency matrix of the realized graph and let $M = \E[A] =
{\scriptsize\left(\begin{array}{cc} P & Q \\ Q & P\end{array}\right)}$, where
$P \mbox{ and } Q \in R^{\frac n 2 \times \frac n 2}$ with every entry equal
to $p$ and $q$, respectively. 
We next explain the inference algorithm, which uses two key observations.
\textit{1. Spectral Properties of $M$}. $M$ has rank $2$ and the non-trivial
eigenvectors are $(1, \ldots, 1)^\transpose$ and $(1, \ldots, 1, -1, \ldots,
-1)$ corresponding to eigenvalues $n (p + q)/2$ and $n(p - q)/2$, respectively.
If one has access to $M$, the hidden structure in the graph is revealed
merely by reading off the second eigenvector. \textit{2.  Low Discrepancy
between $A$ and $M$}.  Provided the average degree $n (p + q)/2$ and the gap $p
- q$ are large enough, the spectrum and eigenspaces of the matrices $A$ and $M$
can be shown to be close using matrix concentration inequalities and the
Davis-Kahan theorem~\cite{Tropp:FOCM2012,DK:1970}. Thus, it is sufficient to
look at the projection of the columns of $A$ onto the top two eigenvectors of
$A$ to identify the hidden latent structure.

\myparab{Small-World Model (SWM)}.~In a 1-dim. SWM, each node $v_i$ is
associated with an independent latent variable $x_i \in [0, 1]$ that is drawn
from the uniform distribution on $[0, 1]$.  The probability of a link between
two nodes is $\Pr[\{v_i, v_j\} \in E] \propto \frac{1}{|x_i - x_j|^{\Delta} +
c_0}$, where $\Delta > 1$ is a hyper-parameter.

The inference algorithm for small-world models uses different ideas.  Each edge
in the graph is considered as either ``short-range'' or ``long-range.''
Short-range edges are those between nodes that are nearby in latent space,
while long-range edges have end-points that are far away in latent space. After
removing the long-range edges, the shortest path distance between two nodes
scales proportionally to the corresponding latent space distance (see
Fig.~\ref{fig:swm} in App.~\ref{app:swintro}). After obtaining estimates for
pairwise distances, standard buidling blocks are used to find the latent
positions $x_i$~\cite{IM:HDCG-2004}.%
\vkmnote{This sentence should be checked for accuracy and citations added.}
The key observation used to remove the long-range edges is: 
an edge $\{v_i, v_j\}$ is a short-range edge if and only if $v_i$ and $v_j$ will
share many neighbors. 

\myparab{A Unified Model}. Both SBM and SWM are special cases of our unified
latent space model. We begin by describing the full-fledged bipartite
(heterogeneous) model that is a better approximation of real-world networks,
but requires sophisticated algorithmic techniques (see
Appendix~\ref{sec:phiest} for a detailed analysis). Next, we present a
simplified (homogeneous) model to explain the key ideas.

\mypara{Bipartite Model}. We use graphon model to characterize the
stochastic interactions between users. Each individual is associated with
a latent variable in $[0, 1]$. The bipartite graph model consists of two
types of users: the left side of the graph $\bfy = \{y_1, \ldots, y_m\}$
are the \emph{followers} (ordinary users) and the right  side $\bfx =
\{x_1, \ldots,  x_n\}$ are the \emph{influencers} (celebrities). Both
$y_i$ and $x_i$ are i.i.d. random variables from a distribution $\cald$.
This assumption follows the convention of existing heterogeneous
models~\cite{zhao2012,QR13NIPS}.  The probability that two individuals
$y_i$ and $x_j$ interact is $\kappa(y_i, x_j)/n$, where $\kappa: [0, 1]
\times [0, 1] \rightarrow (0, 1]$ is a kernel function (these are
sometimes referred to as graphon-based
models~\cite{lovasz2012large,WolfeO13,ACC13}). Throughout this paper we
assume that $\kappa$ is a small-world kernel, \ie $\kappa(x, y) =
c_0/(\|x - y\|^{\Delta} + c_1)$ for some $\Delta > 1$ and suitable
constants $c_0, c_1$, and that $m = \Theta(n \cdot \mathrm{polylog}(n))$.
Let $B \in R^{m \times n}$ be a binary matrix that $B_{i, j} = 1$ if and
only if there is an edge between $y_i$ and $x_j$. Our goal is to estimate
$\{x_i\}_{i \in [n]}$ based on $B$ for suitably large $n$.

\mypara{Simplified Model}. The graph only has the node set is $\bfx =
\{x_1, ..., x_n\}$ of celebrity users. Each $x_i$ is again an i.i.d.
random variable from $\cald$. The probability that two users $v_i$ and
$v_j$ interact is $\kappa(x_i, x_j)/C(n)$. The denominator is a
normalization term that controls the edge density of the graph. We assume
$C(n) = n/\mathrm{polylog}(n)$, \ie the average degree is
$\mathrm{polylog}(n)$.  Unlike the SWM where the $x_i$ are drawn
uniformly from $[0, 1]$, in the unified model $\cald$ can be flexible.
When $\cald$ is the uniform distribution, the model is the standard SWM.
When $\cald$ has discrete support (\eg $x_i = 0$ with prob. 1/2 and $x_i
= 1$ otherwise), then the unified model reduces to the SBM. Our
distribution-agnostic algorithm can automatically select the most
suitable model from SBM and SWM, and infer the latent positions of
(almost) all the nodes.

\myparab{Bipartite vs. Simplified Model}. The simplified model suffers
from the following problem: If the average degree is $O(1)$, then we err
on estimating every individual's latent position with a constant
probability (\eg whp the graph is disconnected), but in practice we
usually want a high prediction accuracy on the subset of nodes
corresponding to high-profile users. Assuming that the average degree is
$\omega(1)$ mismatches empirical social network data. Therefore, we use a
bipartite model that introduces heterogeneity among nodes: By splitting
the nodes into two classes, we achieve high estimation accuracy on the
influencers and the degree distribution more closely matches real-world
data. For example, in most online social networks, nodes have $O(1)$
average degree, and a small fraction of users (influencers) account for
the production of almost all ``trendy'' content while most users
(followers) simply consume the content. 

\myparab{Additional Remarks on the Bipartite Model}. \emph{1.
Algorithmic contribution}. Our algorithm  computes $B^{\transpose}B$ and
then regularizes the product by shrinking the diagonal entries before
carrying out spectral analysis. Previous studies of the bipartite
graph in similar settings~\cite{Dhillon2001,2007PhRvEZhou,WongTSC16}
attempt to construct a regularized product using different heuristics.
Our work presents the first theoretically sound regularization technique
for spectral algorithms. In addition, some studies have suggested
running SVD on $B$ directly (\eg \cite{rohe2016co}). We show that the
(right) singular vectors of $B$ \emph{do not} converge to the
eigenvectors of $K$ (the matrix with entries $\kappa(x_i, x_j)$).  Thus,
it is necessary to take the product and use regularization.  \emph{2.
Comparison to degree-corrected models (DCM).} In DCM, each node $v_i$ is
associated with a degree parameter $D(v_i)$. Then we have $\Pr[\{v_i,
v_j\} \in \E ] \propto D(v_i)\kappa(x_i, x_j)D(v_j)$. The DCM model
implies the subgraph induced by the highest degree nodes is dense, which
is inconsistent with real-world networks. There is a need for better
tools to analyze the asymptotic behavior of such models and we leave this
for future work (see, \eg \cite{zhao2012,QR13NIPS}).

\myparab{Theoretical Results}. 
Let $F$ be the cdf of $\cald$. We say $F$ and $\kappa$ are
\textbf{well-conditioned} if: \smallskip \\ 
\emph{(1) $F$ has finitely many points of discontinuity}, \ie the
closure of the support of $F$ can be expressed as the union of
non-overlapping closed intervals $I_1$, $I_2$, ..., $I_k$ for a finite
number $k$. \smallskip  \\
\emph{(2) $F$ is near-uniform}, \ie for any interval $I$ that has
non-empty overlap with $F$'s support, 
$\int_I dF(x) \geq c_0 |I|$, for some constant $c_0$.  \smallskip \\
\emph{(3) Decay Condition: The eigenvalues of the integral operator based
on $\kappa$ and $F$ decay sufficiently fast.} We define the $\calk f(x) =
\int \kappa(x, x^\prime) f(x^\prime) dF(x^\prime)$ and let
$(\lambda_i)_{i \geq 1}$ denote the eigenvalues of $\calk$. Then, it
holds that $\lambda_i = O(i^{-2.5})$.

If we use the small-word kernel $\kappa(x, y) = c_0/ ( |x - y|^{\Delta} + c_1)$
and choose $F$ that give rise to SBM or SWM, in each case the pair $F$ and
$\kappa$ are well-conditioned, as described below. As the decay condition is
slightly more invoved, we comment upon it. The condition is a mild one. When
$F$ is uniformly distributed on $[0, 1]$, it is equivalent to requiring $\calk$
to be twice differentiable, which is true for the small world kernel
(Theorem~\ref{thm:weyl}). When $F$ has a finite discrete support, there are
only finitely many non-zero eigenvalues, \ie this condition also holds.  The
decay condition holds in more general settings, \eg when $F$ is piecewise
linear~\cite{konig1986eigenvalue} (see App.~\ref{asec:existingresult}).
Without the decay condition, we would require much stronger assumptions: Either
the graph is very dense or $\Delta \gg 2$. Neither of these assumptions is
realistic, so effectively our algorithm fails to work. In practice, whether the
decay condition is satisfied can be checked by making a log-log plot and it has
been observed for several real-world networks, the eigenvalues follow a
power-law distribution~\cite{MP:RANDOM2002}.

Next, we define the notion of latent position recovery for our algorithms.

\begin{definition}[$(\alpha, \beta, \gamma)$-Aproximation
	Algorithm]\label{def:quality} 
	Let $I_i$, $F$, and $\calk$ be defined as above, and let $R_i = \{x_j : x_j
	\in I_i\}$. An algorithm is called an $(\alpha, \beta,
	\gamma)$-approximation algorithm if  \\ 
	\noindent{1.} It outputs a collection of disjoint points $C_1, C_2, ...,
	C_k$ such that $C_i \subseteq R_i$, which correspond to subsets of
	reconstructed latent variables. \\
	\noindent{2.} For each $C_i$, it produces a distance matrix $D^{(i)}$.
	Let $G_i \subseteq C_i$ be such that for any $i_j, i_k \in G_i$
	\begin{equation}\label{eqn:approx}
		D^{(i)}_{i_j, i_k} \leq |x_{i_j} - x_{i_k}| \leq (1+\beta)D^{(i)}_{i_j,
		i_k} + \gamma.
	\end{equation}
	\noindent{3.} $\left|\bigcup_iG_i\right| \geq (1-\alpha)n$. \\ 
	In bipartite graphs, Eq.\eqref{eqn:approx} is required for only influencers.
\end{definition}

We do not attempt to optimize constants in this paper. We set $\alpha = o(1)$,
$\beta$ a small constant, and $\gamma = o(1)$.  Definition~\ref{def:quality}
allows two types of errors: $C_i$s are not required to form a partition \ie
some nodes can be left out, and a small fraction of estimation errors is
allowed in $C_i$, \eg if $x_j = 0.9$ but $\hat x_j = 0.2$, then the $j$-th
``row'' in $D^{(i)}$ is  incorrect.  
To interpret the definition, consider the blockmodel with $2$ communities.
Condition 1 means that our algorithm will output two disjoint groups of points.
Each group corresponds to one block. Condition 2 means that there are pairwise
distance estimates within each group. Since the true distances for nodes within
the same block are zero,  our estimates must also be zero to satisfy
 Eq.\ref{eqn:approx}.  Condition 3 says that the portion of misclassified nodes
is $\alpha = o(1)$. We can also interpret the definition when we consider a
small-world graph, in which case $k = 1$. The algorithm outputs pairwise
distances for a subset $C_1$. We know that there is a sufficiently large $G_1
\subseteq C_1$ such that the pairwise distances are all correct in $C_1$. 

Our algorithm does not attempt to estimate the distance between $C_i$ and $C_j$
for $i \neq j$. When the support contains multiple
disjoint intervals, \eg in the SBM case, it first pulls
apart the nodes in different communities. Estimating the distance between intervals, given the
output of our algorithm is straightforward. Our main result is the following.


\begin{theorem}\label{thm:main} 
	Using the notation above, assume $F$ and $\kappa$ are well-conditioned, and
	$C(n)$ and  $m/n$ are  $\Omega(\log^cn)$ for some suitably large $c$. The
	algorithm for the simplified model shown in Figure~\ref{fig:latentalgo} and
	that for the bipartite model shown in Figure~\ref{fig:fullalgo} give us an
	$(1/\log^2n, \epsilon, O(1/\log n))$-approximation algorithm w.h.p. for any
	constant $\epsilon$. Furthermore, the distance estimates $D^{(i)}$ for each
	$C_i$ are constructed using the shortest path distance of an unweighted
	graph. 
\end{theorem}
We focus only on the simplified model and the analysis for the bipartite graph
algorithm is in Appendix~\ref{sec:phiest}.


\myparab{Pairwise Estimation to Line-embedding and High-dimensional
Generalization.} Our algorithm builds estimates on pairwise latent distance and
uses well-studied metric-embedding methods~\cite{BadoiuCIS05,BorgGroenen2005}
as blackboxes to infer latent positions. Our inference algorithm can be
generalized to $d$-dimensional space with $d$ being a constant. But the
metric-embedding on $\ell^d_p$ becomes increasingly difficult, \eg when $d =
2$, the approximation ratio for embedding a graph is
$\Omega(\sqrt{n})$~\cite{Indyk04low-distortionembeddings}.

\section{Our algorithms}\label{sec:ouralgo}

As previous noted, SBM and SWM are special cases of our unified model and both
require different algorithmic techniques. Given that it is not surprising that
our algorithm blends ingredients from both sets of techniques. Before
proceeding, we review basics of kernel learning. 

\myparab{Notations}. 
Let $A$ be the adjacency matrix of the observed graph (simplified model) and
let $\rho(n) \triangleq n/C(n)$. Let $K$ be the matrix with entries
$\kappa(x_i, x_j)$. Let $\tilde{U}_{K}\tilde{S}_{K} \tilde{V}^\transpose_{K}$
($\tilde{U}_{A} \tilde{S}_{A} \tilde{V}^\transpose_{A}$) be the SVD of $K$
($A$). Let $d$ be a parameter to be chosen later.  Let $S_K$ ($S_A$) be a $d
\times d$ diagonal matrix comprising the $d$-largest eigenvalues of $K$ ($A$).
Let $U_K$ ($U_A$) and $V_K$ ($V_A$) be the corresponding singular vectors of
$K$ ($A$). Finally, let $\bar K = U_KS_KV^\transpose_K$ ($\bar A =
U_AS_AV^\transpose_A$) be the low-rank approximation of $K$ ($A$). Note that
when a matrix is positive definite and symmetric SVD coincides with
eigen-decomposition; as a consequence $U_K = V_K$ and $U_A = V_A$.

\myparab{Kernel Learning.} Define an integral operator $\calk$ as $\calk f(x) =
\int \kappa(x, x') f(x') dF(x')$.  Let $\psi_1, \psi_2, \ldots$ be the
eigenfunctions of $\calk$ and $\lambda_1, \lambda_2, \ldots$ be the
corresponding eigenvalues such that $\lambda_1 \geq \lambda_2 \geq \cdots$ and
$\lambda_i \geq 0$ for each $i$. Also let $N_{\calh}$ be the number of
eigenfunctions/eigenvalues of $\calk$, which is either finite or countably
infinite. We recall some important properties of
$\calk$~\cite{Scholkopf2001,tang2013}. For $x \in [0, 1]$,  define the feature
map $\Phi(x) = (\sqrt{\lambda_j}\psi_j(x): j = 1, 2, ...)$, so that $\langle
\Phi(x), \Phi(x^\prime) \rangle = \kappa(x, x^\prime)$. We also consider a
truncated feature $\Phi_d(x) = (\sqrt{\lambda_j}\psi_j(x): j = 1, 2, ..., d)$.
Intuitively, if $\lambda_j$ is too small for sufficiently large $j$, then the
first $d$ coordinates (\ie $\Phi_d$) already approximate the feature map well.
Finally, let $\Phi_d(\bfx) \in \reals^{n \times d}$ such that its $(i,j)$-th entry
is $\sqrt{\lambda_j}\psi_j(x_i)$. Let's further write $(\Phi_d(\bfx))_{:, i}$
be the $i$-th column of $\Phi_d(\bfx)$. Let  $\Phi(\bfx) = \lim_{d \rightarrow
\infty}\Phi_d(\bfx)$.  When the context is clear,  shorten $\Phi_d(\bfx)$ and
$\Phi(\bfx)$ to $\Phi_d$ and $\Phi$, respectively.

There are two main steps in our algorithm which we explain in the following two
subsections.

\begin{figure}
\hspace{-.5cm}
\begin{subfigure}[t]{0.5\textwidth}
{\footnotesize
\begin{codebox}
\Procname{$\proc{Latent-Inference}(A)$}
\li \Comment \textbf{Step 1. Estimate $\Phi$ }.
\li $\hat \Phi \gets \proc{SM-Est}(A)$.
\li \Comment \textbf{Step 2. Execute isomap algo.}
\li $D \gets \proc{Isomap-Algo}(\hat \Phi)$
\li \Comment \textbf{Step 3. Find latent variables.}
\li Run a line embedding algorithm~\cite{BadoiuCIS05,BorgGroenen2005}.
\end{codebox}

\begin{codebox}
\Procname{$\proc{Isomap-Algo}(\hat \Phi)$}
\li Execute $S \leftarrow \proc{Denoise}(\hat \Phi)$ (See Section~\ref{sec:isomap_shorten}) 
\li \Comment $S$ is a subset of $[n]$. 
\li Build $G = \{S, E\}$ s.t. $\{i, j\} \in E$ iff 
\li \quad \quad $|(\tilde \Phi_d)_i - (\tilde \Phi_d)_j| \leq \ell/\log n$ ($\ell$ a constant). 
\li Compute $D$ such $D(i, j)$ is the shortest 
\li \quad \quad path distance between $i$ and $j$ when $i, j \in S$. 
\li \Return D
\end{codebox}

}
\end{subfigure}
\begin{subfigure}[t]{0.5\textwidth}
{\footnotesize
\begin{codebox}
\Procname{$\proc{SM-Est}(A)$}
\li $[\tilde U_A, \tilde S_A, \tilde V_A] = \mathrm{svd}(A)$. 
\li Let also $\lambda_i$ be $i$-th singular  value of $A$.
\li \Comment let $t$ be a suitable parameter. 
\li $d \gets \proc{DecideThreshold}(t, \rho(n))$. 
\li $S_A$: diagonal matrix comprised of $\{\lambda_i\}_{i \leq d}$
\li $U_A$, $V_A$: the singular vectors
\li \quad corresponding to $S_A$. 
\li Let $\hat \Phi = \sqrt{C(n)} U_AS_A^{1/2}$. 
\li \Return $\hat \Phi$
\end{codebox}

\begin{codebox}
\Procname{$\proc{DecideThreshold}(t, \rho(n))$}
\li \Comment This procedure decides $d$ the number 
\li \quad of Eigenvectors to keep. 
\li \Comment $t$ is a tunable parameter. See Proposition~\ref{prop:latentestimate}.
\li $d = \arg\max_d \{\lambda_d(\frac A{\rho(n)}) - \lambda_{d + 1}(\frac A{\rho(n)}) \geq 10(\frac{t}{\rho(n)})^{\frac 2{29}}\}$.
\end{codebox}

}
\end{subfigure}\caption{Subroutines of our Latent Inference Algorithm.}
\label{fig:latentalgo}
\end{figure}

\subsection{Estimation of $\Phi$ through $K$ and $A$}\label{sec:estphi_short}

The mapping $\Phi:[0, 1] \rightarrow \reals^{N_{\calh}}$ is bijective so a
(reasonably) accurate estimate of $\Phi(x_i)$ can be used to recover $x_i$. Our
main result is the design of a data-driven procedure to choose a suitable
number of eigenvectors and eigenvalues of $A$ to approximate $\Phi$ (see
$\proc{SM-Est}(A)$ in Fig.~\ref{fig:latentalgo}).

\begin{proposition}\label{prop:latentestimate} 
	Let $t$ be a tunable parameter such that $t = o(\rho(n))$ and $t^2/\rho(n) =
	\omega(\log n)$. Let $d$ be chosen by $\proc{DecideThreshold}(\cdot)$. Let
	$\hat \Phi \in \reals^{N_{\calh}}$ be such that its first $d$-coordinates
	are equal to $\sqrt{C(n)}U_AS^{1/2}_A$, and its remaining entries are 0. If
	$\rho(n) = \omega(\log n)$ and $\calk$ ($F$ and $\kappa$) is
	well-conditioned, then with high probability:
	\begin{equation}\label{eqn:latentestimate}
		\| \hat \Phi - \Phi\|_F= O\left(\sqrt n \left( t /(\rho(n))\right)^{\frac 2{29}}\right)
		\end{equation}
\end{proposition}
Specifically, by letting $t = \rho^{2/3}(n)$, we have 
$\| \hat \Phi - \Phi \|_F = O\left(\sqrt n\rho^{-2/87}(n)\right).$

\myparab{Remark on the Eigengap.} In our analysis, there are three groups of eigenvalues: the eigenvalues of $\calk$, those of $K$, and those of $A$. They are in different scales: $\lambda_i(\calk) \leq 1$ (resulting from the fact that $\kappa(x, y) \leq 1$ for all $x$ and $y$), and $\lambda_i(A/\rho(n)) \approx \lambda_i(K/n) \approx \lambda_i(\calk)$ if $n$ and $\rho(n)$ are sufficiently large. 
Thus, $\lambda_d(\calk)$ are \emph{independent of} $n$ for a \emph{fixed $d$} and should be treated as $\Theta(1)$. Also $\delta_d \triangleq \lambda_d(\calk) - \lambda_{d + 1}(\calk) \rightarrow 0$ as $d \rightarrow \infty$. Since \emph{the procedure of choosing $d$ depends on $C(n)$} (and thus also on $n$), \emph{$\delta_d$ depends on $n$ and can be bounded by a function in $n$. This is the reason why Proposition~\ref{prop:latentestimate} does not explicitly depend on the eigengap}. We also note that we cannot directly find $\delta_d$ based on the input matrix $A$. But standard interlacing results can give  $\delta_d = \Theta(\lambda_d(A/\rho(n)) - \lambda_{d + 1}(A/\rho(n)))$ (see Lemma~\ref{alem:gap} in Appendix.)



\myparab{Intuition of the algorithm.} Using Mercer's theorem, we have $\langle
\Phi(x_i), \Phi(x_j) \rangle = \lim_{d \rightarrow \infty} \langle \Phi_d(x_i),
\Phi_d(x_j) \rangle = \kappa(x_i, x_j)$. Thus, $\lim_{d \rightarrow
\infty}\Phi_d \Phi^{\transpose}_d = K$. On the other hand, we have
$(\tilde{U}_{K}\tilde{S}^{1/2}_{K})(\tilde{U}_{K}\tilde{S}^{1/2}_{K})^\transpose
= K$. Thus, $\Phi_d(\mathbf X)$  and $\tilde{U}_{K}\tilde{S}^{1/2}_{K}$ are
approximately the same, up to a unitary transformation. We need to identify
sources of errors to understand the approximation quality. 


\myparab{Error source 1} \emph{Finite samples to learn the kernel}. We want to
infer about ``continuous objects'' $\kappa$ and $\cald$ (specifically the
eigenfunctions of $\calk$) but $K$ gives only the kernel values of a finite set
of pairs.
From standard results in Kernel PCA~\cite{Rosasco2010,tang2013}, we have with
probability $\geq 1 - \epsilon$,
\[ \| U_K S^{\sfrac{1}{2}}_K W - \Phi_d(X) \|_F \leq 2\sqrt 2 \frac{\sqrt{\log
\epsilon^{-1}}}{\lambda_d(\calk) - \lambda_{d + 1}(\calk)} = 2\sqrt 2
\frac{\sqrt{\log \epsilon^{-1}}}{\delta_d}. \]

\myparab{Error source 2}. \emph{Only observe $A$}. We observe only the realized
graph $A$ and not $K$, though it holds that $\E A = K/C(n)$. Thus, we can only
use singular vectors of $C(n)A$ to approximate
$\tilde{U}_{K}\tilde{S}^{1/2}_{K}$.  We have: $\left\|\sqrt{C(n)}U_AS^{1/2}_AW
- U_KS^{1/2}_K\right\|_F = O\left(\frac{t
\sqrt{dn}}{\delta^2_d\rho(n)}\right)$. When $A$ is dense (\ie $C(n) = O(1)$),
the problem is analyzed in~\cite{tang2013}. We generalize the results
in~\cite{tang2013} for the sparse graph case. See Appendix~\ref{a:undirected}
for a complete analysis. 


\myparab{Error source 3}. \emph{Truncation error}. When $i$ is large, the noise
in $\lambda_i(A)(\tilde U_A)_{:, i}$ ``outweighs'' the signal. Thus, we need to
decide a $d$ such that only the first $d$ eigenvectors/eigenvalues of $A$ are
used to approximate $\Phi_d$. Here, we need to address \emph{the truncation
error}: the tail $\{\sqrt{\lambda_i}\psi_i(x_j)\}_{i > d}$ is thrown away. 

Next we analyze the magitude of the tail. We abuse notation so that $\Phi_d(x)$
refers to both a $d$-dimensional vector and a $N_{\calh}$-dimensional vector in
which all the entries after $d$-th one are $0$. We have $\E\|\Phi(x) -
\Phi_d(x)\|^2 = \sum_{i > d}\E[(\sqrt{\lambda_i}\psi_i(x))^2] = \sum_{i >
d}\lambda_i \int |\psi_i(x)|^2 dF(x) = \sum_{i > d}\lambda_i.$ (A Chernoff
bound is used to obtain that $\|\Phi - \Phi_d\|_F = O(\sqrt{n}/(\sqrt{\sum_{i >
d}\lambda_i)})$). Using the decay condition, we show that a $d$ can be
identified so that the tail can be bounded by a polynomial in $\delta_d$. The
details are technical and are provided in the supplementary material (cf. Proof
of Prop.~\ref{prop:latentestimate} in Appendix~\ref{a:undirected}).

\subsection{Estimating Pairwise Distances from $\hat \Phi(x_i)$ through
Isomap}\label{sec:isomap_shorten}
See $\proc{Isomap-Algo}(\cdot)$ in Fig.~\ref{fig:latentalgo} for the
pseudocode. After we construct our estimate $\hat \Phi_d$, we estimate $K$ by
letting $\hat K = \hat \Phi_d \hat \Phi^{\transpose}_d$. Recalling $K_{i, j} =
c_0/(|x_i - x_j|^{\Delta} + c_1)$, a plausible approach is to estimate $|x_i -
x_j| = (c_0/\hat K_{i, j} - c_1)^{1/\Delta}$. However, $\kappa(x_i, x_j)$ is a
convex function in $|x_i - x_j|$. Thus, when $K_{i, j}$ is small, a small
estimation error here will result in an amplified estimation error in $|x_i -
x_j|$ (see also Fig.~\ref{fig:ker} in App.~\ref{app:isomaphelp}). But when $|x_i - x_j|$ is
small, $K_{i, j}$ is reliable (see the ``reliable'' region in
Fig.~\ref{fig:ker}). 


Thus, our algorithm only uses large values of $K_{i, j}$ to construct
estimates.  The isomap technique introduced in topological
learning~\cite{tenenbaumglobal2000,Silva03globalversus} is designed to handle
this setting.  Specifically, the set $\calc = \{\Phi(x)\}_{x \in [0, 1]}$ forms
a curve in $\reals^{N_{\calh}}$ (Fig.~\ref{fig:isomapf}(a)). Our estimate $\{\hat
\Phi(x_i)\}_{i \in [n]}$ will be a noisy approximation of the curve
(Fig.~\ref{fig:isomapf}(b)). Thus, we build up a graph on $\{\Phi(x_i)\}_{i
\leq n}$ so that $x_i$ and $x_j$ are connected if and only if $\hat \Phi(x_i)$
and $\hat \Phi(x_j)$ are close (Fig.~\ref{fig:isomapf}(c) and/or (d)). Then the
shortest path distance on $G$ approximates the geodesic distance on $\calc$. By
using the fact that $\kappa$ is a radial basis kernel, the geodesic distance
will also be proportional to the latent distance.
 
\myparab{Corrupted nodes.} Excessively corrupted nodes may help build up
``undesirable bridges'' and interfere with the shortest-path based estimation
(cf.Fig.~\ref{fig:isomapf}(c)). Here, the shortest path between two green nodes
``jumps through'' the excessively corrupted nodes (labeled in red) so the
shortest path distance is very different from the geodesic distance. 

Below, we describe a procedure to remove excessively corrupted nodes and then
explain how to analyze the isomap technique's performance after their removal.
Note that $d$ in this section mostly refers to the shortest path distance
(rather than the number of eigenvectors we keep as used in the previous
section). 

\begin{figure}
\hspace{-1cm}
 \centering
 \begin{subfigure}[b]{0.23\textwidth}
 \includegraphics[width=\linewidth]{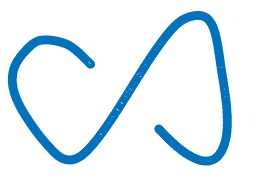}
	 \caption{\footnotesize{True features}}
 \end{subfigure}
 \begin{subfigure}[b]{0.23\textwidth}
 \includegraphics[width=\linewidth]{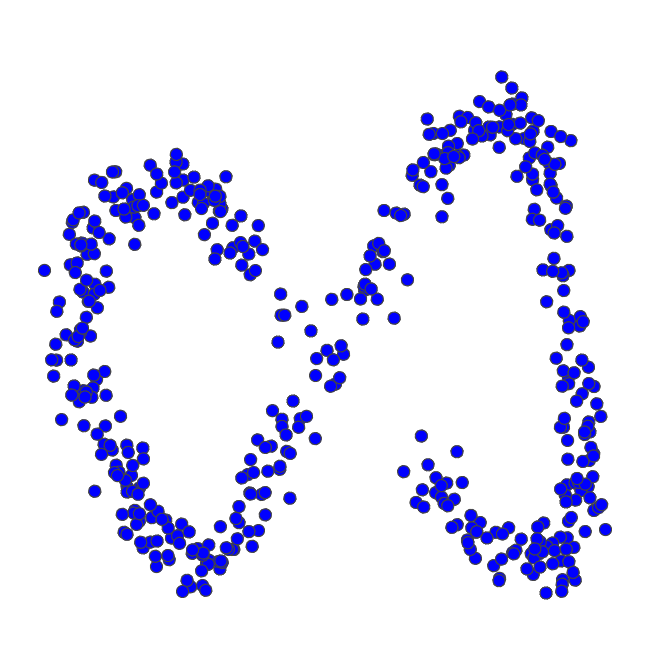}
	 \caption{\small{Estimated features}}
 \end{subfigure}
  \begin{subfigure}[b]{0.23\textwidth}
 \includegraphics[width=\linewidth]{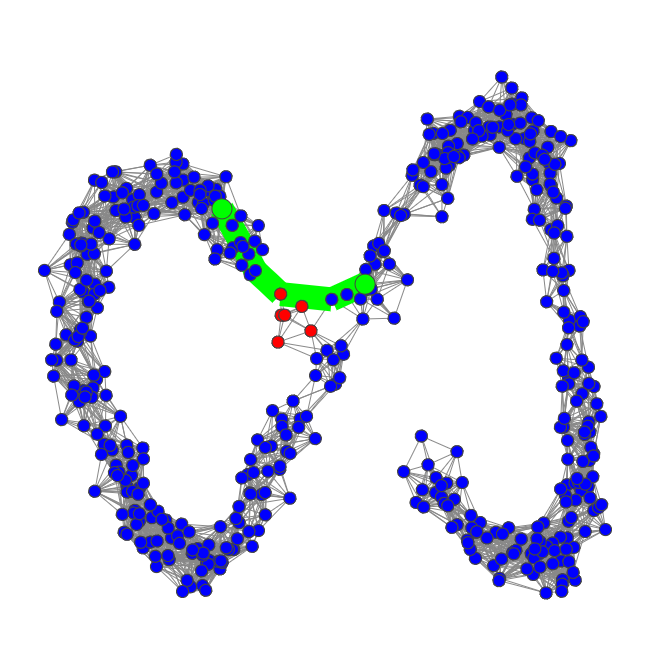}
	  \caption{\small{Isomap w/o denoising}}
 \end{subfigure}
  \begin{subfigure}[b]{0.23\textwidth}
 \includegraphics[width=\linewidth]{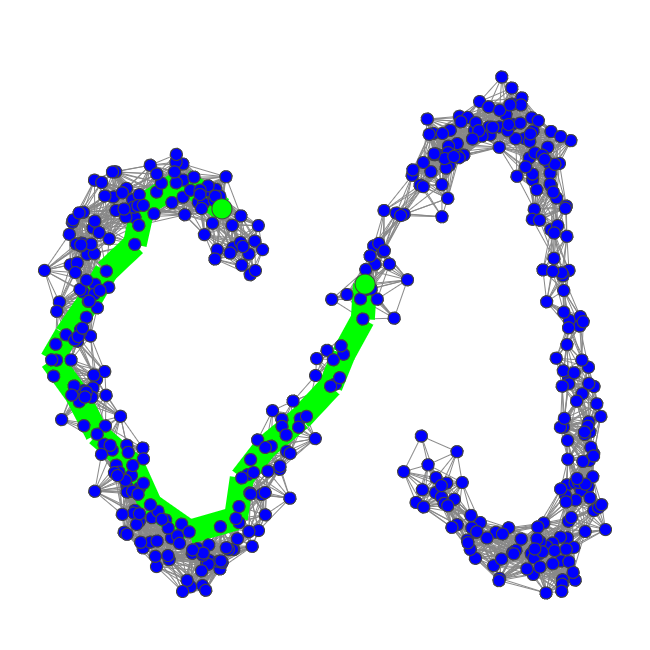}
	  \caption{\small{Isomap + denoising}}
 \end{subfigure}
 \caption{{\footnotesize Using the Isomap Algorithm to recover pairwise distances. (a) The true curve $\calc = \{\Phi(x)\}_{x \in [0, 1]}$ (b) Estimate $\hat \Phi$ (c) Shows that an undesirable short-cut may exist when we run the Isomap algorithm and (d) Shows the result of running the Isomap algorithm after removal of the corrupted nodes. 
 }}
 \label{fig:isomapf}
\end{figure}

\myparab{Step 1. Eliminate corrupted nodes.}
Recall that $x_1, x_2, ..., x_n$ are the latent variables. Let $z_i =
\Phi(x_i)$ and $\hat z_i = \hat \Phi(x_i)$. For any $z \in \reals^{N_{\calh}}$
and $r > 0$, we let $\Ball(z, r) = \{z': \|z'- z\| \leq r\}$.  Define
projection $\proj(z) = \arg \min_{z' \in \calc}\|z' - z\|$, where $\calc$ is
the curve formed by $\{\phi(x) \}_{x \in [0, 1]}$. Finally, for any point $z
\in \calc$, define $\Phi^{-1}(z)$ such that $\Phi(\Phi^{-1}(z)) = z$ (\ie $z$'s
original latent position). For the points that fall outside of $\calc$, define
$\Phi^{-1}(z) = \Phi^{-1}(\proj(z))$. 
%

Let us re-parametrize the error term in Propostion~\ref{prop:latentestimate}. Let $f(n)$ be that $\|\hat \Phi - \Phi\|_F \leq \sqrt{n}/f(n)$, where $f(n) = \rho^{2/87}(n) = \Omega(\log^2n)$ for sufficiently large $\rho(n)$. \Znote{remember to fix here.}
By  a  Markov inequality, we have $\Pr_i[\|\hat \Phi(x_i)  - \Phi(x_i) \|^2 \geq 1/\sqrt{f(n)}] \leq 1/f(n)$.

%

Intuitively, when $\|\hat \Phi(x_i) - \Phi(x_i) \|^2 \geq 1/\sqrt{f(n)}$, $i$
becomes a candidate that can serve to build up undesirable shortcuts. Thus, we
want to eliminate these nodes. 

Looking at a ball of radius $O(1/\sqrt{f(n)})$ centered at a point $\hat z_i$,
consider two cases. \smallskip \\
\emph{Case 1.} If $\hat z_i$ is  close to $\proj(\hat z_i)$, \ie corresponding
to the blue nodes in Figure~\ref{fig:isomapf}(c). For exposition purpose, let
us assume $\hat z_i = z_i$. Now for any point $z_j$, if $|x_i - x_j| =
O(f^{-1/\Delta}(n))$, then by Lemma~\ref{lem:latent}, we have $\|\hat z_i -
\hat z_j \| = O(1/\sqrt{f(n)})$, which means $z_j$ is in $\Ball(z_i,
O(1/\sqrt{f(n)}))$.  The total number of such nodes will be in the order of
$\Theta(n/f^{1/\Delta}(n))$, by using the near-uniform density assumption. \smallskip \\
\emph{Case 2.} If $\hat z_i$ is far away from any point in $\calc$, \ie
corresponding to the red ball in Figure~\ref{fig:isomapf}(c), any points in
$\Ball(\hat z_i, O(1/\sqrt{f(n)}))$ will also be far  from $\calc$. Then the
total number of such nodes will be $O(n/f(n))$.

As $n/f^{1/\Delta}(n) = \omega(n/f(n))$ for $\Delta > 1$, there is a
phase-transition phenomenon: When $\hat z_i$ is far  from $\calc$, then a
neighborhood of $\hat z_i$  contains $O(n/f(n))$ nodes. When $\hat z_i$ is
close to $\calc$, then a neighborhood of $\hat z_i$  contains $\omega(n/f(n))$
nodes.
 
We can leverage this intuition to design a counting-based algorithm to eliminate nodes that are far from $\calc$:
{\small
\begin{equation}\label{eqn:denoise}
\proc{Denoise}(\hat z_i): \mbox{If } |\Ball(\hat z_i, 3/\sqrt{f(n)})| < n/f(n), \mbox{  remove $\hat z_i$}.
\end{equation}
}
\mypara{Theoretical result.} We classify a point $i$ into three groups: \\  
1. \textbf{Good}: Satisfying $\| \hat z_i - \proj(\hat z_i) \| \leq 1/ \sqrt{f(n)}$.
We further partition the set of good points into two parts.
Good-I are points such that $\| \hat z_i - z_i \| \leq 1
/\sqrt{f(n)}$, while {Good-II} are points that are good but not in Good-I.  \\
2. \textbf{Bad}: when $\|z_i - \proj(z_i) \| > 4/\sqrt{f(n)}$.  \\
3. \textbf{Unclear}: otherwise. 

We prove the following result (see Appendix~\ref{asec:isomap} for a proof).

\begin{lemma}\label{lem:denoise}After running $\proc{Denoise}$ that uses the counting-based decision rule, all  good points are kept, all  bad points are eliminated, 
and all unclear points have no performance guarantee. The total number of eliminated nodes is $\leq n/f(n)$. 
\end{lemma} 

\myparab{Step 2. An isomap-based algorithm}.
Wlog assume there is only one closed interval for $\mathrm{support}(F)$. We build a graph $G$ on $[n]$ so that two nodes $\hat z_i$ and $\hat z_j$ are connected if and only if $\|\hat z_i - \hat z_j\| \leq \ell/\sqrt{f(n)}$, where $\ell$ is a sufficiently large constant (say 10). Consider the shortest path distance between arbitrary pairs of nodes $i$ and $j$ (that are not eliminated.) Because the corrupted nodes are removed, the whole path is around $\calc$. Also, by the uniform density assumption, walking on the shortest path in $G$ is equivalent to walking on $\calc$ with ``uniform speed'', \ie each edge on the path will map to an approximately fixed distance on $\calc$. Thus, the shortest path distance scales with the latent distance, \ie 
$(d - 1) \left(\frac c 2\right)^{1/\Delta}\left(\frac{\ell-3}{\sqrt{f(n)}}\right)^{2/\Delta} \leq |x_i - x_j| \leq d\left(\frac c 2\right)^{1/\Delta}\left(\frac{\ell + 8}{\sqrt{f(n)}}\right)^{2/\Delta}$, which implies 
Theorem~\ref{thm:main}. See Appendix~\ref{sec:isomapperformance} for a detailed
analysis.

\myparab{Bipartite Model}. Although we have focused our discussion on the
simplified model, we make a few remarks about inference in the more realistic
bipartite model. A more detailed discussion and the inference algorithm is
available at the beginning of Appendix~\ref{app:bipart} and full details
follow in that appendix. In the bipartite case, we no longer have access to
the kernel matrix $K$ for pairs of celebrity nodes; however, any non-diagonal
entry of $B^\transpose B$, say the $ij\th$ one, can be written as $\sum_{k}
Z_{ik} Z_{jk}$ where $Z_{ik}$ and $Z_{jk}$ are independent Bernoulli random
variables with parameters $\kappa(x_i, y_k)$ and $\kappa(x_j, y_k)$. This gives
rise to a \emph{square} kernel (of $\kappa$) which can be used to identify the
eigenvalues and eigenvectors of the kernel operator $\calk$ used in the
analysis of the simplified model. The diagonal entries have to be regularized
as there is no independence in the corresponding terms.

\myparab{Discussion: ``Gluing together'' two algorithms?} The unified model is
much more flexible than SBM and SWM. We were intrigued that the generalized
algorithm needs only to ``glue together'' important techniques used in both
models: Step 1 uses the spectral technique inspired by SBM inference methods,
while Step 2 resembles techniques used in SWM: the isomap $G$ only connects
between two nodes that are close, which is the same as throwing away the
long-range edges.

\section{Experiments}\label{sec:shortexp}

 \begin{wrapfigure}{L}{0.5\textwidth}
\centering
\footnotesize
\begin{tabular}{c|cccc}
\hline
Algo. & $\rho$ & Slope of $\beta$ & S.E. & p-value \\ \hline
Ours & \textbf{0.53} & 9.54 & 0.28 & $<$ 0.001\\ 
Mod.~\cite{newmanfinding2006} & 0.16 & 1.14 & 0.02 & $<$ 0.001\\ 
CA~\cite{BJN15} & 0.20 & 0.11 & 7e-4 & $<$ 0.001\\ 
Maj~\cite{raghavan}  & 0.13 & 0.09 & 0.02 & $<$ 0.001\\ 
RW~\cite{Silva03globalversus} & 0.01 & 1.92 & 0.65& $<$ 0.001\\ 
MDS~\cite{BorgGroenen2005}  & 0.05 & 30.91 & 120.9 & 0.09 \\ 
 \hline
\end{tabular}
\caption {{\footnotesize Latent Estimates vs. Ground-truth.  }}
\label{table:scorrelation}
\end{wrapfigure}

We apply our algorithm to a social interaction graph from Twitter to construct
users' ideology scores. We assembled a dataset by tracking keywords related to
the 2016 US presidential election for 10 million users. First, we note that as
of 2016 the Twitter interaction graph behaves ``in between'' the small-world
and stochastic blockmodels (see Figure~\ref{fig:inferred_kernel}), \ie the
latent distributions are bimodal but not as extreme as the SBM. 

 \begin{figure}[!h]
\hspace{-1cm}
 \centering
 \begin{subfigure}[b]{0.23\textwidth}
 \includegraphics[width=\linewidth]{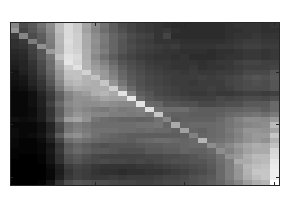}
 \caption{Inferred kernel}
 \end{subfigure}
 \begin{subfigure}[b]{0.23\textwidth}
 \includegraphics[width=\linewidth]{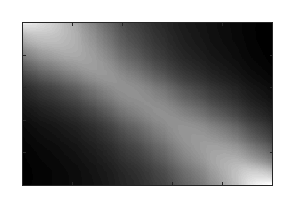}
 \caption{SWM}
 \end{subfigure}
  \begin{subfigure}[b]{0.23\textwidth}
 \includegraphics[width=\linewidth]{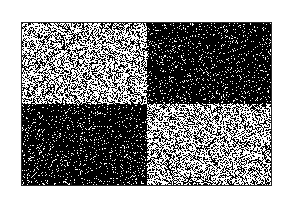}
 \caption{SBM}
 \end{subfigure}
	 \caption{{\footnotesize Visualization of real and synthetic networks. (a) Our inferred kernel matrix, which is ``in-between'' the small-world model (b) and the stochastic blockmodel (c). 
 }}
 \label{fig:inferred_kernel}
\end{figure}

\mypara{Ground-truth data.} Ideology scores of the US Congress (estimated by
third parties~\cite{congressIdeology}) are usually considered as a
``ground-truth'' (see, \eg~\cite{BJN15}) dataset. We apply our algorithm and
other baselines on Twitter data to estimate the ideology score of politicians
(members of the 114th Congress), and observe that our algorithm has the highest
correlation with ground-truth. See Fig.~\ref{table:scorrelation}.  Beyond
correlation, we also need to estimate the statistical significance of our
estimates. We set up a linear model $y\sim\beta_1\hat{x}+\beta_0$, in which
$\hat x$'s are our estimates and $y$'s are ground-truth.  We then use
bootstrapping  to compute the standard error of our estimator, and then use the
standard error to estimate the p-value of our estimator.  The details of this
experiment and additional empirical evaluation are available in
Appendix~\ref{sec:exp}. 



\footnotesize
\bibliographystyle{alpha}
\bibliography{Reference,submit_vers,twitter}

\newcommand{\etalchar}[1]{$^{#1}$}
\begin{thebibliography}{WTSC16}

\bibitem[Abb16]{Abbe15}
Emmanuel Abbe.
\newblock Community detection and the stochastic block model.
\newblock 2016.

\bibitem[ABFX08]{Airoldi2008}
Edoardo~M. Airoldi, David~M. Blei, Stephen~E. Fienberg, and Eric~P. Xing.
\newblock Mixed membership stochastic blockmodels.
\newblock {\em J. Mach. Learn. Res.}, 9:1981--2014, 2008.

\bibitem[ACC13]{ACC13}
Edo~M Airoldi, Thiago~B Costa, and Stanley~H Chan.
\newblock Stochastic blockmodel approximation of a graphon: Theory and
  consistent estimation.
\newblock In C.~J.~C. Burges, L.~Bottou, M.~Welling, Z.~Ghahramani, and K.~Q.
  Weinberger, editors, {\em Advances in Neural Information Processing Systems
  26}, pages 692--700. Curran Associates, Inc., 2013.

\bibitem[ACKS13]{AbrahamCKS13}
Ittai Abraham, Shiri Chechik, David Kempe, and Aleksandrs Slivkins.
\newblock Low-distortion inference of latent similarities from a multiplex
  social network.
\newblock In {\em SODA}, pages 1853--1872. SIAM, 2013.

\bibitem[AS15a]{AS:FOCS2015}
Emmanuel Abbe and Colin Sandon.
\newblock Community detection in the general stochastic block model:
  Fundamental limits and efficient algorithms for recovery.
\newblock In {\em Proceedings of 56th Annual IEEE Symposium on Foundations of
  Computer Science, Berkely, CA, USA}, pages 18--20, 2015.

\bibitem[AS15b]{AS:pp15}
Emmanuel Abbe and Colin Sandon.
\newblock Detection in the stochastic block model with multiple clusters: proof
  of the achievability conjectures, acyclic {BP}, and the
  information-computation gap.
\newblock {\em arXiv preprint arXiv:1512.09080}, 2015.

\bibitem[Bar12]{Barbera12}
Pablo Barber\'{a}.
\newblock {Birds of the Same Feather Tweet Together. Bayesian Ideal Point
  Estimation Using Twitter Data}.
\newblock 2012.

\bibitem[BC09]{BC:PNAS2009}
Peter~J. Bickel and Aiyou Chen.
\newblock A nonparametric view of network models and newman–girvan and other
  modularities.
\newblock {\em Proceedings of the National Academy of Sciences},
  106(50):21068--21073, 2009.

\bibitem[BCIS05]{BadoiuCIS05}
Mihai Badoiu, Julia Chuzhoy, Piotr Indyk, and Anastasios Sidiropoulos.
\newblock Low-distortion embeddings of general metrics into the line.
\newblock In {\em Proceedings of the 37th Annual {ACM} Symposium on Theory of
  Computing, Baltimore, MD, USA, May 22-24, 2005}, pages 225--233, 2005.

\bibitem[BG05]{BorgGroenen2005}
I.~Borg and P.J.F. Groenen.
\newblock {\em {Modern Multidimensional Scaling: Theory and Applications}}.
\newblock Springer, 2005.

\bibitem[BJN{\etalchar{+}}15]{BJN15}
Pablo Barber\'{a}, John~T. Jost, Jonathan Nagler, Joshua~A. Tucker, and Richard
  Bonneau.
\newblock Tweeting from left to right.
\newblock {\em Psychological Science}, 26(10):1531--1542, 2015.

\bibitem[Cha15]{chatterjee2015}
Sourav Chatterjee.
\newblock Matrix estimation by universal singular value thresholding.
\newblock {\em Ann. Statist.}, 43(1):177--214, 02 2015.

\bibitem[CJR04]{clinton04}
J.~Clinton, S.~Jackman, and D.~Rivers.
\newblock The statistical analysis of roll call data.
\newblock {\em American Political Science Review}, 98(2):355--370, 2004.

\bibitem[CR13]{CohenR13}
Raviv Cohen and Derek Ruths.
\newblock Classifying political orientation on twitter: It’s not easy!
\newblock In {\em International AAAI Conference on Weblogs and Social Media},
  2013.

\bibitem[Dhi01]{Dhillon2001}
Inderjit~S. Dhillon.
\newblock Co-clustering documents and words using bipartite spectral graph
  partitioning.
\newblock In {\em Proceedings of the Seventh ACM SIGKDD International
  Conference on Knowledge Discovery and Data Mining}, KDD '01, pages 269--274,
  New York, NY, USA, 2001. ACM.

\bibitem[DK70]{DK:1970}
C.~Davis and W.~M. Kahan.
\newblock The rotation of eigenvectors by a perturbation.
\newblock {\em SIAM J. Numer. Anal.}, 7:1--46, 1970.

\bibitem[GB11]{gerrish2}
S.~Gerrish and D.~Blei.
\newblock Predicting legislative roll calls from text.
\newblock In {\em Proc.\ ICML}, 2011.

\bibitem[GB12]{gerrish12}
S.~Gerrish and D.~Blei.
\newblock How the vote: Issue-adjusted models of legislative behavior.
\newblock In {\em Proc.\ NIPS}, 2012.

\bibitem[GS13]{grimmer13}
J.~Grimmer and B.~M. Stewart.
\newblock Text as data: The promise and pitfalls of automatic content analysis
  methods for political texts.
\newblock {\em Political Analysis}, 2013.

\bibitem[HLL83]{HLL:1983}
Paul~W Holland, Kathryn~Blackmond Laskey, and Samuel Leinhardt.
\newblock Stochastic blockmodels: First steps.
\newblock {\em Social networks}, 5(2):109--137, 1983.

\bibitem[HRH01]{Hoff01latentspace}
Peter~D. Hoff, Adrian~E. Raftery, and Mark~S. Handcock.
\newblock Latent space approaches to social network analysis.
\newblock {\em JOURNAL OF THE AMERICAN STATISTICAL ASSOCIATION}, 97:1090--1098,
  2001.

\bibitem[IM04a]{IM:HDCG-2004}
Piotr Indyk and Jiri Matou{\v{s}}ek.
\newblock Low-distortion embeddings of finite metric spaces.
\newblock {\em Handbook of discrete and computational geometry}, page 177,
  2004.

\bibitem[IM04b]{Indyk04low-distortionembeddings}
Piotr Indyk and Jiri Matousek.
\newblock Low-distortion embeddings of finite metric spaces.
\newblock In {\em in Handbook of Discrete and Computational Geometry}, pages
  177--196. CRC Press, 2004.

\bibitem[Kat87]{Kat:1987}
Tosio Kato.
\newblock Variation of discrete spectra.
\newblock {\em Communications in Mathematical Physics}, 111(3):501--504, 1987.

\bibitem[Kle00]{K:2000}
Jon Kleinberg.
\newblock The small-world phenomenon: An algorithmic perspective.
\newblock In {\em Proceedings of the thirty-second annual ACM symposium on
  Theory of computing}, pages 163--170. ACM, 2000.

\bibitem[KMS16]{KanadeMS16}
Varun Kanade, Elchanan Mossel, and Tselil Schramm.
\newblock Global and local information in clustering labeled block models.
\newblock {\em {IEEE} Trans. Information Theory}, 62(10):5906--5917, 2016.

\bibitem[K{\"o}n86]{konig1986eigenvalue}
H.~K{\"o}nig.
\newblock {\em Eigenvalue Distribution of Compact Operators}.
\newblock Operator Theory: Advances and Applications. Birkh{\"a}user, 1986.

\bibitem[LBG03]{laver03}
M.~Laver, K.~Benoit, and J.~Garry.
\newblock Extracting policy positions from political texts using words as data.
\newblock {\em American Political Science Review}, 97(2), 2003.

\bibitem[LLDM08]{LLDM:WWW2008}
Jure Leskovec, Kevin~J Lang, Anirban Dasgupta, and Michael~W Mahoney.
\newblock Statistical properties of community structure in large social and
  information networks.
\newblock In {\em Proceedings of the 17th international conference on World
  Wide Web}, pages 695--704. ACM, 2008.

\bibitem[LLM10]{LLM:WWW2010}
Jure Leskovec, Kevin~J Lang, and Michael Mahoney.
\newblock Empirical comparison of algorithms for network community detection.
\newblock In {\em Proceedings of the 19th international conference on World
  wide web}, pages 631--640. ACM, 2010.

\bibitem[Lov12]{lovasz2012large}
L.~Lovasz.
\newblock {\em Large Networks and Graph Limits}.
\newblock American Mathematical Society colloquium publications. American
  Mathematical Society, 2012.

\bibitem[LP12]{stateIdeology}
Jeffrey~R Lax and Justin~H Phillips.
\newblock The democratic deficit in the states.
\newblock {\em American Journal of Political Science}, 56(1):148--166, 2012.

\bibitem[Mas14]{Mas:STOC2014}
Laurent Massouli{\'e}.
\newblock Community detection thresholds and the weak {R}amanujan property.
\newblock In {\em Proceedings of the 46th Annual ACM Symposium on Theory of
  Computing}, pages 694--703. ACM, 2014.

\bibitem[McS01]{McS:FOCS2001}
Frank McSherry.
\newblock Spectral partitioning of random graphs.
\newblock In {\em Foundations of Computer Science, 2001. Proceedings. 42nd IEEE
  Symposium on}, pages 529--537. IEEE, 2001.

\bibitem[MNS13]{MNS:pp2013}
Elchanan Mossel, Joe Neeman, and Allan Sly.
\newblock A proof of the block model threshold conjecture.
\newblock {\em arXiv preprint arXiv:1311.4115}, 2013.

\bibitem[MNS15]{MNS:2015}
Elchanan Mossel, Joe Neeman, and Allan Sly.
\newblock Reconstruction and estimation in the planted partition model.
\newblock {\em Probability Theory and Related Fields}, 162(3-4):431--461, 2015.

\bibitem[MP02]{MP:RANDOM2002}
Milena Mihail and Christos Papadimitriou.
\newblock On the eigenvalue power law.
\newblock In {\em International Workshop on Randomization and Approximation
  Techniques in Computer Science}, pages 254--262. Springer, 2002.

\bibitem[New06]{newmanfinding2006}
Mark~EJ Newman.
\newblock Finding community structure in networks using the eigenvectors of
  matrices.
\newblock {\em Physical review E}, 74, 2006.

\bibitem[NG04]{NG:PRE2004}
Mark~EJ Newman and Michelle Girvan.
\newblock Finding and evaluating community structure in networks.
\newblock {\em Physical review E}, 69(2):026113, 2004.

\bibitem[NWS02]{NWS:PNAS2002}
Mark~EJ Newman, Duncan~J Watts, and Steven~H Strogatz.
\newblock Random graph models of social networks.
\newblock {\em Proceedings of the National Academy of Sciences}, 99(suppl
  1):2566--2572, 2002.

\bibitem[O{\etalchar{+}}10]{Oli:2010}
Roberto~Imbuzeiro Oliveira et~al.
\newblock Sums of random hermitian matrices and an inequality by rudelson.
\newblock {\em Electron. Commun. Probab}, 15(203-212):26, 2010.

\bibitem[Oli10]{oliveira2010}
Roberto Oliveira.
\newblock Sums of random hermitian matrices and an inequality by rudelson.
\newblock {\em Electron. Commun. Probab.}, 15:203--212, 2010.

\bibitem[PJW13]{WolfeO13}
Sofia C.~Olhede Patrick J.~Wolfe.
\newblock Nonparametric graphon estimation.
\newblock 2013.

\bibitem[PR85]{poole85}
K.\~T. Poole and H.~Rosenthal.
\newblock A spatial model for legislative roll call analysis.
\newblock {\em American Journal of Political Science}, 29(2):357--384, 1985.

\bibitem[QR13a]{QR13NIPS}
Tai Qin and Karl Rohe.
\newblock Regularized spectral clustering under the degree-corrected stochastic
  blockmodel.
\newblock In C.j.c. Burges, L.~Bottou, M.~Welling, Z.~Ghahramani, and K.q.
  Weinberger, editors, {\em Advances in Neural Information Processing Systems
  26}, pages 3120--3128. 2013.

\bibitem[QR13b]{QinR2013}
Tai Qin and Karl Rohe.
\newblock Regularized spectral clustering under the degree-corrected stochastic
  blockmodel.
\newblock In {\em Proceedings of the 26th International Conference on Neural
  Information Processing Systems}, NIPS'13, pages 3120--3128, USA, 2013. Curran
  Associates Inc.

\bibitem[RAK07]{raghavan}
U.~N. Raghavan, R.~Albert, and S.~Kumara.
\newblock Near linear time algorithm to detect community structures in
  large-scale networks.
\newblock {\em Physical Review E}, 76(3), 2007.

\bibitem[RBV10]{Rosasco2010}
Lorenzo Rosasco, Mikhail Belkin, and Ernesto~De Vito.
\newblock On learning with integral operators.
\newblock {\em J. Mach. Learn. Res.}, 11:905--934, March 2010.

\bibitem[RCY11]{Rohe11}
Karl Rohe, Sourav Chatterjee, and Bin Yu.
\newblock Spectral clustering and the high-dimensional stochastic blockmodel.
\newblock {\em The Annals of Statistics}, 39(4):1878--1915, 2011.

\bibitem[RQY16]{rohe2016co}
Karl Rohe, Tai Qin, and Bin Yu.
\newblock Co-clustering directed graphs to discover asymmetries and directional
  communities.
\newblock {\em Proceedings of the National Academy of Sciences},
  113(45):12679--12684, 2016.

\bibitem[SS01]{Scholkopf2001}
Bernhard Scholkopf and Alexander~J. Smola.
\newblock {\em Learning with Kernels: Support Vector Machines, Regularization,
  Optimization, and Beyond}.
\newblock MIT Press, Cambridge, MA, USA, 2001.

\bibitem[ST03]{Silva03globalversus}
Vin~De Silva and Joshua~B. Tenenbaum.
\newblock Global versus local methods in nonlinear dimensionality reduction.
\newblock In {\em Advances in Neural Information Processing Systems 15}, pages
  705--712. MIT Press, 2003.

\bibitem[Tau12]{congressIdeology}
Joshua Tauberer.
\newblock Observing the unobservables in the us congress.
\newblock {\em Law Via the Internet}, 2012.

\bibitem[TdSL00]{tenenbaumglobal2000}
Joshua~B. Tenenbaum, Vin de~Silva, and John~C. Langford.
\newblock A global geometric framework for nonlinear dimensionality reduction.
\newblock {\em Science}, 290(5500):2319, 2000.

\bibitem[Tro12]{Tropp:FOCM2012}
Joel~A. Tropp.
\newblock User-friendly tail bounds for sums of random matrices.
\newblock {\em Foundations of Computational Mathematics}, 12(4):389--434, 2012.

\bibitem[TSP13]{tang2013}
Minh Tang, Daniel~L. Sussman, and Carey~E. Priebe.
\newblock Universally consistent vertex classification for latent positions
  graphs.
\newblock {\em Ann. Statist.}, 41(3):1406--1430, 06 2013.

\bibitem[WC14]{WolfeChoi14}
Patrick~J. Wolfe and David Choi.
\newblock Co-clustering separately exchangeable network data.
\newblock {\em The Annals of Statistics}, 42(1):29--63, 2014.

\bibitem[Wey12]{Weyl1912}
H.~Weyl.
\newblock Das asymptotische verteilungsgesetz der eigenwerte linearer
  partieller differentialgleichungen (mit einer anwendung auf die theorie der
  hohlraumstrahlung).
\newblock {\em Mathematische Annalen}, 71:441--479, 1912.

\bibitem[WS98]{WS:1998}
Duncan~J Watts and Steven~H Strogatz.
\newblock Collective dynamics of ‘small-world’ networks.
\newblock {\em nature}, 393(6684):440--442, 1998.

\bibitem[WTSC16]{WongTSC16}
Felix Ming~Fai Wong, Chee{-}Wei Tan, Soumya Sen, and Mung Chiang.
\newblock Quantifying political leaning from tweets, retweets, and retweeters.
\newblock {\em {IEEE} Trans. Knowl. Data Eng.}, 28(8):2158--2172, 2016.

\bibitem[WZ15]{Wathen2015}
Andrew~J. Wathen and Shengxin Zhu.
\newblock On spectral distribution of kernel matrices related to radial basis
  functions.
\newblock {\em Numerical Algorithms}, 70(4):709--726, 2015.

\bibitem[XML14]{xu14}
Jiaming Xu, Laurent Massouli\'{e}, and Marc Lelarge.
\newblock Edge label inference in generalized stochastic block models: from
  spectral theory to impossibility results.
\newblock In Maria~Florina Balcan, Vitaly Feldman, and Csaba Szepesvári,
  editors, {\em Proceedings of The 27th Conference on Learning Theory},
  volume~35 of {\em Proceedings of Machine Learning Research}, pages 903--920,
  Barcelona, Spain, 13--15 Jun 2014. PMLR.

\bibitem[YP16]{YunP16}
Se{-}Young Yun and Alexandre Prouti{\`{e}}re.
\newblock Optimal cluster recovery in the labeled stochastic block model.
\newblock In {\em Advances in Neural Information Processing Systems 29: Annual
  Conference on Neural Information Processing Systems 2016, December 5-10,
  2016, Barcelona, Spain}, pages 965--973, 2016.

\bibitem[ZB05]{ZB:2005}
Laurent Zwald and Gilles Blanchard.
\newblock On the convergence of eigenspaces in kernel principal component
  analysis.
\newblock In {\em Advances in Neural Information Processing Systems 18 [Neural
  Information Processing Systems, {NIPS} 2005, December 5-8, 2005, Vancouver,
  British Columbia, Canada]}, pages 1649--1656, 2005.

\bibitem[ZLZ12]{zhao2012}
Yunpeng Zhao, Elizaveta Levina, and Ji~Zhu.
\newblock Consistency of community detection in networks under degree-corrected
  stochastic block models.
\newblock {\em Ann. Statist.}, 40(4):2266--2292, 08 2012.

\bibitem[ZRMZ07]{2007PhRvEZhou}
T.~{Zhou}, J.~{Ren}, M.~{Medo}, and Y.-C. {Zhang}.
\newblock {Bipartite network projection and personal recommendation}.
\newblock 76(4):046115, October 2007.

\end{thebibliography}

\newpage
\appendix

\section{Additional Illustrations}
This section provides additional illustrations related to our work. 

\subsection{Model Selection Probem (presented in Section~\ref{sec:intro})}

\label{app:introstuff}
Without observing the latent positions or knowing which model generated the
underlying graph, the adjacency matrix of a social graph typically looks like
the one shown in Fig.~\ref{fig:adjmats}(a). However, if the model generating
the graph is known, it is then possible to run a suitable ``clustering
algorithm''~\cite{McS:FOCS2001,AbrahamCKS13} that reveals the hidden structure.
When the vertices are ordered suitably, the SBM's adjacency matrix looks like
the one shown in Fig.~\ref{fig:adjmats}(b) and that of the SWM looks like the
one shown in Fig.~\ref{fig:adjmats}(c).

\begin{figure}[!h]
	\hspace{-1cm}
 	\centering
	\includegraphics[scale=0.65]{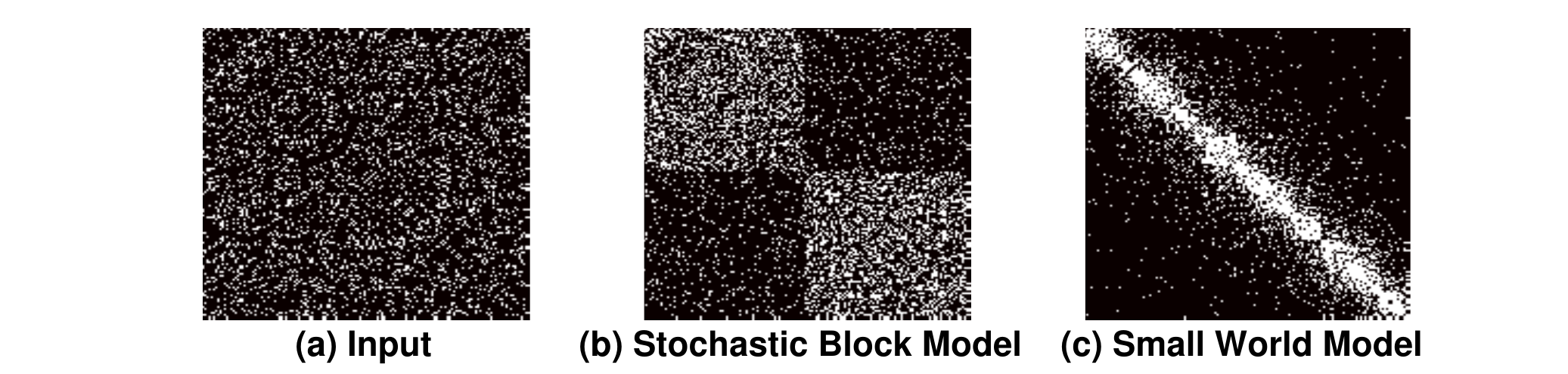}
	\caption{Inference problem in social graphs: Given an input graph (a), are we
	able to shuffle the nodes so that statistical patterns are revealed? A
	major problem in network inference is the model selection problem.
	The input here can come from stochastic block model (b) or small-world
	model (c). \label{fig:adjmats}}
\end{figure}

\subsection{Algorithm for the Small-world Model (presented in Section~\ref{sec:prelim})}
\label{app:swintro}
The inference algorithm for small-world networks uses different ideas. Each
edge in the graph can be thought of as a ``short-range'' or ``long-range'' one.
Short-range edges are those between nodes that are nearby in latent space,
while long-range ones have end-points that are far away in latent space. After
the removal of all the long-range edges, the shortest path distance between two
nodes scales proportionally to the corresponding latent space distance (see
Fig.~\ref{fig:swm}). Once estimates for pairwise distances are obtained,
standard buidling blocks may be used to find the latent positions
$x_i$~\cite{IM:HDCG-2004}.

\begin{figure}[!h]
	\centering
	\begin{tikzpicture}[scale=0.7]
		\foreach \x in {0,...,9} {
			\node[circle, fill] at (\x, 0) {};
			}
		\foreach \x in {0,...,8} {
			\draw [black] (\x, 0) -- ({\x + 1}, 0);
		}
		\foreach \x in {0,...,7} {
			\draw (\x, 0) .. controls ({\x + 1}, 0.5) .. ({ \x + 2}, 0);
		}
		\foreach \x/\y in {0/4,  2/6, 3/7,  5/9} {
			\draw [red, very thick] (\x, 0) .. controls ({(\x + \y)/2}, 1.5) .. (\y, 0);
		}
	\end{tikzpicture}
	\caption{{In the small-world model, after removal of
	long range edges (red thick), the shortest-path distance between two
	nodes approximates latent space distance}}
\label{fig:swm}
\end{figure}
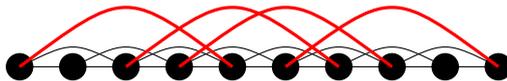

\subsection{Sensitivity of the Gram matrix $K$ (presented in Section~\ref{sec:isomap_shorten})}

\label{app:isomaphelp}

After we construct our estimate $\hat \Phi_d$, we may estimate $K$ by letting
$\hat K = \hat \Phi_d \hat \Phi^{\transpose}_d$. Recalling $K_{i, j} =
c_0/(|x_i - x_j|^{\Delta} + c_1)$, one plausible approach would be estimating
$|x_i - x_j| = (c_0/\hat K_{i, j} - c_1)^{1/\Delta}$. A main issue with this
approach is that $\kappa(x_i, x_j)$ is a convex function in $|x_i - x_j|$.
Thus, when $K_{i, j}$ is small, a small estimation error here will result in an
amplified estimation error in $|x_i - x_j|$ (cf. Fig.~\ref{fig:ker}).  But when
$|x_i - x_j|$ is small, $K_{i, j}$ is reliable (see the ``reliable'' region in
Fig.~\ref{fig:ker}). 

 \begin{figure}[!h]
\centering
\includegraphics[width=0.5\textwidth]{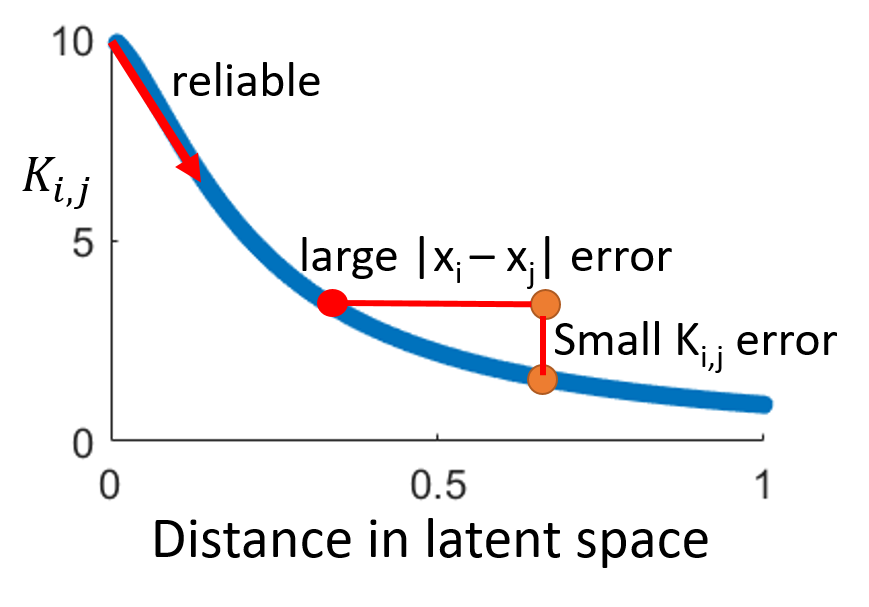}
\caption{The behavior of $\hat K_{i,j}$.}
\label{fig:ker}
\end{figure}

\section{Simplified model case: Using $A$ to approximate $\Phi(X)$}\label{a:undirected}
This section proves the following proposition.


\begin{proposition}[Repeat of Proposition~\ref{prop:latentestimate}]\label{prop:latentestimatea} 
	Let $t$ be a tunable parameter such that $t = o(\rho(n))$ and $t^2/\rho(n) =
	\omega(\log n)$. Let $d$ be chosen by $\proc{DecideThreshold}(\cdot)$. Let
	$\hat \Phi \in \reals^{N_{\calh}}$ be such that its first $d$-coordinates are
	equal to $\sqrt{C(n)}U_AS^{1/2}_A$. If $\rho(n) = \omega(\log n)$ and
	$\calk$ is well-conditioned, then with high probability:
\begin{equation}\label{eqn:latentestimate}
\| \hat \Phi - \Phi\|_F= O\left(\sqrt n \left(\frac t {\rho(n)}\right)^{\frac 2{29}}\right)
\end{equation}
\end{proposition}

We will break down our analysis into three components, each of which
corresponds to an approximation error source presented in
Section~\ref{sec:estphi_short}. Some statements that appeared earlier are
repeated in this section to make it self-contained. 

Before proceeding, we need more notation. 

\myparab{Additional Notation}.  Let $\calh$ denote the reproducing kernel
Hilbert space of $\kappa$, so that each element $\eta \in \calh$ can be
uniquely expressed as $\eta = \sum_j a_j \sqrt{\lambda_j}\psi_j$. The inner
product of elements in $\calh$ is given by $\left\langle\sum_j a_j
\sqrt{\lambda_j}\psi_j, \sum_j b_j \sqrt{\lambda_j}\psi_j\right\rangle_{\calh}
= \sum_j a_j b_j$. 

\subsection{Error Source 1: Finite Samples to Learn the Kernel}

Recall that we want to infer about ``continous objects'' $\kappa$ and $\cald$
(more specifically eigenfunctions of the integral operator $\calk$ derived
using $\kappa$ and $F$) but $K$ gives only the kernel values for a finite set of
pairs, so estimates constructed from $K$ are only approximations. Here, we need
only an existing result from Kernel PCA~\cite{Rosasco2010,tang2013}. 

\begin{lemma}\label{lem:kpcaexist}Using the notations above, we have 
	\begin{equation}\label{eqn:kpcar}
		\| U_K S^{\sfrac{1}{2}}_K W - \Phi_d(X) \|_F \leq 2\sqrt 2
		\frac{\sqrt{\log \epsilon^{-1}}}{\lambda_d(\calk) - \lambda_{d + 1}(\calk)} = 2\sqrt 2
		\frac{\sqrt{\log \epsilon^{-1}}}{\delta_d} 
	\end{equation}
\end{lemma}

We remark on the (implicit) dependence on the sample size in (\ref{eqn:kpcar}).
Here, the right-hand side is the total error on all the samples, which is
independent of $n$, and hence the \emph{average} square error shrinks as $O(1/n)$.  

\subsection{Error Source 2: Only Observe $A$} 

We observe only the realized graph $A$ rather than the gram matrix $K$, such
that $\E A = K/C(n)$.  Thus, we can use only singular vectors of $C(n)A$ to
approximate $\tilde{U}_{K}\tilde{S}^{1/2}_{K}$. Our main goal is to prove the
following lemma. 

\begin{lemma}\label{lem:aapproxk}Using the notation above, we have
\begin{equation}\label{eqn:aapprox}
\left\|\sqrt{C(n)}U_AS^{1/2}_A- U_KS^{1/2}_K\right\|_F = O\left(\frac{t \sqrt{dn}}{\delta^2_d\rho(n)}\right)
\end{equation}
\end{lemma}

The outline of the proof is as follows. 

\myparab{Step 1. Show that $\|A - K/C(n)\|$ is small}. This can be done by observing that $A_{i, j}$ are independent for different pairs of $i < j$ and applying a tail inequality on independent matrix sum. 

\myparab{Step 2. Apply a Davis-Kahan theorem to show that $\calp_A$ and $\calp_K$ are close.}. Let $\calp_A = U_A U^{\transpose}_A$ and $\calp_K = U_KU^{\transpose}_K$ be the projection operators onto the linear subspaces spanned by the eigenvectors corresponding to the $d$ largest eigenvalues of $A$ and $K$ respectively. Davis-Kahan theorem gives a sufficient condition that $U_A$ and $U_K$ are close (upto a unitary operation), \ie $\|A - K/C(n)\|$ needs to be small (from step 1) and $\delta_d = \lambda_d(\calk) - \lambda_{d + 1}(\calk)$ needs to be large (from $d$ is a suitable constant). Thus $U_A$  and $U_K$ are close up to a unitary operation, which implies $\calp_A$ and $\calp_K$ are close. We will specifically show that  $\|\calp_A - \calp_K\|_{\hsn}$ is small. $\|\cdot \|_{\hsn}$ refers to the Hilbert-Schmidt norm.

\begin{definition}\label{def:hsn}The Hilbert-Schmidt norm of a bounded operator $A$ on a Hilbert space $H$ is 
\begin{equation}
\|A\|^2_{\hsn} = \sum_{i \in I}\|Ae_i\|^2, 
\end{equation}
where $\{e_i: i \in I\}$ is an orthonormal basis of $H$. 
\end{definition}

\myparab{Step 3. Show that $\sqrt{C(n)}U_AS^{1/2}_A$  and $U_KS^{1/2}_K$ are close (up to a unitary operation).} We first argue that $\|\calp_A C(n)A - \calp_K K)\|$ is small. Then by observing that $\sqrt{C(n)}U_AS^{1/2}_A$  and $U_KS^{1/2}_K$ are ``square root'' of $\calp_A (C(n) A)$ and $\calp_K K$, we can show $\sqrt{C(n)}U_AS^{1/2}_A$  and $U_KS^{1/2}_K$ are close.

We now follow the workflow to prove the proposition.

\subsubsection{Step 1. $\|A - K\|$ is small}
We use the following concentration bound for matrix~\cite{Tropp:FOCM2012}. 

\begin{theorem}\label{thm:matrixineq}
Consider a finite sequence $\{X_k\}$ of independent random, self-adjoint matrices with dimension $d$.  Assume that each random matrix satisfies 
\begin{equation}
\E[X_k] = 0 \quad \mbox{ and } \quad \lambda_{\max}(X_k) \leq R \quad a.s.
\end{equation}
Then for all $t \geq 0$, 
\begin{equation}
\Pr[\lambda_{\max}(\sum_k X_k) \geq t] \leq d \exp\left(-\frac{t^2/2}{\sigma^2 + Rt/3}\right),
\end{equation}
where $\sigma^2 = \|\sum_k \E[X^2_k]\|$.
\end{theorem}

We apply the above theorem to bound $\|A - K/(C(n))\|$. Let $p_{i, j} = K_{i, j}/C(n)$  represent the probability that there is a link between $v_i$ and $v_j$.  Let random matrix $E_{i, j} \in \reals^{n \times n}$ be that the $(i,j)$-th entry and $(j, i)$-th entry are 1 with probability $p_{i, j}$, and $0$ otherwise. The remaining entries in $E_{i,j}$ are all $0$. Let $F_{i, j} = E_{i, j} - \E[E_{i, j}]$. Note that $A = \sum_{i \leq j}E_{i, j}$ and $\{E_{i, j}\}_{i\leq j}$ are all independent. We also have $\|A - K/C(n)\| = \|\sum_{i\leq j}F_{i,j}\|$. 

Note that: 
\begin{enumerate}
\item $\lambda_{\max}(F_{i, j}) = \Theta(1)$ a.s.
\item $F^2_{i, j} \in \reals^2$ is a matrix such that only $(i, i)$-th and $(j, j)$-th entries can non-zero. Furthermore,
 $$(F^2_{i, j})_{i,i} = (F^2_{i,j})_{j , j} = \left\{\begin{array}{ll}p^2_{i, j} & \mbox{with probability } 1- p_{i, j} \\ (1-p_{i,j})^2 & \mbox{with probability } p_{i,j} \end{array}\right.$$
Thus, $\E[(F^2_{i,j})_{i,i}] \leq p_{i, j}$. One can see that $\sum_{i \leq j} \E[(F^2_{i,j})]$ is a diagonal matrix such that the $(i, i)$-th entry is 
$\leq 2 \sum_{j \leq n}\E[(F_{i,j})^2_{i, i}] = O(n p_{i, j}) = O(\rho(n))$. Thus $\sigma^2$ in the theorem shall be $O(\rho(n))$. 
\end{enumerate}

We then have
\begin{equation}
\Pr\left[\lambda_{\max}\left(\sum_{i, j} F_{i, j}\right) \geq t\right]  \leq  \exp\left(-\frac{t^2}{O(\rho(n) + t)}\right) = \exp\left(-\Theta\left(\min\left(\frac{t^2}{\rho(n)}, t\right)\right)\right)
\end{equation}

We shall also see that $\rho(n) = \omega(t)$ is needed. Thus, 

\begin{equation}
\Pr\left[\left\|A - \frac K{C(n)}\right\| \geq t\right] \leq \exp\left(-\Theta\left(\frac{t^2}{\rho(n)}\right)\right). 
\end{equation}

\subsubsection{Step 2. Show that $\|\calp_A - \calp_K\|_{\hsn}$ is small}
Recall that $\delta_d = \lambda_d(\calk) - \lambda_{d + 1}(\calk)$. 
Because the projection is scale-invariant, we work on the matrices $C(n)A/n = A/\rho(n)$ and $K/n$ instead of $A$ and $K$. By standard results from kernel PCA~\cite{Rosasco2010}, we have with probability $\geq 1 - \epsilon$,
\begin{equation}
\lambda_d\left(\frac K n\right) - \lambda_{d + 1}\left(\frac K n\right) \geq \delta_d - 4 \sqrt 2  \sqrt{\frac{\log(1/\epsilon)}{n}}
\end{equation}

Define 
\begin{equation}
\begin{array}{ll}
S_1 & = \{ \lambda: \lambda \geq \lambda_d(K/n) - t/\rho(n)\}\\
S_2 & = \{\lambda: \lambda \leq \lambda_{d + 1}(K/n) + t/\rho(n)\}.
\end{array}
\end{equation}

\myparab{Comparing $\delta$ with $t/\rho(n)$.} 
We next relate $\delta_d$ with $t/\rho(n)$. Recall that $t = o(\rho(n))$ our algorithm $\proc{DecideThreshold}(t, \rho(n))$ in Fig.~\ref{fig:latentalgo}). We claim that  $\delta_d = \omega(t/\rho(n))$.

\begin{lemma}\label{alem:gap}Using the notations above. Suppose we use $\proc{DecideThreshold}(t, \rho(n))$ in Fig.~\ref{fig:latentalgo} to decide the 
number of eigenvectors/eigenvalues to keep, we have  
$$\delta_d =  \Theta(\lambda_d(A/\rho(n)) - \lambda_{d + 1}(A/\rho(n))) =  \omega(t/\rho(n)).$$
\end{lemma}

\begin{proof}Note that 
$$|\lambda_d(\calk) - \delta_d(A/\rho(n))| \leq |\lambda_d(\calk) - \lambda_d(K/n)| + |\lambda_d(K/n) + \lambda_d(A/\rho(n))|.$$
From~\cite{Rosasco2010}, we have
$$|\lambda_d(\calk) - \lambda_d(K/n)| = O\left(\sqrt{\frac{\log(1/\epsilon)}{n}}\right) = O\left(\frac t {\rho(n)}\right).$$
Then from Step 1 and~\cite{Kat:1987}, we have 
$$|\lambda_d(K/n) - \lambda_d(A/\rho(n))| \leq \|K/n - A/\rho(n)\| = O(t/\rho(n)).$$
Thus, we have
$$|\lambda_d(\calk) - \lambda_d(A/\rho(n))| = O(t/\rho(n)).$$
Similarly, we can show that 
$$|\lambda_{d + 1}(\calk) - \lambda_{d+1}(A/\rho(n))| = O(t/\rho(n)).$$
Finally, note that 
\begin{eqnarray*}
|\lambda_d(A/\rho(n)) - \lambda_{d + 1}(A/\rho(n))| & \leq  & |\lambda_d(A/\rho(n)) - \lambda_d(\calk)| + |\lambda_d(\calk) - \lambda_{d + 1}(\calk)| \\
& & \quad + |\lambda_{d + 1}(A/\rho(n))  -\lambda_{d + 1}(\calk)|. 
\end{eqnarray*}
Thus, 
\begin{eqnarray*}
|\lambda_d(\calk) - \lambda_{d + 1}(\calk)| & \geq  & |\lambda_d(A/\rho(n)) - \lambda_{d + 1}(A/\rho(n))| - | \lambda_d(A/\rho(n)) - \lambda_d(\calk)|  \\
& & \quad + |\lambda_{d + 1}(A/\rho(n))  -\lambda_{d + 1}(\calk)| \\
& = & \omega(t/\rho(n)) \quad \mbox{ (Using the way $\proc{DecideThreshold}(t, \rho(n))$ chooses $d$.)}
\end{eqnarray*}
\end{proof}

We have $\mathrm{dist}(S_1, S_2) \geq \delta_d - 2t/\rho(n) - 8\sqrt{2}\sqrt{\frac{\log(2/\epsilon)}{n}} \geq \delta_d /2$.

One can also see that the first $d$ eigenvalues of $K/n$ and $C(n)A/n$ are in $S_1$, and the rest are in $S_2$. Then by a Davis-Kahan theorem~\cite{DK:1970}, we have whp
\begin{equation}
\|\calp_A - \calp_K\| \leq \frac{\|C(n)A/n - K/n\|}{\mathrm{dist}(S_1, S_2)} = O\left(\frac{t}{\rho(n)\delta_d}\right). 
\end{equation}

\myparab{Step 3. Show that $\sqrt{C(n)}U_AS^{1/2}_A$ and $U_KS^{1/2}_K$ are close (up to a unitary operation).}
We first argue that $\|\calp_A (C(n)A) - \calp_K K\|$ is small. Before proceeding, let us re-scale the matrices so 
that their eigenvalues are in the same magnitude of those of $\calk$. We have $\calp_A(C(n)A) - \calp_K(K) = n\left(\calp_A(A/\rho(n)) - \calp_P(K/n)\right).$

Note that $\|K/n\| = O(1)$ and $A/\rho(n) = O(1)$ whp when $\rho = \omega(\log n)$. We have

$$\|\calp_A A/\rho(n) - \calp_K K/n \|  = \|(\calp_A - \calp_K) K/n\| + \| \calp_A (A/\rho(n) - K/n)\| = O\left(\frac{t}{\rho(n) \delta_d}\right).$$

Observing that $\calp_A A/\rho(n) = U_AS_AU^{\transpose}_A/\rho(n)$ and $\calp_K K/n = U_KS_KU^{\transpose}_K/n$, we see that $U_AS^{1/2}_A$ and $U_KS^{1/2}_K$ are ``square root'' of $\calp_A A/\rho(n)$ and $\calp_K K / n$ (up to scaling). We use the following lemma to relate $U_AS^{1/2}_A$ and $U_KS^{1/2}_K$  (Lemma A.1 from~\cite{tang2013}). 

\begin{lemma}\label{lem:xxyy}Let $A$ and $B$ be $n$ by $n$ positive semi-definite matrices with $\mathrm{rank}(A) = \mathrm{rank}(B) = d$. Let $X, Y \in \reals^{n, d}$ be  full column rank matrices such that $XX^\transpose = A$ and $YY^\transpose = B$. Let $\delta$ be the smallest non-zero eigenvalues of $B$. Then there exists a rotational matrix $W$ such that 
\begin{equation}
\|XW - Y\|_F \leq \frac{\|A - B\| (\sqrt{d\|A\|} + \sqrt{d\|B\|})}{\delta}. 
\end{equation}
\end{lemma}

By treating $A/\rho(n)$ and $K/n$ as $A$ and $B$ in the Lemma, we have 
\begin{equation}
\left\|U_AS^{1/2}_A/\sqrt{\rho(n)} - U_KS^{1/2}_K/\sqrt{n}\right\|_F = O\left(\frac{t \sqrt d}{\delta^2_d\rho(n)}\right)
\end{equation}
In other words, 
\begin{equation}
\left\|\sqrt{C(n)}U_AS^{1/2}_A- U_KS^{1/2}_K\right\|_F = O\left(\frac{t \sqrt{dn}}{\delta^2_d\rho(n)}\right)
\end{equation}

This completes our proof of Lemma~\ref{prop:latentestimatea}.

\subsection{Error source 3: truncation error}
This section analyzes the error $\| \Phi_d - \Phi\|^2_F$. Recall that we abuse the notation to let $\Phi_d \in \reals^{N_{\calh}}$ by ``padding'' 0's after the $d$-th coordinate. We make an additional assumption that the eigengaps $\delta_d = \lambda_d(\calk) - \lambda_{d + 1}$  monotonically decreases whenever the number of non-zero eigenvalues is infinite. Removing this assumption requires arduous analysis with limited insights. Section~\ref{sec:refined} presents an analysis without the assumption. 

We have $\E\|\Phi(x) - \Phi_d(x)\|^2 = \sum_{i > d}\E[(\sqrt{\lambda_i}\psi_i(x))^2] = \sum_{i > d}\lambda_i \int |\psi_i(x)|^2 dF(x) = \sum_{i > d}\lambda_i.$ Then we may apply a standard Chernoff bound to obtain $\|\Phi - \Phi_d\|_F = O(\sqrt{n}/(\sqrt{\sum_{i > d}\lambda_i)})$).  

In general small $\delta_d$ does not imply small tail \eg when $\lambda_i = \Theta(1/(i\log^2 i))$. Thus, we need to rely on the decay assumption in Theorem~\ref{thm:main}. \ie $\lambda_i(\calk) = O(i^{-2.5})$. One can see that when this condition is given, $\sum_{i > d}\lambda_i$ can by bounded by $\delta^{1/3}_d$ and $d = O(\delta^{1/2}_d)$. 

Together with Lemma~\ref{lem:kpcaexist} and Lemma~\ref{lem:aapproxk}, we have 
$$\| \hat \Phi - \Phi\|_F = O\left(\sqrt n \left(\frac{t\sqrt d}{\rho(n)\delta^2_d} + \delta^{1/6}_d\right)\right).$$

 By setting {\ifdraft\color{blue}\fi$\delta_d = \left(t/\rho(n)\right)^{12/29}$}, we have 
 \begin{equation}
\| \hat \Phi - \Phi\|_F= O\left(\sqrt n \left(\frac t {\rho(n)}\right)^{\frac 2{29}}\right)
\end{equation}
This completes the proof of Proposition~\ref{prop:latentestimate}.

\section{Estimation of $\Phi(X)$ in the  bipartite graph model}\label{sec:phiest}

\label{app:bipart}

This section explains how we can use $B$ to estimate $\Phi(\bfx)$. See $\proc{Bipartite-Est}(B)$ in Fig.~\ref{fig:fullalgo} for the pseudocode. 
A major difficulties in our analysis is that we cannot decouple the error into different approximation error sources like we did for the undirected graph case, \ie the approximation error sources interference with each other. So more involved analysis is needed.

\begin{figure}
\begin{codebox}
\Procname{$\proc{Bipartite-Est}(B)$}
\li \Comment \textbf{Step 1. Regulating $B$}.
\li $A \gets B^\transpose B$. 
\li $\mathrm{diag}(A) = \mathrm{diag}(A)^{\theta}$ 
\li \Comment $\theta < 1$ so  diagonal entries of $A$ are shrinked. 
\li \Comment \textbf{Step 2. PCA with data-driving thresholding}
\li $[\tilde U_A, \tilde S_A, \tilde V_A] = \mathrm{svd}(A)$. 
\li Let also $\lambda_i$ be $i$-th singular value of $A$.
\li $d \gets \max_d\{\lambda_{d} - \lambda_{d +1} > \left(\frac m n\right)^{12/43}$.
\li $S_A$: diagonal matrix comprised of $\{\lambda_i\}_{i \leq d}$
\li $U_A$, $V_A$: the corresponding singular vectors of $S_A$. 
\li Let $\hat \Phi_d = \frac{n^{3/4}}{m^{1/4}} U_AS_A^{1/4}$. 
\li \Return $\hat \Phi_d$. 
\end{codebox}
\caption{Estimation of $\Phi$ for bipartite graphs.}
\label{fig:fullalgo}
\end{figure}

Below is our main proposition. 

 \begin{proposition}\label{prop:bipartiteest} Consider the algorithm $\proc{Bipartite-Est}(\cdot)$. Let $\hat \Phi \in R^{N_{\calh}}$ be that its first $d$-coordinates coincide with $\hat \Phi_d$ returned by $\proc{Bipartite-Est}$ and the rest coordinates are $0$. If the eigenvalues of $\calk$ satisfies the decay condition, we have whp
 \begin{equation}
 \|\hat \Phi - \Phi\|_F = O\left(\sqrt n \left(\frac n m\right)^{2/43}\log n\right).
 \end{equation}
 \end{proposition} 
 
In other words, when $m = n \mathrm{poly}\log^c n$ for a suitably large $c$, then $\|\hat \Phi - \Phi\|_F \leq \sqrt n / \log^2n$. 

\myparab{Intuition of the algorithm.} Recall that $K \in R^{n \times n}$ such that $K_{i, j} = \kappa(x_i, x_j)$, $K$ is
the Gram matrix of the kernel $\kappa( \cdot, \cdot)$ obtained using the latent
positions of the influencers, $x_1, \ldots, x_n$.  Standard Kernel PCA results
suggest that as long as $n$ is sufficiently large, we may use $K$ to estimate
the eigenfunctions of $\calk$. 
But we do not directly observe the matrix $K$. Instead, we observe $B$ such
that $\E[B_{j,i}] = \kappa(y_j, x_i)$. In other words, our ``raw observations''
are about the relationship between followers $\{y_j\}_{j = 1}^m$ and influencers
$\{x_i \}_{i = 1}^n$, but our principal goal is to understand the relationships
within $\{x_i\}_{i = 1 }^n$. Our algorithm does so by computing $B^\transpose
B$. This product corresponds to another kernel. Specifically, let 
{\small
\begin{equation}
	\mu(x, x^\prime) = \int \kappa(x, z)\kappa(z, x^\prime) dF(z) \mbox{ and } \calm f(x) = \int \mu(x, y) f(y) dF(y)\label{eqn:kernelmu}
\end{equation}
}

Finally, let $M \in R^{n \times n}$ such that $M_{i, j} = \mu(x_i, x_j)$. For
$i \neq j$, one can see that $B^\transpose B$ and $M$ are related as follows: 
$$\E[(B^\transpose B)_{i,j}]= m \int \frac{\kappa(x_i, y)\kappa(y, x_j)}{n^2} dF(y) = \frac m {n^2}\E[M_{i,j}].$$
\myparab{Regularization of the diagonals.} Observing that $\E[(B^\transpose B)_{i,i}] \neq
\frac{m}{n^2} \E[M_{i,j}]$, we need to shrink the
diagonals of $B^\transpose B$ to construct $A$. Proposition~\ref{prop:bipartiteest} works for all $\theta < 0.75$ but 
for exposition purpose, we focus on only the case $\theta = -\infty$, \ie setting the whole diagonal to be $0$. One can use simple triangle 
inequalities on top of our techniques to analyze the general $\theta$ case. 

\myparab{The kernel $\mu$.} $\mu$ is a Mercer kernel as the Gram matrix
for any $\{x_i\}_{i =1}^{n}$ is positive definite; however, $\mu$ \emph{is not}
a radial-basis kernel. This can seen from the fact that $\mu$ depends on the
measure $F$. The quality of the isomap-based algorithm
presented in the next section crucially depends on the kernel being
a radial basis kernel (RBK).\footnote{The isomap-based algorithm will still work but the approximation
guarantee will be worse.} Thus, we need to find a way to reconstruct $\kappa$
from $\mu$, and reconstruct $\mu$ from $M$.

Note that  $\mu$ is a ``square'' of $\kappa$.

\begin{lemma} \label{lem:samespectrum}
	Consider the linear operators $\calk$ and $\calm$. Let $\{\psi_i \}_{i \geq
	1}$ and $\{\lambda_i\}_{i \geq 1}$ be the eigenfunctions and eigenvalues of
	$\calk$. Then the eigenfunctions and eigenvalues of $\calm$ are
	$\{\psi_i\}_{i \geq 1}$ and $\{\lambda^2_{i}\}_{i \geq 1}$, respectively. 
\end{lemma}

\begin{proof}
	Let $\psi$ be an eigenfunction of $\calk$ with eigenvalue $\lambda$. We can see that
	\begin{eqnarray*}
		\calm \psi(x) & = & \int \mu(x, y) \psi(y) dF(y) \\
		& = & \int \int \kappa(x, z) \kappa(z, y) dF(z) \psi(y) dF(y) \\
		& = & \int \kappa(x, z) \int \kappa(z, y) \psi_i(y) dF(y) dF(z) \\
		& = & \lambda \int \kappa(x, z) \psi(z) dF(z) \\
		& = & \lambda^2 \psi(x). 
	\end{eqnarray*} 
	
We can also verify that any function that is orthogonal to $\calk$ will also be
orthogonal to $\calm$, showing that the dimension of $\calk$ and $\calm$ are
the same.%
\end{proof}

We break down the analysis into smaller steps:
\begin{itemize}
\item \textbf{Step 0 (known results):} Given $K$, we can approximate $\Phi$. 
\item \textbf{Step 1.} If we have access to $\E
A \approx M$, then we can approximate $K^2$, \ie $M\propto K^2$ ($A
\propto B$ refer to that a suitable scalar $s$ exists such that $\|sA - B\| =
o(1)$). 
\item \textbf{Step 2.} Show that $A \propto M$ using Chernoff type
inequalities for matrices (together with step 1, we have $A \propto K^2$).
\item \textbf{Step 3.} Show that if $A \propto K^2$, then $A^{1/2} \propto K$ (note
that we need to be able to properly define taking the square root of $A$ as \eg
$A$ could have negative eigenvalues). Thus, we can construct $\hat \Phi_d$
from $A$. 
\item \textbf{Step 4.} Finally, argue that $\Phi_d$ approximates $\Phi$
well, \ie it is fine to truncate all the tail eigenvalues and eigenfunctions.
Thus, we can construct $\hat \Phi \in R^{N_{\calh}}$ by appending a suitable
number of $0$'s after the $d$-th coordinate so that $\hat \Phi$ approximates
$\Phi$ well.
\end{itemize}

We now walk through each step. In the proofs, constants $c_0, c_1$, etc.  are
used as ``intermediate variables.'' Constants that appear in different proof
should not be treated as the same unless stated explicitly.

\subsection{Notations and Step 0.}
Recall that we let $\tilde{U}_{K}\tilde{S}_{K}
\tilde{V}^\transpose_{K}$ ($\tilde{U}_{M}
\tilde{S}_{M} \tilde{V}^\transpose_{M}$ and $\tilde{U}_{A}
\tilde{S}_{A} \tilde{V}^\transpose_{A}$) be the SVD of $K$ ($M$ and $A$). Let $S_K$ ($S_M$ and $S_A$) be a $d
\times d$ diagonal matrix comprising the $d$-largest eigenvalues of $K$ ($M$ and $A$). Let
$U_K$ ($U_M$ and $U_A$) and $V_K$ ($V_M$ and $V_A$) be the corresponding singular vectors of $K$ ($M$ and $A$). Finally let $\bar K = U_KS_KV^\transpose_K$ ($\bar M = U_MS_MV^\transpose_M$ and $\bar A = U_AS_AV^\transpose_A$) be the low rank approximation of $K$ ($M$ and $A$). 


Recall that $\Phi$ is the feature map associated with $\calk$. We also let
$\Phi^{\calm}$ be the feature map associated with $\calm$ and $\Phi^{\calm}_d$
be the first $d$ coordinates of $\Phi^{\calm}$. For any $x \in [0, 1]$,
$\Phi_d(x)$ and $\Phi^\calm_d(x)$ may be viewed as vectors that satisfy
$\Phi^\calm_d(x) = S^{\sfrac{1}{2}} \Phi_d(x)$, where $S$ is a $d \times d$
diagonal matrix with $S_{ii} = \lambda_i(\calk)$.%

Existing results regarding kernel PCA~\cite{Rosasco2010,tang2013} state that if $K$ (or $M$) is
sufficiently large, then we are able to reconstruct $\Phi_d(X)$ (or
$\Phi^{\calm}_d(X)$) for the observed datapoints. Specifically, let $X = \{x_1,
\ldots, x_n \} \subseteq [0, 1]$ be the latent positions of the observed
datapoints and let $\Phi_d(X)$ and $\Phi^\calm_d(X)$ be the $n \times d$
matrices, where the $i\th$ row of $\Phi_d(X)$ and $\Phi^{\calm}_d(X)$ are $\Phi_d(x_i)$ and  $\Phi^\calm_d(x_i)$, respectively.

We have with probability $\geq 1 - \epsilon$,%
	\begin{equation}\label{eqn:kpca}
		\| U_K S^{\sfrac{1}{2}}_K W - \Phi_d(X) \|_F \leq 2\sqrt 2
		\frac{\sqrt{\log \epsilon^{-1}}}{\lambda_d(\calk) - \lambda_{d + 1}(\calk)}, 
	\end{equation}
	where $W$ is an orthogonal matrix.
	Similar results also hold for $M$ (see \eg~\cite{tang2013}). Furthermore, $S_K$ (and $S_M$) are approximations of
	the eigenvalues of $\calk$ (and $\calm$), \ie $(S_K)_{i,i}/n \rightarrow
	\lambda_i(\calk)$ ($(S_M)_{i,i}/n \rightarrow \lambda_i(\calm)$). The specific
	convergence rate is stated in Theorem~\ref{thm:normclose}.

\subsection{Step 1. From $M$ to $\calm$ and $\calk$.} 
	Using the above facts, we know that (these are hand-waving arguments to
	deliver intuitions; formal treatment will be presented below) (1) $U_K$ and
	$U_M$ ``approximate'' the eigenfunctions of $\calk$ and $\calm$ respectively;
	but eigenfunctions of $\calk$ and $\calm$ are the same so $U_K$ and $U_M$ are
	close. (2) $\lambda_i(K) / n \approx \lambda_i(\calk)$, $\lambda_i(M) /n
	\approx \lambda_i(\calm)$, and $\lambda_i(\calm) = \lambda^2_i(\calk)$. Thus,
	we roughly have $\sqrt{\lambda_i(M)/n} = \lambda_i(K)/n$. These two
	observations imply we may have $S_K/n \approx \sqrt{ S_{M}/n}$ and thus, $\sqrt
	n \cdot U_M S^{\sfrac{1}{2}}_M U^\transpose_M \approx U_K S_K U^\transpose_K$. 

	We now formalize the intuition. Our main goal is to prove the following proposition.

\begin{proposition}  \label{prop:K2Mclose}
	Let $K$ and $M$ be the matrices defined above. Let $d \in \naturals$ and
	$\tilde{\delta}_d \in \reals^+$ be such that if $(\lambda_i(K))_{i = 1}^n$
	and $(\lambda_i(M))_{i = 1}^n$ are the eigenvalues of $K$ and $M$,
	respectively, then $\lambda_{i}(K) - \lambda_{i + 1}(K) \geq
	\tilde{\delta}_d$ and $\lambda_i(M) -  \lambda_{i + 1}(M) \geq
	\tilde{\delta}_d$ for $i = 1, \ldots, d - 1$. Let $\calp_K = U_K
	U_K^\transpose$ and $\calp_M = U_M U_M^\transpose$ be projection operators
	onto the linear subspaces spanned by the eigenvectors corresponding to the
	$d$ largest eigenvalues of $K$ and $M$, respectively.  Then with probability
	at least $1 - 4 \epsilon$, 
	\begin{equation}
		\left\|\frac{\calp_K K^2}{n^2} - \frac{\calp_M M }{n} \right\| = O \left(
		\left(\frac{d^2 \log(1/\epsilon)}{n} \right)^{\sfrac{1}{4}} + \left(
		\frac{d \log(1/\epsilon)}{ \tilde{\delta}_d^2 n} \right)^{\sfrac{1}{2}}
		\right)
	\end{equation}
\end{proposition}

Our analysis consists of two parts:  (1) Show that $\frac{U_K S_K W_K}{n}$ and $\frac{U_MS^{\sfrac{1}{2}}_M
	W_M}{\sqrt{n}}$ are ``close'' , where $W_K$ and $W_M$ are orthogonal
	matrices, and (2) Show that if two matrices $X$ and $Y$ are close, then $XX^\transpose$
	and $YY^\transpose$ are also close \Znote{moved up to Lemma~\ref{lem:product}}.

\myparab{Part 1 of proof of Proposition~\ref{prop:K2Mclose}.} We shall show the following lemma. 

	\begin{lemma}\label{lem:k2m}
		Using the notation defined above, we have with probability at least $1 - 4 \epsilon$, 
		{\footnotesize
		\begin{equation}\label{eqn:dis1}
			\left\|\frac{U_KS_KW_K}{n} - \frac{\Phi_d(X) S^{\sfrac{1}{2}}}{\sqrt n}\right\|_F
			\leq 
			C \left( \left(\frac{d^2
			\log(1/\epsilon)}{n} \right)^{\sfrac{1}{4}} + \left( \frac{d
			\log(1/\epsilon)}{ \tilde{\delta}_d^2 n} \right)^{\sfrac{1}{2}}
			\right)
		\end{equation}
		\begin{equation}\label{eqn:dis2}
			\left\|\frac{U_MS^{\sfrac{1}{2}}_MW_M}{\sqrt n} - \frac{\Phi_d(X)
			S^{\sfrac{1}{2}}}{\sqrt n}\right\|_F \leq \frac{4}{\tilde{\delta}_d}
			\sqrt{\frac{2 \log(1/\epsilon)}{n}}.
		\end{equation}}
	\end{lemma}
	\begin{proof}[Proof of Lemma~\ref{lem:k2m}] 
		Let $\calh$ be the Hilbert space corresponding to the kernel $\kappa(
		\cdot, \cdot)$. Then, we define the following two positive symmetric
		linear operators that act on $\calh$. 
{\small
\begin{equation}
\calk_\calh \eta = \int \langle \eta, \kappa( \cdot, x) \rangle_\calh \kappa( \cdot, x) dF(x) \mbox{ \& }
\calk_n \eta = \frac{1}{n} \sum_{i = 1}^n \langle \eta, \kappa(\cdot, x_i) \rangle_\calh \kappa(\cdot, x_i)
\end{equation}
}
		We note that the operators $\calk_\calh$ and $\calk_n$ are closely
		related to $\calk$ and $K$ respectively, however while $\calk_\calh$ and
		$\calk_n$ both act on $\calh$, $\calk$ acts on $L^2(\X, F)$ and $K$ acts
		on $\reals^n$. The eigenvalues of $\calk_\calh$ and $\calk$ are the same,
		and the eigenvalues of $\calk_n$ and $K/n$ are the same. Furthermore, if
		$\psi$ is an eigenfunction of $\calk$ with eigenvalue $\lambda$, then
		$\sqrt{\lambda}\psi$ is an eigenfunction of $\calk_\calh$ with eigenvalue
		$\lambda$; the $\sqrt{\lambda}$ factor is required to ensure that the
		norm in $\calh$ of the eigenfunction is $1$.  Similarly if $\hat{u} \in
		\reals^n$ is an eigenvector of $K/n$ with eigenvalue $\hat{\lambda}$,
		then $\hat{v}(\cdot) = \frac{1}{\sqrt{\hat{\lambda} n}} \sum_{i = 1}^n
		\kappa(\cdot, x_i) \hat{u}_i$ is an eigenfunction of $\calk_n$ with
		eigenvalue $\hat{\lambda}$. See also~\cite{Rosasco2010}. 

		For some $r \leq d$, let $\lambda_r$ and $\hat{\lambda}_r$ be the $r\th$
		largest eigenvalues of $\calk_\calh$ and $\calk_n$, respectively. Denote
		by $\calp_r$ and $\hat{\calp}_r$ the projection operators on to the
		corresponding eigenfunctions. We will use the following Theorem, which 
		generalizes Davis-Kahan sin theorem to linear operators. 
		
\begin{theorem}[Thm. 2~\cite{ZB:2005}]\label{thm:zb}
			Let $\cala$ and $\calb$ be symmetric positive Hilbert-Schmidt
			operators on some Hilbert space.  Let $\tilde{\delta}_d > 0$ and $d
			\in \naturals$ such that
			\begin{enumerate}
				\item For all $i < d$, $\lambda_i(\cala) - \lambda_{i + 1}(\cala) \geq \tilde{\delta}_d$.
				\item $\| \cala - \calb \|_{\HS} \leq \frac{\tilde{\delta}_d}{4}$.
			\end{enumerate}
			Let $\calp^\cala_r$ and $\calp^\calb_r$ be projection operators that
			project onto the eigenfunctions corresponding to the $r\th$ largest
			eigenvalue of $\cala$ and $\calb$, respectively. Then,
			\begin{align}
				\| \calp^\cala_r - \calp^\calb_r \|_{\HS} &\leq \frac{2 \| \cala - \calb \|_{\HS}}{\tilde{\delta}_d} \label{eqn:projops}
			\end{align}
		\end{theorem}

		Applying the theorem with $\calk_\calh$ and $\calk_n$ taking the
		role of $\cala$ and $\calb$ and recalling that $\calp_r$ and
		$\hat{\calp}_r$ are the corresponding projection operators, we have,
		$\| \calp_r - \hat{\calp}_r \|_\HS \leq 2 \left( \left\| \calk_\calh -
			\calk_n \right \|_\HS\right)/\tilde{\delta_d}$. It can be shown that with probability (over the random draw
			of $\{ x_1, \ldots, x_n \}$) at least $1 - 2 \epsilon$, $\left \|
			\calk_\calh - \calk_n \right \|_\HS \leq 2 \sqrt{\frac{2
			\log(1/\epsilon)}{n}}$; the proof of this claim appears in
			Theorem~B.2~\cite{tang2013}. Thus, we get 
			$\| \calp_r - \hat{\calp}_r \|_\HS \leq \frac{4}{\tilde{\delta}_d}
			\sqrt{\frac{2 \log(1 / \epsilon) }{n}}.$ Thus, for any $x \in \X$, we have
			{\small
			$$\| \calp_r \kappa( \cdot, x) - \hat{\calp}_r \kappa (\cdot, x) \|_\calh
			\leq \| \calp_r - \hat{\calp}_r\|_{\HS} \| \kappa(\cdot, x)
			\|_{\calh} \leq \frac{4}{\tilde{\delta}_d} \sqrt{\frac{2 \log(1 /
			\epsilon) }{n}}. $$}

		Recall that $ \calp_K K = U_K S_K U^\transpose_K$ and let $(U_K)_{:, r}$
		the $r\th$ column of $U_K$ be the eigenvector of $K$ corresponding to the
		$r\th$ largest eigenvalue. Note that the corresponding eigenvalue
		$(S_K)_{rr} = n \hat{\lambda}_r$ (the factor $n$ appears because
		$\hat{\lambda}_r$ is the eigenvalue of $K/n$). Then $K^{(r)} := (U_K)_{:,
		r} (U_K)_{:, r}^\transpose K = n \hat{\lambda}_r (U_K)_{:, r} (U_K)_{:,
		r}^\transpose$ denotes the projection of $K$ on to the space
		corresponding to the $r\th$ eigenvector. We have:
		\begin{lemma}\label{lem:34}$K^{(\ell)}_{i, j} = \langle \hat{\calp}_\ell \kappa(\cdot, x_i), \hat{\calp}_\ell
			\kappa( \cdot, x_j) \rangle_\calh$.
		\end{lemma}
		The proof can be found in, \eg Lemma 3.4 in~\cite{tang2013}. For completeness, we repeat the arguments here. 
		\begin{proof}[Proof of Lemma~\ref{lem:34}] Let $\Psi_{r, n} \in R^n$ be the vector whose entries are 
		$\sqrt{\lambda_r} \psi_r(x_i)$ for $i \in [n]$. We have $K = \sum_{r \geq 1}\Psi_{r, n}\Psi^{\transpose}_{r, n}$. Recall 
		that $\hat u^{(1)}, ..., \hat u^{(d)}$ are eigenvectors associated with the $d$ largest eigenvalues of $K/n$, we have
		\begin{equation}
		K^{(s)} = \sum_{r \geq 1}\hat u^{(s)}\big(\hat u^{(s)}\big)^{\transpose} \Psi_{r, n}\Psi^{\transpose}_{r, n}\hat u^{(s)}\big(\hat u^{(s)}\big)^{\transpose}
		\end{equation}
		Thus, we have 
		$$K^{(s)}_{i, j} = \hat u^{(s)}_i \big(\hat u^{(s)}\big)^{\transpose} \Psi_{r, n}\Psi^{\transpose}_{r, n}\hat u^{(s)}\big(\hat u^{(s)}_j\big)^{\transpose}.$$
		Recall that 
		$$\hat v^{(i)}(\cdot) = \frac{1}{\sqrt{\hat \lambda_i n}} \sum_{i = 1}^n\kappa(\cdot, x_i) \hat u_i.$$
		We have for any $s \in [d]$:
		
\begin{eqnarray*}
\langle \hat v^{(s)}, \sqrt{\lambda_r} \psi_r\rangle_{\calh} & = & \left\langle \frac{1}{\sqrt{\hat \lambda_s n}}\sum_{i = 1}^n\kappa(\cdot, x_i)\hat u^{(s)}_i, \sqrt{\lambda_r}\psi_r\right\rangle_{\calh} \\
& = & \left\langle \frac{1}{\sqrt{\hat \lambda_s n}}\sum_{i = 1}^n \sum_{r'\geq 1} \sqrt{\lambda_{r'}}\psi_{r'}(x_i)\sqrt{\lambda_{r'}}\psi_{r'}\hat u^{(s)}, \sqrt{\lambda_r}\psi_r \right\rangle_{\calh} \\
& = & \frac{1}{\sqrt{\hat \lambda_s n}} \sum_{i = 1}^n \psi_r(x_i)\sqrt{\lambda_r}\hat u^{(s)}_i \\
& = & \frac{1}{\sqrt{\hat \lambda_s n}}\langle \hat u^{(s)}, \Psi_{r, n}\rangle_{R^n}.
\end{eqnarray*}

This implies
\begin{equation}
\hat u^{(s)}_i \big(\hat u^{(s)}\big)^{\transpose} \Psi_{r, n} = \hat u^{(s)}_i \langle \hat u^{(s)}, \Psi_{r, n}\rangle_{R^n} = \hat v^{(s)}(x_i) \langle \hat v^{(s)}, \sqrt{\lambda_r}\psi_r\rangle_{\calh}.
\end{equation}

Next, let $\xi^{(s)}(x) = \sum_{r \geq 1}\langle \hat v^{(s)}, \psi_r \sqrt{\lambda_r}\rangle_{\calh}\hat v^{(s)}(x)\sqrt{\lambda_r}\psi_r \in \calh$. We have
\begin{equation}
K^{(\ell)}_{i, j} = \sum_{r \geq 1}\hat u^{(\ell)}_i \big(\hat u^{(\ell)}\big)^{\transpose} \Psi_{r, n} \Psi^{\transpose}_{r, n}\hat u^{(\ell)}\hat u^{(\ell)}_j = \langle \xi^{(\ell)}(x_i), \xi^{(\ell)}(x_j)\rangle_{\calh}
\end{equation}
We then use the reproducing kernel property of $\kappa(., x)$: 
\begin{eqnarray*}
\xi^{(\ell)}(s) & = & \sum_{r = 1}^{\infty}\langle \hat v^{(s)}, \psi_r \sqrt{\lambda_r}\rangle_{\calh} \hat v^{(s)}(x) \sqrt{\lambda_r}\psi_r \\
& = & \sum_{r = 1}^{\infty} \langle \hat v^{(s)}, \psi_r \sqrt{\lambda_r}\rangle_{\calh}\langle \hat v^{(s)}, \kappa(\cdot, x)\rangle_{\calh}\sqrt{\lambda_r}\psi_r \\
& = & \langle \hat v^{(s)}, \kappa(\cdot, x) \rangle_{\calh}\sum_{r = 1}^{\infty} \langle \hat v^{(s)}, \psi_r \sqrt{\lambda_r}\rangle_{\calh} \sqrt \lambda_r \psi_r \\
& = & \langle \hat v^{(s)}, \kappa(\cdot, x) \rangle_{\calh} \hat v^{(s)}.
\end{eqnarray*}
Finally, we have
\begin{eqnarray*}
K^{(\ell)}_{i, j} & = & \langle \xi^{(\ell)}(x_i), \xi^{(s)}(x_j) \rangle_{\calh} \\
& = &  \langle \hat v^{(\ell)}, \kappa(\cdot, x_i) \rangle_{\calh}\langle \hat v^{(\ell)}, \hat v^{(\ell)}\rangle_{\calh}\langle \hat v^{(\ell)}, \kappa(\cdot, x_j) \rangle_{\calh} \\
 & = & \left\langle \langle \hat v^{(\ell)}, \kappa(\cdot, x_i)\rangle_{\calh}\hat v^{(\ell)}, \langle \hat v^{(s)}, \kappa(\cdot, x_j)\rangle_{\calh} \hat v^{(\ell)} \right\rangle_{\calh} \\
 & = & \langle \hat \calp_{\ell}\kappa(\cdot, x_i), \hat \calp_{\ell}\kappa(\cdot, x_j)\rangle_{\calh}.
\end{eqnarray*}

		\end{proof}

		It then follows that, there exists a $w_r \in \{-1, 1\}$, such that
		$\sqrt{\hat{\lambda}_r n} (U_K)_{:, r} w_r$ corresponds to an isometric
		isomorphism from the one dimensional subspace of the Hilbert space
		$\calh$ under $\hat{\calp}_r$ to $\reals$, for the datapoints $x_1,
		\ldots, x_n$.  Similarly, the projection under operator $\calp_r$
		corresponds to the vector $\Phi_{:, r}(X) := \sqrt{\lambda}_r
		[\psi_r(x_1), \ldots, \psi_r(x_n)]^\transpose$.  Thus, we have that:
			$\| (U_K)_{:, r} \sqrt{\hat{\lambda}_r n} \cdot w_r - \Phi_{:, r}(X)\| \leq
			\frac{4}{\tilde{\delta}_d} \sqrt{2 \log(1 / \epsilon)} $
		Applying the above to all $r \leq d$ and writing succinctly, we get:
		\begin{align}
			\left\| U_K S^{\sfrac{1}{2}}_K W_K - \Phi_d(X) \right\|_F &\leq 
			\frac{4}{\tilde{\delta}_d} \sqrt{2 d \log(1 / \epsilon)}.
			\label{eqn:relatefeatureK}
		\end{align}
		Above $W_K$ is a diagonal orthogonal matrix, \ie every diagonal element
		is $\pm 1$. 
		
Using the fact that $S^{\sfrac{1}{2}}_K$,
			$S^{\sfrac{1}{2}}$ and $W_K$ are all diagonal and hence commute and
			that $\| AB \|_F \leq \min \{ \| A \|_F \| B \|, \| A \|\| B \|_F
			\}$, we get

{\footnotesize
		\begin{align*}
			%
			%
			%
			\left\| \frac{U_K S_K W_K}{n}  - \frac{\Phi_d(X) S^{\sfrac{1}{2}}}{\sqrt{n}} \right\|_F &\leq 
			\left\| \frac{U_K S^{\sfrac{1}{2}}_K W_K}{\sqrt{n}} \right\|_F \cdot 
			\left\| \frac{S_K^{\sfrac{1}{2}}}{\sqrt{n}} - S^{\sfrac{1}{2}} \right\| \\
			& \quad +
			\left\| \frac{U_K S^{\sfrac{1}{2}}_K W_K}{\sqrt{n}} - \frac{\Phi_d(X)}{\sqrt{n}} \right\|_F  \cdot
			\left\| S^{\sfrac{1}{2}} \right\|.
		\end{align*}}

		Next, we note that $\left\| \frac{U_K S_K^{\sfrac{1}{2}} W_K}{\sqrt{n}}
		\right \|_F \leq \| U_K \|_F \| S_K^{\sfrac{1}{2}}/\sqrt{n} \| \| W_K \|
		\leq \sqrt{d}$; as $\hat{\lambda}_i = (S_K)_{ii} \leq 1$ for all $i$ and
		$W_K$ is an orthogonal matrix.  Also, $\left \|
		\frac{S_K^{\sfrac{1}{2}}}{\sqrt{n}} - S^{\sfrac{1}{2}} \right \| \leq
		\max_{i \leq d} |\sqrt{\lambda_i} - \sqrt{\hat{\lambda}_i}| \leq \max_{i
		\leq d} \sqrt{|\lambda_i - \hat{\lambda}_i|}$, where $\lambda_i$ and
		$\hat{\lambda}_i$ are the $i\th$ largest eigenvalues of the operators
		$\calk_\calh$ and $\calk_n$, respectively. (We also use that for $a, b >
		0$, $|\sqrt{a} - \sqrt{b}| \leq \sqrt{|a - b|}$). We know using Theorem
		B.2 from~\cite{tang2013} that $\max_{i} | \lambda_i - \hat{\lambda}_i|
		\leq 2 \sqrt{\frac{2 \log(1/\epsilon)}{n}}$ with probability at least $1
		- 2\epsilon$. For the second term, we use~\eqref{eqn:relatefeatureK} and
		the fact that $\| S^{\sfrac{1}{2}} \| \leq \max_{i \leq d} \lambda_i \leq
		1$. Putting everything together and simplifying, we get that for some
		constant $C$,
		\begin{align}
			\left \| \frac{U_K S_K W_K}{n} - \frac{\Phi_d(X)
			S^{\sfrac{1}{2}}}{\sqrt{n}} \right \|_F \leq C \left( \left(\frac{d^2
			\log(1/\epsilon)}{n} \right)^{\sfrac{1}{4}} + \left( \frac{d
			\log(1/\epsilon)}{ \tilde{\delta}_d^2 n} \right)^{\sfrac{1}{2}}
			\right)
		\end{align}
		We can prove the statement regarding $M$, by obtaining the equivalent
		of~\eqref{eqn:relatefeatureK} for $M$. This completes the proof of the
		lemma.
	\end{proof}

\myparab{Part 2 of Proposition~\ref{prop:K2Mclose}.} For part 2, we need the following lemma. 

	\begin{lemma}
		\label{lem:product} Let $X, Y \in R^{m \times d}$, $\|X\|$ and $\|Y\|$
		are bounded by $c_1$, and $\|X - Y\|_F \leq \epsilon$, we have 
		\begin{equation}
			\| XX^\transpose - YY^\transpose\|_F \leq 2 c_1 \epsilon.
		\end{equation}
	\end{lemma}
	\begin{proof}We have 
		\begin{eqnarray*}
			\| XX^\transpose - YY^\transpose\|_F & \leq & \|XX^\transpose -
			XY^\transpose + XY^\transpose - YY^\transpose\|_F \\
			&  \leq  & \|X\|_2
			\|X^\transpose-Y^\transpose\|_F + \|Y\|_2 \|X - Y\|_F \\
			& \leq & 2c_1
			\epsilon. 
		\end{eqnarray*}
	\end{proof}

Finally, using Lemmas~\ref{lem:k2m} and~\ref{lem:product} together finishes the proof
	of Prop.~\ref{prop:K2Mclose}.

\subsection{Step 2. $A \propto M$.}
We next show that the spectral norm of the difference between $\frac{n}{m} A$
and $\frac{1}{n}M$ is small by making use of suitable matrix tail inequalities.

	Before we proceed, we comment on the reason we decided to bound the difference
	between $A$ and $M$, instead of the difference between $B$ and $K$. One
	commonly used approach to analyze directed graphs (\ie the $B$ matrix) is to
	use standard Chernoff-type inequalities for matrices to bound $\|B - \E B\|$
	via the ``symmetrization'' trick, \ie by considering the symmetric matrix
	{\footnotesize $\left(\begin{array}{cc}0 & B \\ B^\transpose & 0\end{array}\right)$}~\cite{Dhillon2001,rohe2016co}. One
	drawback of this trick is that it requires both the (average) in- and out-
	degree to be in the order of $\Omega(\log n)$. This requirement is not
	satisfied in our model, because the average out-degree of followers is a constant.
	In fact, this is not the problem of the quality of Chernoff bound. Instead,
	$\|B - \E B\|$ could be large when $B$ is tall and thin.  

\myparab{Example. $\|B - \E B\|$ can be large.}
	Let us consider a
	simplified example where $B \in R^{m \times 1}$ and entries are independent
	r.v. from $\{-1, 1\}$. Each value appears with 0.5 probability. Note that
	$\E B = 0$. Thus, $\|B - \E B\|_2 = \|B\|_F = \sqrt n$ (which is considered
	to be large). Our product trick can  address the issue directly: if we take the
	product of $B$, we notice that $\|B^\transpose B - \E[B^\transpose B]\| = 0$ (note that $B^\transpose B$ degenerates to a scalar) 
	so the spectral gap
	between $BB^\transpose$ and $\E[B^\transpose B]$ is significantly reduced. 

{\ifdraft\color{blue}\fi
\myparab{Example. Singular vectors of $B$ do not converge to eigenvectors of $M$.} 
Observe that the right singular vectors of $B$ are the same as the eigenvectors of $B^{\transpose}B$. 
 If we fix $n$ and let $m \rightarrow \infty$, we have $\frac 1 m B^{\transpose}B \rightarrow \E B^{\transpose}B$. 
 Because the diagonal of $B^{\transpose}B$ is not proportional to that of $M$, the eigenvectors of $\E B^{\transpose}B$ 
 are different from those of $M$ unless $\mathrm{diag}(\frac{n^2}{m}B^{\transpose}B) -\mathrm{diag}(M) \propto I$, which 
 often is not true. Thus, we cannot use singular vectors of $B$ to find $U_MS^{1/4}_M$, which is used to approximate $\Phi_d(\bfx)$.  
}

	 Our goal is to prove the following lemma. 

\Znote{We will use a new statement.} 
{\ifdraft\color{blue}\fi

\begin{lemma}\label{lem:amclose} 
	Using the notation defined above, if $m = \Omega(n\log^cn)$ for a suitably large constant $c$, we have 
	with high probability
	\begin{equation}
		\left\|\frac n m A - \frac 1 n M\right\| = O\left(\left(\frac n m\right)^{1/4}\right).
	\end{equation}
\end{lemma}

}

\begin{proof}
	We let $Z_j = \sqrt n B_j$ (viewed as a column vector in $\reals^n$), where
	$B_j$ refers to the $j\th$ row of $B$ (\ie this vector encodes the
	connectivity of $y_j$). Note that $Z_j$ are i.i.d.  conditioned on knowing
	the latent variables $(x_i)_{i = 1}^n$. 

	There are two sources of randomness for each $Z_j$: 
	(1) Random selection of latent position of $y_j$ drawn according to $F$, and (2) 
	Random realizations of edges given $x_i$ and $y_j$. 
	
	As the $Z_j$ are identically distributed, we can calculate the expected
	behavior of $Z_1Z^\transpose_1$. 
	We have already seen that for $i \neq j$, $\E[(Z_1Z^\transpose_1)_{i,j}] = \frac{1}{n} \mu(x_i,
	x_j)$. We can also see that
	$ \E [(Z_1 Z^\transpose_1)_{i, i}] = \int \E[(Z_1)^2_i \mid x_i] dF(y) = \int
	\kappa(y, x_i) dF(y) \neq \frac 1 n \mu(x_i, x_i).
	$

	We first a use triangle inequality to ``decouple'' diagonal entries from the
	off-diagonal entries. Let $\tilde M := M - \diag(M)$ be the matrix with $M$
	with diagonal entries set to $0$. We have,
	\begin{align}
		\left\|\frac n m A - \frac 1 n M\right\| &\leq \left\|\frac n m A - \frac
		1 n \tilde M \right\| + \frac 1 n\|\tilde M - M\| \label{eqn:splitdiag}
	\end{align}

	The second term is straightforward to bound as $\frac 1 n (\tilde M - M )$ is
	a diagonal matrix, with each diagonal element of order $\frac 1 n$. Thus, we
	have $ \frac 1 n \|\tilde M - M\| = O(1/n)$. 

	In order to bound the first term of~\eqref{eqn:splitdiag}, note that $\frac
	n m A = \frac n m B^\transpose B - \frac n m\diag(B^\transpose B) $ and
	$\frac 1 n \tilde M = \E[Z_1 Z^\transpose_1] - \diag(\E[Z_1Z^\transpose_1])$.
	We have 
	{\small
	\begin{align}
		\left\| (n /m) A - ( 1 /n) \bar M\right\| &\leq \left\|( n /m)
		B^\transpose B - \E[Z_1Z^\transpose_1]\right\| \nonumber \\ 
		& \quad + \left\|( n
		/m)\diag(B^\transpose B) - \diag(\E[Z_1Z^\transpose_1])\right\|
		\label{eqn:splidiag2}
	\end{align}
}
	We use a matrix inequality for sum of low rank matrices to bound the first
	term and a standard Chernoff bound for the second term. 

	First, let us bound the (easier) second term $\|(n/m)\diag(B^\transpose B) -
	\diag (\E[Z_1Z^\transpose_1])\|$. Our crucial observation here is that
	$(B^\transpose B)_{i,i} = \sum_j B^2_{j,i} = \sum_j B_{j, i}$. Thus each
	entry on the $\diag(B^\transpose B)$ is a sum of i.i.d. Bernoulli random
	variables. We note that $\E[B_{j, i}] = \Theta(1/n)$. We may thus apply a
	Chernoff bound on each element of the diagonal and a union bound over all
	$n$ diagonal elements to get that there exists a constant $c_2$, such that: 
	\begin{equation}\label{eqn:diag}
		\Pr\left[\left\|\diag(B^\transpose B) \frac n m - \diag \E[Z_1
		Z^\transpose_1]\right\| \geq c_2 \sqrt{\log(1/\epsilon)\cdot\frac n
		m}\right] \leq \epsilon
	\end{equation}

	We use the following matrix inequality in Lemma~\ref{lem:matrixinq}~\cite{Oli:2010} to bound
	$\left\|\frac n m B^\transpose B - \E[Z_1Z^\transpose_1]\right\|$.

	\begin{lemma} \label{lem:matrixinq}[Lemma 1 in~\cite{oliveira2010}]
		Let $Z_1, ..., Z_m$ be i.i.d. random column vectors in $\reals^d$, with
		$|Z_i|\leq \alpha$ a.s. and $\| \E Z_i Z^\transpose_i\| \leq \beta$. Then
		we have for any $t \geq 0$:
		\begin{equation}\label{eqn:matrixinq}
			\Pr\left[\left\|\frac 1 m \sum_{i = 1}^m Z_iZ^\transpose_i - \E[Z_1
			Z^\transpose_1]\right\| \geq t\right] \leq (2m^2)
			\exp\left(-\frac{mt^2}{16\beta \alpha^2 + 8\alpha^2t}\right)
		\end{equation}
	\end{lemma}%

\Znote{This is the new text.}
{
\ifdraft\color{blue}\fi
We will apply the above lemma to obtain the required result. Observe that
	$\| Z_i \|^2 = n \sum_{j = 1}^n B_{i, j}$, where $B_{j, i}$ are i.i.d.
	Bernoulli random variables and $\E[B_{j, i}] = \Theta(1/n)$. One can see that 
	$|Z_i| \leq \log^2 n \sqrt n$ a.s. But as $n \rightarrow \infty$ both $n$ and $m$ grow simultaneous
	 so a more careful analysis is needed. 
	 
Here, we do not directly work with $Z_i$. Instead, we couple $Z_i$ with another group of bounded random 
variables so that we may use the matrix trail inequality directly. Specifically, we define the coupled process as follows. 

\begin{enumerate}
\item Sample $C_i$ from the distribution that's identical to $\sum_{j = 1}^nB_{i,j}$. Then let $\tilde C_i = \min\{C_i, \zeta(n)\}$ and $H_i =  I(C_i \geq \zeta(n))$. 
\item Sample $\tilde B_i$ from the distribution $B_i | (|B_i|_1 = \tilde C_i)$ (interpreted as ``sample $B_i$ conditioned on knowing $|B_i|_1$''). and sample $B_i$ from the distribution $B_i | |B_i|_1 = C_i$. 
\item Set $\tilde Z_i = \sqrt n \tilde B_i$ and $Z_i = \sqrt n B_i$. 
\end{enumerate}
Let us also set $R_i = Z_i - \tilde Z_i$ and $S = \E[Z_1 Z^{\transpose}_1] - \E[\tilde Z_1 \tilde Z^{\transpose}_1] $. 
One can see that $\{\tilde Z_i\}_{i \leq n}$ are independent and the statistical difference between $\tilde Z_i$ and $Z_i$ is $O(\zeta(n))$ because $\Pr[\sum_{j = 1}^nB_{i,j} > \zeta(n)] = O(\zeta(n))$. This also implies $S = n^{-\omega(1)}$. 

Let $t$ be a parameter to be decided later. We have
\begin{eqnarray*}
& & \Pr\left[\left\|\frac 1 m \sum_{i = 1}^m Z_iZ^\transpose_i - \E[Z_1
			Z^\transpose_1]\right\| \geq t \right] \\
 & \leq & 
			\Pr\left[\left\|\frac 1 m \sum_{i = 1}^m \tilde Z_i \tilde Z^\transpose_i - \E[\tilde Z_1
			\tilde Z^\transpose_1]\right\| \geq\frac t 2 \right] + \Pr\left[\left\|\frac 1 m \sum_{i = 1}^{m}R_i - S\right\|\geq \frac t 2\right]
\end{eqnarray*}
Using the fact that $S = n^{-\omega(1)}$, it is simple to bound the second term: 
\begin{eqnarray*}
\Pr\left[\left\|\frac 1 m \sum_{i = 1}^{m}R_i - S\right\|\geq \frac t 2\right] & \leq & \Pr\left[\left\|\frac 1 m \sum_{i = 1}^{m}R_i \right\|\geq \frac t 4\right] \\
& \leq & \sum_{i = 1}^m\Pr[R_i \neq 0] = \sum_{i = 1}^n \E H_i = n^{-\omega(1)}.
\end{eqnarray*}

We next use Lemma~\ref{lem:matrixinq} to bound the first term. We have $|\tilde Z_i| \leq \sqrt{n \zeta(n)}$. Next, we observe that since for $i
	\neq j$, $(\E Z_1 Z_1^\transpose)_{i, j} = \frac{1}{n} \mu(x_i, x_j)$ and
	$(\E Z_1 Z_1^\transpose)_{i, i} = \int \kappa(y, x_i) dF(y) \leq \sup_{y} \{
		\kappa(y, x) \}$, $ \| \E Z_1 Z^\transpose_1 \| = O \left(\sup_{x,
	x^\prime} \mu(x, x^\prime) + \sup_{x, x^\prime} \kappa(x, x^\prime) \right)
	= O(1)$. Thus, $\| \E \tilde Z_1 \tilde Z^{\transpose}_1\| = O(1)$.

We set 
	$ t = \frac{c_1}{\left(\frac m n\right)^{1/4}}. $
	Assuming $m/n = \omega(\log^cn)$ and for $c$ and $c_1$ chosen suitably,
	Lemma~\ref{lem:matrixinq} gives us,%
	
$$\Pr\left[\left\|\frac 1 m \sum_{i = 1}^m \tilde Z_i \tilde Z^\transpose_i - \E[\tilde Z_1
			\tilde Z^\transpose_1]\right\| \geq\frac t 2 \right] = n^{-\omega(1)}.$$	

	Together with (\ref{eqn:diag}), this completes the proof of the lemma, we complete the proof of Lemma~\ref{lem:amclose}.
}

\end{proof}
\myparab{Remark. Eigengaps for different operators.} From $A \propto M$ and how we decide $d$ in our algorithm, we can bound the eigengaps between different matrices/operators by using Theorem~\ref{thm:normclose} and Theorem~\ref{thm:zb}: 
{\small
\begin{equation}\label{eqn:gapbnd}
\delta_j\left((n/ m)A\right)  =  \Theta\left(\delta_j\left( M /n\right)\right) = \Theta( \delta_j(\calm) = \Theta(\delta_j(\calk)) =  \Theta\left(\delta_j\left( K/ n\right)\right).
\end{equation}
The argument here is similar to the one presented in Lemma~\ref{alem:gap}.
}
for all $j < d$, where $\delta_j(\cdot)$ is the eigengap between $j$-th and $j+1$-th eigenvalue of the matrix of interest.  Furthermore, all the above gaps are $\Omega(\delta_d(\calk))$ and $\bar A$ is positive definite. 

We can also show that $\bar A \propto \bar M$. Specifically, let $\calp_A = U_A
U^\transpose_A$ and $\calp_M = U_M U^\transpose_M$ be the corresponding
projection operators. Note that $\bar A = \calp_A A$ and $\bar M = \calp_M M$. We have


\begin{lemma}\label{lem:project}
	Let $\calp_A$ and $\calp_M$ be defined above. Then $\calp_A A$ is
	positive definite. Furthermore, with high probability (over the randomness
	in matrix $A$), 
	\begin{align}
		\left \| \frac{n}{m} \calp_A A - \frac{1}{n} \calp_M M \right \| = O \left( \frac{ \left \| \frac{n}{m} A - \frac{1}{m} M \right \|}{\delta_d(M/n)} \right), \label{eqn:relateprojAM}
	\end{align}
	where $\delta_d(M/n) = \lambda_d(M/n)- \lambda_{d + 1}(M/n)$.
\end{lemma}
\begin{proof}
	First, we will show that $\calp_A A$ is positive definite.
	Lemma~\ref{lem:amclose} gives a bound on $\left \| \frac{n}{m} A -
	\frac{1}{n} M \right\|$. Using~Theorem II of~\cite{Kat:1987}, we know that if
	$\hat{\lambda}_i$ and $\lambda_i$ denote the $i\th$ largest eigenvalues of
	$\frac{n}{m} A$ and $\frac{1}{n} M$, respectively, then 
	\begin{align}
		\max_{i} \left|\hat{\lambda}_i  - \lambda_i \right| &\leq \left\| \frac{n}{m} A - \frac{1}{n} M \right \|. \label{eqn:eigenAMclose}
	\end{align}
		The algorithm chooses $d$ so that $\hat{\lambda}_d - \hat{\lambda}_{d +
		1} = \Omega\left(\left(\frac{n}{m} \right)^{c_2} \right)$ for some $c_2
		\ll 1/4$. This together with the bound on $\left \| \frac n m A - \frac 1
		n M \right \|$ given by Lemma~\ref{lem:amclose} and the fact that $M$ is
		positive definite shows that the $d$ largest eigenvalues of $\frac{n}{m}
		A$ are all positive and hence $\calp_A A$ is positive definite.

	Next, we show that $\left \| \calp_A - \calp_M \right \|$ can be suitably
	bounded by using the Davis-Kahan sin $\Theta$ theorem~\cite{DK:1970}. In
	particular, let $\eta = \left \| \frac{n}{m} A - \frac{1}{n} M  \right \|$,
	$S_1 = \{ \lambda ~|~ \lambda \geq \lambda_d - \eta \}$ and $S_2 = \{
		\lambda ~|~ \lambda < \lambda_{d + 1} + \eta \}$. Let $\delta_d :=
	\lambda_{d} - \lambda_{d + 1}$. We know that $0 < \mathrm{dist}(S_1, S_2)
	\leq \delta_d - 2 \eta$. Let $\calp_A(S_1)$ and $\calp_M(S_1)$ be the
	projection operations defined using eigenvectors with eigenvalues in $S_1$
	of the matrices $\frac{n}{m} A$ and $\frac{1}{n} M$, respectively. By our
	choice of parameters, we see that $\calp_A(S_1) = \calp_A$ and $\calp_M(S_1)
	= \calp_M$. Thus, by using the Davis-Kahan sin $\Theta$ theorem, we get
	\begin{align}
		\left\| \calp_A - \calp_M \right \| &\leq \frac{\left \| \frac n m A - \frac 1 n M \right \|}{\mathrm{dist}(S_1, S_2)} = O \left( \frac{\eta}{\delta_d} \right) \label{eqn:projAMopclose}
	\end{align}

	Finally, we observe that,
	\begin{align*}
		\left \| \calp_A \frac n m A - \calp_M \frac 1 n M \right \| &\leq \left
		\| \calp_A \left(\frac n m A - \frac 1 n M\right) \right\| + \left \| \left( \calp_A - \calp_M \right) \frac 1 n M \right \| \\
		&\leq \left \| \frac n m A - \frac 1 n M \right \| + O \left( \frac{\eta}{\delta_d} \right) .
	\end{align*}
	Above, we used the fact that $\left \| \calp_A \right \| \leq 1$ and that
	$\left \| \frac 1 n M \right \| = O(1)$. 
\end{proof}
\myparab{Step 3: $\bar A^{1/2}\propto \bar K$.}
\textbf{Let $\eta = \|(n/m)A - M/n\|$.}  From (\ref{eqn:gapbnd}), (\ref{eqn:relateprojAM}), and Proposition~\ref{prop:K2Mclose}, we can get $\|(n/m)\bar A - \bar K^2/n^2\| = O(\eta/\delta_d(\calm))$. 
Next, we need to show that the square root of $\bar A$ and those of $\bar K^2$ will be close, \ie

\begin{lemma}\label{lem:sqrtak}
Let $\eta = \left \| \frac{n}{m} A - \frac{1}{n} M  \right \|$.
Let $\bar A^{1/2}$ and $\bar K$ be defined as above. We have
\begin{equation}
\|\sqrt{\frac n m}\bar A^{1/2} - \frac 1 n\bar K\| \leq O(\sqrt{\eta/\delta_d(\calm)} + \sqrt d \eta/\delta^2_d(\calm)) 
\end{equation}
\end{lemma}

Here, $\bar A^{1/2} = U_AS^{1/2}_A U^\transpose_A$. Our techniques for proving the above lemma are similar to those used in Step 1. Specifically, we show that the \emph{pairwise} eigenvectors of $\bar A$ and that of $\bar K$ are close. Thus, after linearly scaling these two set of vectors by using (approximation of) $S^{-1/4}$ will result in two set of vectors that are still close.

\begin{proof}
We need to argue that each eigenvector of $\bar A$ is close to that of $K^2$ (up to a sign difference). 
But this time we need to handle matrices, rather than linear operators, so we can use the original Davis-Kahan theorem. 

We shall also set that $\eta > 10\delta^2_d$ ($\eta$ grows with $m$ and it is in $\mathrm{poly log}n$ scale; 10 is an arbitrarily large constant). Let $(U_A)_{:, i}$ and $(U_K)_{:, i}$ be the $i$-th column of $U_A$ and $U_K$, respectively. By the Davis-Kahan theorem~\cite{DK:1970} and (\ref{eqn:gapbnd}), we have  
\begin{equation}
|\sin \Theta((U_A)_{, i}, (U_K)_{, i})| \leq \frac{\|\frac n m \bar A - \frac 1 {n^2}\bar K^2\|}{\Theta(\delta_d(\calm))} = \Theta(\eta/\delta^2_d(\calm)). 
\end{equation}
Thus, there exists an $w \in \{\pm 1\}$ such that 
\begin{equation}
\|(U_A)_{:, i}w -  (U_K)_{:, i} \| \leq |\sin((U_A)_{:, i}, (U_K)_{:, i})| = O(\eta/\delta^2_d(\calm)). 
\end{equation}

Therefore,
\begin{equation}
\|U_A W - U_K \|_F = O(\sqrt d \eta/\delta^2_d(\calm)), 
\end{equation}
where $W$ is a diagonal matrix so that each diagonal entry is in $\{\pm 1\}$, and recall that $\eta$ is an upper bound of  $\|(n/m)A - (1/n)M\|$. 

Now we can move to bound $\|\sqrt{\frac n m}\bar A^{1/2} - \frac 1 n\bar K\|$. Let $E = U_A W - U_K$, \ie $U_A = (U_K + E)W^\transpose$. 
We have
{\small
\begin{eqnarray*}
& & \left\|\sqrt{\frac n m}\bar A^{1/2} - \frac 1 n \bar K\right\| \\
& = & \left\| (U_K + E)W^\transpose \left(\sqrt{\frac n m}S^{1/2}_A\right)(U_K+E)^\transpose  - U_K\left(\frac{S_K}{n}\right)U^\transpose_K\right\| \\
& \leq & \left\|U_K\left(W^\transpose \sqrt{\frac n m}S^{1/2}_A W - \frac{S_K}{n}\right)U^\transpose_K\right\| + \left\| EW^\transpose \sqrt{\frac n m}S^{1/2}_A W\right\|  \\
& & \quad \quad + \|W\left(\sqrt{\frac n m}S^{1/2}_A\right)E^\transpose \| + \left\|EW^\transpose \sqrt{\frac n m}S^{1/2}_A WE^\transpose \right\| \\
& = &\left\|U_K\left(\sqrt{\frac n m}S^{1/2}_A  - \frac{S_K}{n}\right)U^\transpose_K\right\|  + O(\|E\|) \\
& \leq &\left\|\left(\sqrt{\frac n m}S^{1/2}_A  - \frac{S_K}{n}\right)\right\|  + O(\|E\|). \\
\end{eqnarray*}}

Note first for the $j$-th eigenvalue of $(n/m)A$, namely $\hat \lambda_j$ and the $j$-th eigenvalue of $\frac{1}{n^2}K$, namely $\lambda_j$, we have $\max_{j}|\lambda_j - \hat \lambda_j| \leq \left\|\frac n m\bar A -\frac 1 {n^2}\bar K\right\|$. See Theorem~\ref{thm:normclose}.  Using a similar trick developed in Step 1 in Section~\ref{sec:phiest}, we can bound $\max_j|\lambda^{1/2}_j - \hat \lambda^{1/2}_j| \leq  \left\|\frac n m\bar A -\frac 1 {n^2}\bar K\right\|^{1/2}$. Thus, we have
$$\|\sqrt{\frac n m}\bar A^{1/2} - \frac 1 n\bar K\| \leq O(\sqrt{\eta/\delta_d(\calm)} + \sqrt d \eta/\delta^2_d(\calm)).$$

\end{proof}

Finally, putting together Lemma~\ref{lem:sqrtak} and $U_KS^{1/2}_KW \approx \Phi_d$ (From Eq.\eqref{eqn:kpca}), one can see that $U_AS^{1/4}_A \propto \Phi_d$, \ie 

\begin{proposition}\label{prop:estimate}
Let $\hat \Phi_d = \frac{n^{3/4}}{m^{1/4}}U_AS^{1/4}_A$. Then there exists some orthogonal matrix $\tilde W$ such that whp
\begin{equation}
\left\|\hat \Phi_d - \Phi_d\right\|_F \leq O\left(\frac{\sqrt{dn}(\sqrt{\eta/\delta_d(\calm)} + \sqrt {dn}\eta/\delta^2_d(\calm))}{\delta_d(\calm)}\right), 
\end{equation}
where  $\eta = \|(n/m)A - M/n\| \leq c \log^{1/2}(1/\epsilon)\left(\frac n m\right)^{1/4}$
\end{proposition}

\begin{proof}
We prove the proposition via using a triangle inequality through $K$, \ie 
we need to show that $U_A S^{1/4}_A \propto U_KS^{1/2}_K$ and $U_KS^{1/2}_K \propto \Phi_d$. As discussed before, the latter part is a known result in kernel PCA~\cite{tang2013}, \ie  there exists a rotation matrix so that with probability $\geq 1 - 2\epsilon$:
\begin{equation}\label{lem:kernel}
\|U_KS^{1/2}_K W - \Phi_d\|_F \leq 2 \sqrt 2\frac{\sqrt{\log(1/\epsilon)}}{\delta_d(\calk)}.
\end{equation}
Thus, we need only understand the relationship between $U_AS^{1/4}$ and $U_K S^{1/2}_K$. Recall from Lemma~\ref{lem:sqrtak} that
\begin{equation}
\left\|\sqrt{\frac n m}\bar A^{1/2} - \frac{\bar K}{n}\right\| = O\left(\sqrt{\frac{\eta}{\delta_d(\calm)}} + \frac{\sqrt d\eta}{\delta^2_d(\calm)}\right). 
\end{equation}
Next, we need to show that the ``square root'' of $\bar A$ is close to the ``square root'' of $\bar K$. We leverage Lemma~\ref{lem:xxyy} appeared before (Lemma A.1 from~\cite{tang2013}).

Matrices $\sqrt{n/m}\bar A^{1/2}$ and $\bar K/n$ are the matrices $A$ and $B$ for Lemma~\ref{lem:xxyy}. Note that both of our matrices have constant operator norm so there exists a rotational matrix such that
{\footnotesize
\begin{equation}
\left\| \left(\frac n m\right)^{1/4}U_AS^{1/4}_AW - \frac{1}{\sqrt n}U_KS^{1/2}_K\right\|_F = O\left(\frac{\sqrt d(\sqrt{\eta/\delta_d(\calm)} + \sqrt d \eta/\delta_d(\calm))}{\delta_d(\calm)}\right).
\end{equation}}
By properly scaling up the above inequality (multiplying both sides by a factor of $\sqrt n$) and using (\ref{lem:kernel}), we know that \\ {\small $\left\| \frac{n^{3/4}}{m^{1/4}}U_AS^{1/4}_A W- \frac{1}{\sqrt n}U_KS^{1/2}_K\right\|_F$} is the dominating term, and this completes the proof of the proposition. 

\end{proof}

\myparab{Step 4. Truncation error $\|\Phi_d - \Phi\|_F$.} We shall use the same argument presented in Section~\ref{a:undirected} to bound $\|\hat \Phi_d - \Phi\|_F$, \ie if the decay condition holds, then $\|\Phi_d - \Phi\|_F = \sqrt n \delta^{1/6}_d$ and $d = \delta^{1/4}_d$. Thus, we may set $\delta_d = (n/m)^{2/43}$. Then we have whp $\|\hat \Phi - \Phi\|_F = O(\sqrt n (n/m)^{2/43} \log n)$, which completes our proof for Proposition~\ref{prop:bipartiteest}. Here we also make an additional assumption that the eigengaps $\delta_d = \lambda_d(\calk) - \lambda_{d + 1}$  monotonically decreases whenever the number of non-zero eigenvalues is infinite. Recall that Section~\ref{sec:refined} presents an analysis without the assumption.

\section{More Refined Truncation Error Analysis}\label{sec:refined}

Let $\lambda_1, ..., \lambda_{N_{\calh}}$ be the eigenvalues of $\calk$.  This section analyzes the truncation error without the assumption the gap $\delta_d = \lambda_d(\calk) - \lambda_{d + 1}(\calk)$ is monotonically decreasing. Specifically, the following proposition suffices to prove Theorem~\ref{thm:main} (the constants in the theorem will become worse).

\begin{proposition}\label{prop:gap}Let $\calk$ be a linear operator such that $\lambda_i(\calk)  = O(1/i^c)$ for some constant $c > 2.5$. Let $\delta = \lambda_d - \lambda_{d + 1}$ be given. Then we can express $\sum_{i \geq d + 1} \lambda_i$ and $d$ in terms of $\delta$. Specifically, 
\begin{equation}
\begin{array}{rl}
\sum_{i \geq d + 1}\lambda_i & = \mathrm{poly}(\delta) \\
d & = \mathrm{poly}(1/\delta)
\end{array}
\end{equation}
\end{proposition} 

We need the following lemma. 

\begin{lemma}\label{lem:dlambda} Let $d$ be a sufficiently large number. There exists an $i^*$ such that:
\begin{enumerate}
\item $\sum_{i \geq i^*} \lambda_i\leq c_1 / (2d^{0.1})$. 
\item $\delta_{i^*} \geq c_2 / d^{2.1}$. 
\end{enumerate}
Here, $c_1$ and $c_2$ are constants that are independent of $d$. 
\end{lemma}

The constants 0.1 and 2.1 are chosen arbitrarily. We do not attempt to optimize them. 

\begin{proof}[Proof of Lemma~\ref{lem:dlambda}] Since $\lambda_d = O(d^{-c})$, there exists a constant $c_0$ such that $\lambda_d \leq c_0 / d^c$ for all $d$. We let $\mathrm{tail}(\ell) \triangleq \sum_{i \geq \ell}c_0/i^c = c_1 / \ell^{c-1}$ for some constant $c_1$. Next, we define $i_1$ and $i_2$:
\begin{equation}
i_1 = \max_{i_1} \left\{\sum_{i \leq i_1} \lambda_i \leq 1 - \mathrm{tail}(d^{\frac{0.1}{c-1}})\right\}.
\end{equation}
\begin{equation}
i_2 = \max_{i_2} \left\{\sum_{i \leq i_2} \lambda_i \leq 1 - 0.5\times \mathrm{tail}(d^{\frac{0.1}{c-1}})\right\}.
\end{equation}
We know that $i_2 \leq i_1 \leq d$. Furthermore, 
$$\sum_{i \leq i_2 + 1}\lambda_i \geq 1 - \frac{c_1}2 d^{-\frac{0.1}{c - 1}\cdot (c - 1)} \geq 1 - \frac{c_1} 2 d^{-0.1}.$$
$$\sum_{i_2 < i \leq i_1 }\lambda_i \geq 0.5 \cdot \mathrm{tail}(d^{\frac{0.1}{c-1}}) = 0.5c_1/d^{0.1}.$$
By using an averaging argument, there exists an $i_3 \in [i_2 + 1, i_1 ]$ such that $\lambda_{i_3} \geq c_1/(2d^{1.1})$. On the other hand, we have $\lambda_d \leq c_0 /d^{c-1}$. We have that there exists an $i^* \in [i_3, d]$ such that $\delta_{i^*} = \Theta(1/d^{2.1})$. On the other hand, 
$$\sum_{i \geq i^*}\lambda_i \leq \sum_{i \geq i_2}\lambda_i \leq c_1 /(2d^{0.1}).$$ 

This completes the proof. 
\end{proof}

\begin{proof}[Proof of Proposition~\ref{prop:gap}] Let $\bar d$ be that $\delta = c_2/(\bar d)^{2.1}$, \ie $\bar d = (c_2/\delta)^{1/2.1}$. Using Lemma~\ref{lem:dlambda}, we can find an $i^*$ such that 
\begin{enumerate}
\item $\delta_{i^*} \geq c_2/(\bar d)^{2.1} = \delta$ 
\item $\sum_{i \geq i^*} \lambda_i \leq c_1 /(2 \bar d^{0.1}) = \Theta(\delta^{1/21})$.
\end{enumerate}
Because $\delta_{i^*} > \delta$, $d \geq i^*$. Thus, $\sum_{i \geq d}\lambda_i \leq \sum_{i \geq i^*}\lambda_i = O(\delta^{1/21})$. 
Next, we need to show that $d$ is $\mathrm{poly}(1/\delta)$. Note that $d \geq i^*$ and $\lambda_d \geq \delta$, we have 
$$\Theta(\delta^{1/21}) \geq \sum_{i^* \leq i \leq d}\lambda_i \geq d \delta.$$
Thus, $d \leq \delta^{-20/21} = \mathrm{poly}(1/\delta).$
\end{proof} 

\section{Analysis for the isomap-based algorithm}\label{asec:isomap}
This section analyzes the isomap-based algorithm.

\myparab{Recall of the notations.}
  $x_1, x_2, ..., x_n$ are the latent variables. 
Also, $z_i = \Phi(x_i)$ and $\hat z_i = \hat \Phi(x_i)$. 
For any $z \in R^{N_{\calh}}$ and $r > 0$, we let $\Ball(z, r) = \{z': \|z'- z\| \leq r\}$. 
Define projection $\proj(z) = 
\arg \min_{z' \in \calc}\|z' - z\|$. Finally, for any point $z \in \calc$, define $\Phi^{-1}(z)$ be that $\Phi(\Phi^{-1}(z)) = z$ (\ie $z$'s original latent position). For points that are outside $\calc$, define $\Phi^{-1}(z) = \Phi^{-1}(\proj(z))$. 

\myparab{Outline.} We first describe the fundamental building block for our analysis. Then we analyze the performance of the denoising procedure~\ref{eqn:denoise}. Finally, we give an analysis for the full isomap algorithm.


\subsection{Fundamental building blocks}
\begin{lemma}\label{lem:latent} Let $x, x' \in [0, 1]$. Let also $g(n)$ and $h(n)$ be two diminishing functions (\ie $g(n), h(n) = o(1)$). We have 

\noindent{(1)} If $|x - x'| = h(n)$, then 
$\| \Phi(x) - \Phi(x') \| = \sqrt{\frac 2 c}h^{\Delta/2}(n) + o(h^{\Delta/2}(n)). $

\noindent{(2)} If $\|\Phi(x) - \Phi(x') \| = g(n)$, then $|x - x'| = \left(\frac{c}{2}\right)^{1/\Delta}g^{2/\Delta}(n)$.%
	\vkmnote{Emphasise small world kernel}
\end{lemma}

\begin{proof}[Proof of Lemma~\ref{lem:latent}] For part (1), we have 
\begin{eqnarray*}
\| \Phi(x) - \Phi(x') \|^2 & = & \kappa(x, x) + \kappa(x', x') -2 \kappa(x, x') \\
& = & 2 - 2 \frac{c}{c(1+ |x-x'|^{\Delta}/c)} \\
	& = & 2 - 2(1 - |x-x'|^{\Delta}/c)  + o(|x-x'|^{\Delta}) \\
& = & \frac 2 c|x - x'|^{\Delta} + o(|x - x'|^{\Delta}).
\end{eqnarray*}
Thus, $\| \Phi(x) - \Phi(x')\| = \sqrt{\frac 2 c} h^{\Delta/2}(n) + o(h^{\Delta/2}(n))$. 

We may then use part 1 to prove part 2 in a straightforward manner. 
\end{proof}

\begin{lemma}\label{eqn:isoinverse}Let $x \in [0, 1]$ and $z = \Phi(x)$. Let $z' \in R^{N_{\calh}}$ be a point such that $\|z' - z\| \leq g(n)$, then 
$\| \proj(z') - z\|  \leq 2 g(n)$ and
$| x - \Phi^{-1}(z')|  \leq \left(\frac c 2\right)^{1/\Delta}(2g(n))^{2/\Delta} = (2c)^{1/\Delta}g^{2/\Delta}(n). 
$
\end{lemma}

\begin{proof}[Proof of Lemma~\ref{eqn:isoinverse}] Since $\| z' - z\| \leq g(n)$, we have $\|\proj(z') - z'\| \leq g(n)$. Then using a triangle inequality, we have $\|\proj(z') - z\| \leq \|\proj(z') - z'\| + \|z' - z\| \leq 2g(n)$. Then by Lemma~\ref{lem:latent}, we may also prove the second part of the lemma. 
\end{proof}

\subsection{Analysis of the denoising procedure}

Also, recall that we classify a point $i$ into three groups:  \textbf{1. Good:} when $\| \hat z_i - \proj(\hat z_i) \| \leq 1/ \sqrt{f(n)}$. We may further partition the set of good points into two parts. \textbf{Good-I:} those points so that $\| \hat z_i - z_i \| \leq 1 /\sqrt{f(n)}$. \textbf{Good-II:} those points that are good but not in Good-I. 
\textbf{2. Bad:} when $\|z_i - \proj(z_i) \| > 4/\sqrt{f(n)}$. 
\textbf{3. Unclear:} otherwise. We have (see Appendix~\ref{asec:isomap} for a proof)

We use the following decision rule to 

{\small
\begin{equation}
\proc{Denoise}(\hat z_i): \mbox{If } |\Ball(\hat z_i, 3/\sqrt{f(n)})| < n/f(n), \mbox{  remove $\hat z_i$}.
\end{equation}
}

We want to prove the following lemma. 
\begin{lemma}\label{alem:denoise}[Repeat of Lemma~\ref{lem:denoise}] After running $\proc{Denoise}$, Using the counting-based decision rule, all the good points are kept, all the bad points are eliminated, 
and the unclear points have no performance guarantee. The total number of eliminated nodes is $\leq n/f(n)$. 
\end{lemma}

\begin{proof}[Proof of Lemma~\ref{lem:denoise}]
We have the follow three facts. 
\begin{fact} Let $\hat z_i$ be a good point. For any good point $j$ such that $\|\proj(\hat z_i) - \proj(\hat z_j)\| \leq 1/\sqrt{f(n)}$, we have $\|\hat z_i - \hat z_j\| \leq 3 /\sqrt{f(n)}$. 
\end{fact}
This can be shown via a simple triangle inequality: 
{\small
\begin{equation}
\| \hat z_j - \hat z_i \| \leq \| \hat z_j - \proj(\hat z_j) \| +\| \proj(\hat z_j) - \proj(\hat z_i)\| + \| \proj(\hat z_i) - \hat z_i\| \leq \frac{3}{\sqrt{f(n)}}. 
\end{equation}
}

\begin{fact}Let $\hat z_i$ be a good point and consider $\Ball(\hat z_i, 3/\sqrt{f(n)})$. The total number of points that are within the ball is at least $n(c_0/3f(n))^{1/\Delta}$ for some constant $c_0$.
\end{fact}
\begin{proof} We need to show that (1) there are a sufficient number of Good-I nodes that are near $\proj(\hat z_i)$, and (2) these points are in $\Ball(\hat z_i, 3/\sqrt{f(n)})$. Note that when we have a Good-I $z_j$ such that $\| z_j - \proj(\hat z_i)\| \leq 1/\sqrt{f(n)}$, we have $|x_j - \Phi^{-1}(\hat z_i)| \leq (c/2)^{1/\Delta}(1/\sqrt{f(n)})^{2/\Delta} = (c/(2f(n)))^{1/\Delta}$. By the near-uniform density assumption, we have the total number of nodes $x_j$ that within the distance of $(c/(2f(n)))^{1/\Delta}$ is at least $n(c_0/(2f(n)))^{1/\Delta}$ for some constant $c_0$. Note that the number of non-Good-I nodes is at most $n/f(n)$. Thus, the total number of Good-I nodes here is 
$$n\left(\left(\frac{c_0}{2f(n)}\right)^{1/\Delta} - \frac 1 {f(n)}\right) \leq n\left(c_0/(3f(n))\right)^{1/\Delta}.$$
Finally, we have 
$$\|\hat z_j - \hat z_i\| \leq \| \hat z_j - z_j\| + \| z_j - \proj(\hat z_i)\| + \|\proj(\hat z_i) - \hat z_i\| \leq 3 /\sqrt{f(n)}.$$
Therefore, all these nodes are in $\Ball(\hat z_i, 3/\sqrt{f(n)})$. 
\end{proof}

\begin{fact} For all the bad points $\hat z_i$, the ball $\Ball(\hat z_i, 3/\sqrt{f(n)})$ does not cover more than $n/f(n)$ nodes. 
\end{fact}

\begin{proof} For any node $\hat z_j$ in $\Ball(\hat z_i, 3/\sqrt{f(n)})$, we have $\|\hat z_j - \proj(\hat z_j)\|> 1/\sqrt{f(n)}$. Otherwise, 
$$\| \hat z_i - \proj(\hat z_i) \| \leq \| \hat z_i - \proj(\hat z_j) \| \leq \| \hat z_i - \hat z_j \| + \| \hat z_j - \proj(\hat z_j)\| \leq 4/\sqrt{f(n)}.$$
This contradicts to that $\hat z_i$ is bad. But $\|\hat z_j - \proj(\hat z_j)\|> 1/\sqrt{f(n)}$ implies $\|\hat z_j - z_j\|> 1/\sqrt{f(n)}$ so the number of such $j$ is upper bounded by $n/ f(n)$. 
\end{proof}

\end{proof}

\subsection{The performance of isomap algorithm}\label{sec:isomapperformance}
This section proves the following proposition .

\begin{proposition}\label{prop:tight} The length of the path connecting between $i$ and $j$ has the following bounds:
\begin{equation}
(d - 1) \left(\frac c 2\right)^{1/\Delta}\left(\frac{\ell-3}{\sqrt{f(n)}}\right)^{2/\Delta} \leq |x_i - x_j| \leq d\left(\frac c 2\right)^{1/\Delta}\left(\frac{\ell + 8}{\sqrt{f(n)}}\right)^{2/\Delta}
\end{equation}
\end{proposition}

Before proceeding, we remark that our final distance estimate should be $(d-1)(c/2)^{1/\Delta}((\ell-3)/\sqrt{f(n)})^{2/\Delta}$ when the distance between two nodes is $d$ on the graph. The multiplicative error ratio is $\left((\ell + 8)/( \ell + 3)\right)^{2/\Delta}$ and the additive error is $\left(\frac c 2\right)^{1/\Delta}\left((\ell + 8)(\sqrt{f(n)})\right)^{2/\Delta}$. This implies Theorem~\ref{thm:main}. 

\begin{proof} The analysis consists of two parts. First we give a lower bound on $d$, \ie $d$ is not too small. Then we give an upper bound. 

\myparab{1. $d$ is not too small.} 
We want to bound the latent distance. For an arbitrary consecutive pair of nodes $i_j$ and $i_{j + 1}$, they are not bad nodes so we have
\begin{eqnarray*}
\| \proj(\hat z_{i_j}) - \proj(\hat z_{i_{j + 1}})\| & \leq & \| \proj(\hat z_{i_j})  - \hat z_{i_j} \| +   \| \proj(\hat z_{i_{j+1}})  - \hat z_{i_{j+1}} \| + \|\hat z_{i_j} - \hat z_{i_{j + 1}}\|\\
& \leq & \frac{\ell + 2\times 4}{\sqrt{f(n)}}.
\end{eqnarray*}

Then by Lemma~\ref{eqn:isoinverse}, we have 
$$\| \Phi^{-1}(\hat z_{i_j}) -  \Phi^{-1}(\hat z_{i_{j + 1}}) \| \leq \left(\frac c 2\right)^{1/\Delta}\left(\frac{\ell + 2 \times 4}{\sqrt{f(n)}}\right)^{2/\Delta}.$$

Thus, we have
\begin{lemma}\label{lem:lower} Let $d$ be the shortest path distance between $i$ and $j$ on the graph built form isomap. The latent distance between $x_i$ and $x_j$ is at most $d\left(\frac c 2\right)^{1/\Delta}\left(\frac{\ell + 8}{\sqrt{f(n)}}\right)^{2/\Delta}$.
\end{lemma}

\myparab{2. $d$ is not too large.} We next give a constructive proof for an upper bound of $d$. 

\begin{lemma}\label{lem:upper}
Using the notations above, we can find a path $i, i_1, ..., i_{d - 1}, j$ such that: 
\begin{enumerate}
\item All these nodes are Good-I nodes. 
\item The corresponding latent variables are monotonically increasing or decreasing, and the distance (in feature space) between two consecutive nodes is at least $(\frac c 2)^{1/\Delta}\left(\frac{\ell-3}{\sqrt{f(n)}}\right)^{2/\Delta}$. 
\item The distance between $\hat z_{i_j}$ and $\hat z_{i_{j + 1}}$ (for any $j$) is within $\ell/\sqrt{f(n)}$. 
\end{enumerate}
\end{lemma}

If all the above claims were true, then we know that the path ``makes good progress'' in every step, \ie when we move from $i_{j}$ to $i_{j - 1}$, in the latent space, we are $(\frac c 2)^{1/\Delta}\left(\frac{\ell-3}{\sqrt{f(n)}}\right)^{2/\Delta}$ closer to the destination $j$. This implies: 

\begin{corollary}\label{cor:upper} The length of the path connecting between $i$ and $j$ has the following upper bound:
\begin{equation}
(d-1) (c/2)^{1/\Delta}((\ell-3)/\sqrt{f(n)})^{2/\Delta} \leq |x_i - x_j| 
\end{equation}
\end{corollary}

Corollary~\ref{cor:upper} and Lemma~\ref{lem:lower} implies Proposition~\ref{prop:tight}. 
\end{proof}

We now proceed to prove Lemma~\ref{lem:lower}. 

\begin{proof}[Proof of Lemma~\ref{lem:lower}]
Let us start with considering an arbitrary interval $I \in [0, 1]$ in the latent space of size $\left(\frac c 2\right)^{1/\Delta}\left(\frac{\ell-3}{\sqrt{f(n)}}\right)^{2/\Delta} - \left(\frac c 2\right)^{1/\Delta}\left(\frac{\ell - 2}{\sqrt{f(n)}}\right)^{2/\Delta} = \left(\frac{c}{2f(n)}\right)^{1/\Delta}((\ell-3)^{2/\Delta} - (\ell-2)^{2/\Delta})$.

The expected number of nodes in this interval is $(c_0/2)^{1/\Delta}n/f^{1/\Delta}(n)$ for some constant $c_0$, and we have a concentration bound, \ie with exponentially small probability that the number of nodes is $\geq \frac 1 2 \left(\frac c+0 2\right)^{1/\Delta}n/f^{1/\Delta}(n)$. On the other hand, out of these nodes only $n/f(n)$ are not Good-I nodes, so there are $\Theta((n/2)^{1/\Delta}n/f^{1/\Delta}(n))$ Good-I nodes in $I$. 

Then we may construct the sequence $i_1, ..., i_{d - 1}$ using this property. Wlog, let $x_i < x_j$. We let 
$$I_1 = \left[x_i + \left(\frac c 2\right)^{1/\Delta}\left(\frac{\ell - 3}{\sqrt{f(n)}}\right)^{2/\Delta}, x_i + \left(\frac c 2\right)^{1/\Delta}\left(\frac{\ell - 2}{\sqrt{f(n)}}\right)^{2/\Delta}\right].$$

Let $i_1$ be an arbitrary Good-I node in $I_1$. We can also recursively define 
$$I_j  = \left[x_{i_j} + \left(\frac c 2\right)^{1/\Delta}\left(\frac{\ell - 3}{\sqrt{f(n)}}\right)^{2/\Delta}, x_{i_j} + \left(\frac c 2\right)^{1/\Delta}\left(\frac{\ell - 2}{\sqrt{f(n)}}\right)^{2/\Delta}\right].$$
Then we can find a Good-I $i_j$ in $I_j$. 

By construction, property 1 and 2 hold. Now we need only verify property (3). Since all the nodes are Good-I, $\| \hat z_{i_j} - z_{i_j}\| \leq \frac{1}{\sqrt{f(n)}}$. Using Lemma~\ref{lem:latent}, we also know that 
\begin{equation}
\|z_{i_j} - z_{i_{j + 1}}\| \leq \sqrt{\frac 2 c}\left(\left(\frac c 2\right)^{1/\Delta}\left(\frac{\ell - 2}{\sqrt{f(n)}}\right)^{2/\Delta}\right)^{\Delta/2} \leq \frac{\ell-2}{\sqrt{f(n)}}.
\end{equation}
By using a triangle inequality, $\| \hat z_{i_j} - \hat z_{i_{j+1}}\| \leq \ell/\sqrt{f(n)}$. 
 \end{proof}

\section{Spectral properties of linear operators and graphs}\label{asec:existingresult}
This section presents prior spectral results on linear operators or graphs that are used in our analysis. 

\begin{theorem}[Simplified from~\cite{Kat:1987}]\label{thm:normclose} Let $A$ and $B$ be self-adjoint operators in $H$ such that $B = A + C$, where $C$ is a compact self-adjoint operator. Let $\{\gamma_k\}$ be an enumeration of the non-zero eigenvalues of $C$. Then there exist extended enumerations $\{\alpha_j\}, \{\beta_j\}$ of discrete eigenvalues for $A$, $B$, respectively, such that the following inequality holds: 
\begin{equation}
\left(\sum_{j \geq 1}|\alpha_i - \beta_i|^p\right)^{1/p} \leq \left(\sum_{j}|\gamma_k|^p\right)^{1/p}, \quad 1 \leq p \leq \infty.
\end{equation}
\end{theorem}

\subsection{Distribution of the eigenvalues}
\begin{theorem}[Weyl~\cite{Weyl1912}, or see~\cite{Wathen2015}]\label{thm:weyl} Let $\kappa$, $F$, and $\calk$ be defined as above.  If $F$ is uniform, $\kappa(x, y) = \kappa(y, x)$, and $\partial^v(x, y)/\partial^vx$ exists and is continuous, then $\lambda_n(\calk) = o(n^{-v - 1/2})$. 
\end{theorem}

 When $\cald$ is not uniform, we also have the following Proposition~\cite{konig1986eigenvalue}. 

\begin{proposition} Let $2 \leq p < \infty$. $\calx$ be a Banach space and $\calt \in L(\calx)$ be $(p, 2)$-summing (\ie $\left(\sum_{i \geq 1}\lambda^p_i(\calt)\right)^{1/2}$ exists), then $\lambda_n(\calt) = O(n^{1/p}).$
\end{proposition}

One can see that if $\cala$ and $\calb$ are, for example, $(4, 2)$-summing, then $(\cala+\calb)$ is also $(4, 2)$-summing. Together with the fact that when $\cald$ is uniform, $\calk$ is $(4, 2)$-summing, we see that when $\cald$ is a mixture of uniform distribution (\ie $\cald$ is piecewise constant), our decay assumption holds.

\section{Experiments}\label{sec:exp}

In this setion, we describe the experiments used to validate the necessity of using specral and isomap techniques, and the efficacy of the new regularization technique on the product $B^{\transpose}B$. We compare our algorithms against baselines while noting that a comprehensive evaluation is beyond the scope of this paper. 

\subsection{Synthetic data}\label{sec:synthetic}
We use synthetic data to carry out a ``sanity check'', \ie an SBM inference algorithm does not perform well in a SWM graph and vice versa. 

We evaluate our algorithm against two algorithms optimized for SBM and SWM, respectively: (simplified) Abraham et al.'s algorithm~\cite{AbrahamCKS13} for small-world and Newman's spectral-based modularity algorithm~\cite{newmanfinding2006}. 
Abraham et al.'s algorithm is the only known algorithm with provable guarantee for SWM. Modularity algorithm is a most widely used community detection algorithm. Both baselines perform well when the input comes from the right model. Here we present the result when the model is mis-specified. See Figure~\ref{fig:failedalgo}. In Figure~\ref{fig:failedalgo}b and ~\ref{fig:failedalgo}d, we plot the scatter plot between the true pairwise distances and estimated distances by our algorithm and a baseline. For the block model, note that the true distances can be only either 0 or 1 (after proper rescaling). We observe that our algorithm can automatically detect SBM and SWM; meanwhile, baselines have reduced performance on models for which they are not optimized.

\begin{figure}
\centering
\includegraphics[width=\linewidth]{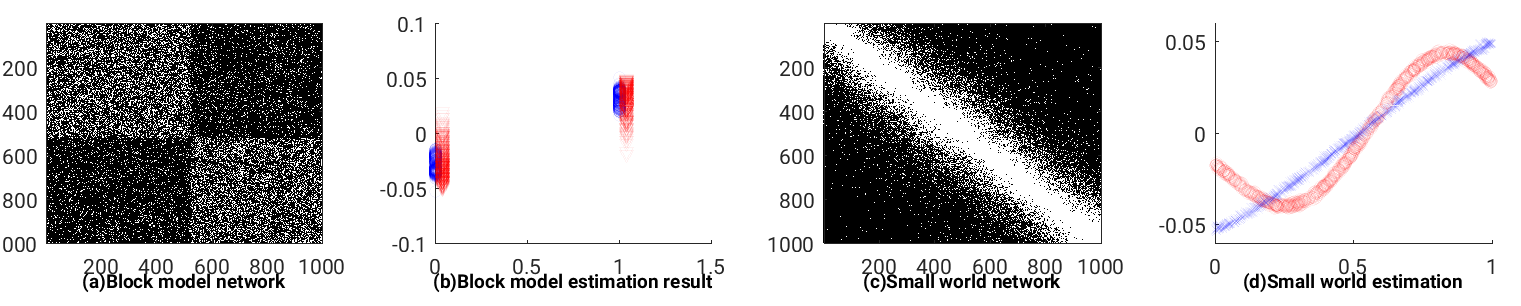}
\caption{\footnotesize{(a) and (b): recover the latent structure of a graph from SBM. (a) is the heat map of the observed graph. (b) is a scatter plot between estimated and true distance. An algorithm that exactly recovers latent variables will output two dots, each of which corresponds to a community; a low quality algorithm will output two intervals that significantly overlap with each other. The blue curve corresponds to our algorithm. The red curve corresponds to a simplified small-world algorithm from~\cite{AbrahamCKS13}. (c) and (d) recover the latent structure of a small-world graph. (c) is the heat map of the observed data. (d) is also a scatter plot between estimated and true distances. High quality algorithms output straight lines. Our algorithm corresponds to the blue curve while Newman's spectral algorithm corresponds to the red curve~\cite{newmanfinding2006}.}}\label{fig:failedalgo}
\end{figure}

\subsection{Real data}
This section evaluates our algorithm on a Twitter dataset related to the US presidential election in 2016. We evaluate the algorithm (1) against binary classification problems for a fair comparison against prior related algorithm (Section~\ref{sec:classification}), (2) against ground-truth of Senate and House members' political leanings (Section~\ref{sec:groundtruth}), and (3) against state level political leanings (Section~\ref{sec:state}). 

\myparab{Data collection.} From October 1 to November 30, 2016, we used the Twitter streaming API to track tweets that contain the keywords ``trump,'' `'clinton,'' ``kaine,'' ``pence,'' and ``election2016'' as text, hashtags (\#trump) or mentions (\@trump). Keyword matching is case-insensitive, and \#election2016 is Twitter's recommended hashtag for the US elections of 2016. We collected a total of 176 million tweets posted by 12 million distinct users. 

We build a directed graph of the users so that node $u$ connects to node $v$ (the edge $(u, v)$ exists) if and only if $u$ retweets/replies to $v$. Our graph is unweighted because we observe insignificant performance differences between weighted and unweighted graphs. Then we choose 3000 nodes with the largest in-degrees as our \emph{influencers} and construct $B$. the mean out-degree of the  followers is 4.6. The influencers also appear at the left-hand side of $B$ because they can also follow other influencers.

\myparab{Sparsity of the data.} Our dataset is very sparse. For example, only 11 out of the 447 legislators with ground-truth scores are in the influencer set (a considerable portion of the influencers are often not highly visible in traditional media). The median degree (in + out) of these legislators is 19. It appears that very few legislators actively use Twitter to discuss election-related topics. The sparsity issue causes the performance degradation of many baselines, many of which have made dense graph assumptions.

\myparab{Regularization and choice of $\theta$.} In our algorithm ($\proc{Bipartite-Est}(B)$ in Fig.~\ref{fig:fullalgo}), 
a regularization parameter $\theta$ needs to be decided. While our result suggests that any $\theta < 0.75$ works (Proposition~\ref{prop:bipartiteest}), the degrees in the real-world graph are more skewed than the graphs specified by our model so $\theta$ impacts the performance of our algorithm. $\theta$ is chosen by using a small portion of classification data as in-sample, \ie find $\theta$ so that the classification error is minimized for 
the in-sample data. 
The skewed degrees also require us to use the standard regularization techniques (introduced and studied
by~\cite{Rohe11,QinR2013}) to push the singular values to be better positioned
at the cost of reducing the gaps between two consecutive singular values. \ie
let $D$ be a diagonal matrix such that $D_{i,i}$ is the row sum of $A$, and
compute $A \leftarrow D^{-1/2}AD^{-1/2}$.

\begin{figure}
\includegraphics[width=\textwidth]{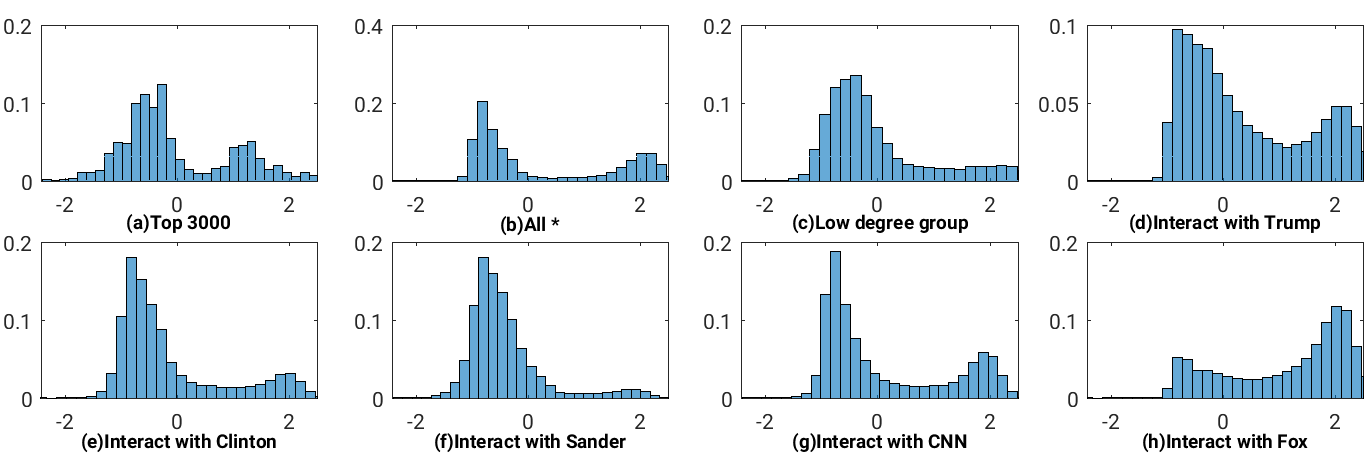}
\caption{Latent distribution of different groups of users. Negative latent variables correspond to liberal users. Group a. Distribution of all influencers (top 3000). Group b. Distribution of users with 30 or more edges. Group c. Distribution of users with only one edge. Group d. Distribution of users who ever interacted with Trump. Group e. Distribution of users who ever interacted with Clinton. Group f. Distribution of users who ever interacted with Sanders. Group g. Distribution of users who ever interacted with CNN Politics. Group h. Distribution of users who ever interacted with Fox News.}\label{fig:latentvisualize}
\end{figure}

\myparab{Estimation of followers.} When influencers' latent variables $x_i$'s are known/estimated, one may run a simple grid-search to find the maximum-likelihood estimator (MLE) for each follower, \ie examine $\{0, \epsilon, 2\epsilon, ..., 1\}$ and find the value that maximizes  
the likelihood:
$\Pr[B_i \mid \hat y_i] = \prod_{j \leq n}\left(\frac{\kappa(\hat y_i, x_j)}{n}\right)^{I(B_{i, j} = 1)}\left(1-\frac{\kappa(\hat y_i, x_j)}{n}\right)^{I(B_{i, j} = 0)}.$ The exponent  $\Delta$ in $\kappa(\cdot, \cdot)$ can also be estimated from the data.

The (approximate) MLE possesses the following property: 1. \emph{Constant additive error:} for any constants $\delta$ and $\epsilon$, there exists a constant $C$ so that if a follower's expected degree is larger than $C$, then with probability $\geq 1 - \delta$, the additive error is $\epsilon$. 2. \emph{Optimality:} standard lower bound arguments using statistical diference\footnote{The statistical difference between $B_{i, :}$ and $B_{j, :}$ is a constant for any $y_i$ and $y_j$ so the probability of having an estimation error is also a constant.} gives us that with at least constant probability the estimation error is $\Omega(\epsilon)$ so the (approximate) MLE is asymptotically optimal. 

\myparab{Speeding up the estimation for the followers.} The grid-search above finds approximate MLE in polynomial time but in practice they are too slow so we use a simpler heuristics to estimate $x_i$'s: we take the mean of $x_i$'s neighbors as the estimate of $x_i$. Our heuristics preserves the following property: if $x_i < x_j$, then with probability $1 - \delta$, $\hat x_i < \hat x_j$ so long as the expected degrees of $x_i$ and $x_j$ are larger than a suitable constants. 

As most of our experiments assess the quality of order statistics of the users, such heuristics has adequate performance.

\myparab{Baseline algorithms.} Our focus is graph-based algorithms. Algorithms that use the content of tweets to forecast users' political leanings are beyond the scope of this project. 

\mypara{Newman's spectral algorithm for modularity}. A spectral-based algorithm that maximizes modularity~\cite{newmanfinding2006} of the interaction graph. In this algorithm, PCA is applied to the properly normalized interaction graph. We use the first singular vector to decide the membership.

\mypara{Correspondence analysis.} Another spectral algorithm that was
frequently used by political scientists, see \eg~\cite{BJN15} and discussions
therein. The algorithm first normalized the graph with its column and row sums.
Then it uses the first left singular vector, which is the result of SVD, to
estimate the influencers' latent positions. Standard Kernel PCA techniques are then used 
to ``generalize'' the model and predict the followers. 


\mypara{Label propagation:} Label propagation (LP) algorithms are local algorithms so that each user updates his/her latent variable based on estimates of his/her neighbors' latent variables. We consider two versions of LP, namely majority~\cite{raghavan} and random walks~\cite{Silva03globalversus}. For the majority algorithm, we use output of modularity algorithm as the initial weights. For the random walks one, we set two presidential candidates to be 0 and 1, respectively. 

\mypara{Multidimensional scaling.} We use a standard MDS algorithm~\cite{BorgGroenen2005} on the raw social network and low-dimension approximation of the social network.  

\subsubsection{Qualitative summaries}
We first present several key qualitative findings, which serve as a sanity check to ensure that the outcomes of our model are consistent with common sense. 

\myparab{Distribution of the users.} Figure~\ref{fig:latentvisualize} shows the latent variables for different groups of users. The output is \emph{standardized} (so that the standard deviation is 1). 
 1\% of the outliers are removed in the visualization. We consider the following groups: (a) influencers, (b) all users with at least 30 out-going edges, (c) users with only one edge, (d - f) users that referred to Trump, Clinton, and Sanders, respectively, and (h \& g) users that referred to CNN Politics and Fox News, respectively. 
 
 We observe a bimodal distribution in both groups and more left-leaning
 populations. Most of the latent estimates in group (c) are negative (liberal).
 This suggests that the first account a Twitter user refers to usually is a
 left-leaning media. Many left-leaning users refer to Trump (mainly to bash
 him), which is consistent with his media coverage. Users referring to Sanders
 skew to the left. CNN Politics  attracts more left-followers while Fox News
 attracts more right-followers.

\myparab{Distribution of the edges:} We can also use heatmap to visualize the interactions between users. Figure~\ref{fig:inferred_kernel} compares the inferred kernel against SWM and SBM. Specifically, Figure~\ref{fig:inferred_kernel}b represents a small-world interaction and Figure~\ref{fig:inferred_kernel}c represents the stochastic block model. 

Figure~\ref{fig:inferred_kernel}a is a visualization of our model. We construct the image as follows. We first sort the influencers according to their latent scores and partition them into 30 groups of equal size. Each group corresponds to one row in the image (\eg first row is the leftmost group). For each group, we compute the ``average'' histogram of users that refers to a member in the group.
Here, we use 20-bins. Finally, we color code the histogram, \eg a white pixel corresponds to a large bin. 

Thus, Figure~\ref{fig:inferred_kernel}a approximates the inferred kernel function for the influencers. We observe that the diagonals are brighter (resembling small world) but at the same time there is some ``blurred'' block structure. This observation confirms that the real dataset ``sits'' between the small-world and block models, and highlights the need to design a unified algorithm that disentangles these two models. 

\myparab{Distribution of the presidential candidates}. As mentioned, we perform a sanity check on our estimates of all candidates' latent scores. These (unnormalized) scores are: Sanders (-0.272) $<$ Clinton (-0.014) $<$ Kasich (0.01) $<$ Cruz (0.013) $<$ Trump(0.037). 

 \subsubsection{Classification result}\label{sec:classification}
 \label{sec:classificiation}
 We sample a subset of users and label them as conservative or liberal according to four categories. Category 1. \emph{Top 2000}, \ie one of the 2000 users with the largest in-degree. Category 2. Top 2000 to 3000. \ie one of the 1000 users with an in-degree rank between 2000 and 3000. Category 3: 30+. \ie users with at least 30 out-going edges. Category 4. Everyone, \ie all users. 

\myparab{Labeling.} For each selected user, we ask two human judges, who are either the authors of this paper or workers from Amazon Mechanical Turk, to label the user as likely to vote Clinton or Trump in the final election. If it is not clear, \eg the user criticizes both candidates, the judges can label the user ``unclear.'' The information we provide to the judge includes user's self-reported name, screen name, Twitter-verified account status (usually indicates a celebrity), self-description, URL, and a random sample of at most 20 tweets. When the labeling is complete, we discard all users except for those unanimously labeled as likely to vote for one of the candidates, for a total of $45\%$ of the original data. 

We need to decide a threshold to turn the output of our algorithms (and some of the baseline algorithms) into binary forecasts.  We use a subset of classification tasks as in-sample data. We note that the threshold is robust for both in-sample and out-of-sample data (the performance difference is inconsequential). 

\setlength{\arrayrulewidth}{0.1pt}

\begin{table}
\caption{Performance of the algorithms on classifying user's political leaning. Column 2: all 752 labeled users, Column 3: 270 influencers in the top 3000 list, Column 4: 571 users with 30 or more edges (exclude those influencers), Column 5: the rest users. MDS 5 = MDS with 5 eigenvalues;
MDS = MDS with all eigenvalues.}
\label{table:classification}
\centering
\small
\begin{tabular}{c|cccc} 
\hline 
Algo.&  All & Top 3000  & 30+  & Others  \\
& (752) & (270) & (571) & (123) \\ \hline
Ours & \textbf{89.9\%} & 89.6\% & \textbf{91.9}\% & \textbf{87.8\%} \\  
Modularity & 87.7\% & 89.6\% & 90.9\% & 82.1\% \\ 
{\scriptsize Correspond Analysis} & 57.3\% & 45.2\% & 58.1\% & 57.7\% \\ 
Majority & 88.4\% & \textbf{90.7\%} & 90.7\% & 79.7\% \\ 
{\scriptsize Random Walk} & 65.2\% & 57.4\% & 64.6\% & 65.9\% \\  
MDS 5& 64.8\% & 59.6\% & 65.7\% & 62.6\% \\ 
MDS& 57.3 \% & 66.3\% & 57.1\% & 55.3\% \\ \hline
\end{tabular}
\end{table}

\myparab{Results}. See Table \ref{table:classification}. The classes are balanced. We report classification accuracy on (1) all 752 labeled users, (2) 270 influencers in the top 3000 list, (3) 571 large-degree users with 30 or more edges (but not influencers), and (4) the rest users. 

Our algorithm has the best overall performance and is the best or near-best for each subgroup. 
Only the performance of the modularity algorithm, which is optimized for classification applications is close to ours. But as we shall see in Section~\ref{sec:groundtruth} and~\ref{sec:state}, the algorithm performs poorly for tasks that requires understanding users' latent structure in finer granularity. For the predictions of average users (the last column), the accuracies of most baselines are below 80\%, which is consistent with prior experiments~\cite{CohenR13}. 

\subsubsection{Correlation with ground-truth}\label{sec:groundtruth}
We compare the latent estimates of politicians (members of the 114th Congress) and the ground-truth. The ground-truth of these politicians is estimated by various third parties using data sources such as voting record and co-sponsorship~\cite{congressIdeology}. 

\myparab{Standard error} Beyond correlation, we also need to estimate the statistical significance of our estimates. We use bootstrapping  to compute the standard error of our estimator, and then use the standard error to estimate the p-value of our estimator.
Specifically, our goal is to understand the explanatory power of our latent estimates $\hat{x}_l(B)$ (how we write it to highlight the statistics we compute depends on the bipartite graph $B$) to response $y_i$ representing the ground-truth of the politicians. Thus, we set up a linear regression:
$y\sim\beta_1\hat{x}+\beta_0$
In the bootstrapping procedure, we repeat the following process for $k$ times:
Sample 80\% of the edges from B and compute the latent estimates as well as $\beta_1$ (by running an OLS linear regression). We mark the estimate at the i-th repetition $\beta_{1,i}$. The standard error of $\beta_1$ is the empirical standard deviation of $\beta_{1,i}$. The t-statistics can also be estimated as $(\sqrt{k}\hat{\beta_1})/s.e(\hat{\beta_1})$, which can be used to construct $p$-value. Here we set $k$ to be 50. We also do not standardize the estimates for all the estimation algorithms. The decision of whether to standardize latent estimates is inconsequential as the ratio between the slope and standard deviation is more important.

Table~\ref{table:correlation} shows the result. Except for MDS 5, all the models' forecast power is statistically significant. Our algorithms are again the best here, and are significantly better than the rest algorithms. 

\begin{table}
\caption {Explanatory power of the estimates of the latent variables against ground-truth of politicians' ideology scores. S.E. stands for standard error}
\centering
\small
\begin{tabular}{c|cccc}
\hline
Algo. & $\rho$ & Slope of $\beta$ & S.E. & p-value \\ \hline
Ours & \textbf{0.53} & 9.54 & 0.28 & $<$ 0.001\\ 
Modularity & 0.16 & 1.14 & 0.02 & $<$ 0.001\\ 
{\scriptsize Correspond Analysis} & 0.20 & 0.11 & 7e-4 & $<$ 0.001\\ 
Majority & 0.13 & 0.09 & 0.02 & $<$ 0.001\\ 
{\scriptsize Random Walk} & 0.01 & 1.92 & 0.65& $<$ 0.001\\ 
MDS 5 & 0.05 & 30.91 & 120.9 & 0.09 \\ 
MDS & 0.31 & 101.3 & 14.88 & $<$ 0.001 \\ \hline
\end{tabular}
\label{table:correlation}
\end{table}

\subsubsection{By State analyses}\label{sec:state}
Next, we aggregate users' latent variables by inferring their locations (see below) and grouping users by state. We sort the states by the mean of the latent scores of users in that state. If we assume that voters in the same state come from the same underlying distribution and different states have different means, then the empirical mean provides a min-variance unbiased estimate of the mean of the distribution of a state.

Around 10\% of the users' locations can be found and no particular state is over- or under-represented. See also Figure~\ref{fig:population2accts}. We construct the graph using two datasets: data prior to November 8, 2016, and all data until November 30. Since there is little difference,  we decide to use the full dataset, \ie all experiments use the same data.

\myparab{Samples from the population.} We observe that the correlation between the number of users in a state and the population in the state is 0.968, which is very high. Figure \ref{fig:population2accts} shows the corresponding scatter plot. This serves as a sanity check of our location extraction algorithm. We also observe that while Twitter users are biased samples of voters, no particular state is over- or under- represented.

\myparab{Result} Our goal is not to predict the election outcome since we know that Twitter users are biased samples of the voter population. On the other hand, we observe that the order statistics of the states' latent variables have the strongest or near-strongest explanatory power against the metrics below. See Figure~\ref{table:bystate}. 

\mypara{Binary outcome of 2016 US presidential election.} Similar to Section~\ref{sec:classificiation}, we find an optimal threshold that turns the ranking into a binary forecast and maximize the forecasting accuracy for both our algorithm and other baselines. 
 
\mypara{Order statistics vs. winning margins.} We can also turn the election results into scalars. For each state, we compute the ratio between the number of votes for Trump and the number of votes for Clinton and order the states according to the ratio. A state with a small ratio corresponds to a left-leaning state. Our goal is to understand the correlation between the order from our estimation and the order by the margin of winning in the election. 

\mypara{Long-term ideology.} We also compare our order statistics against the estimated liberal ideology in~\cite{stateIdeology}, which is based on the survey and sociodemographic data. This dataset can be considered as a long-term ideology score while the election result is a short-term one.

Table~\ref{table:bystate} shows the result. The following states are mis-classified in our algorithm: NV, NH, NJ, UT, and WI. Except for NJ, all other states are swing states. NJ has a Republican governor (Christie) who ran in the 2016 presidential campaign before dropping out. Except for the Kendall correlation for the election data, our algorithm continues to have the best performance. The winner (label propagation/majority) here also suffers from poor performance on other metrics. 

\begin{table}
 \caption{Quality of the order statistics by states. 
 Miss = misclassified state;
 E = correlation with election winning margin; L = represents correlation with long-term ideology. K = Kendall's $\tau$; S = Spearman's $\rho$. For example, KE refers to Kendall's correlation with election winning margin.}
 \label{table:bystate}
 \small
\centering
 \begin{tabular}{c|ccccc}
  \hline
  Algo. & Miss & KE & SE & KL & SL \\ \hline
  Ours & \textbf{5} & 0.613 & \textbf{0.816} & \textbf{0.636} & \textbf{0.800} \\ 
  Modularity & 8 & 0.462 & 0.649 & 0.468 & 0.641 \\ 
  {\scriptsize Correspond Analysis} & 21 & 0.410 & 0.567 & 0.374 & 0.547 \\ 
  Majority & 6 & {\textbf{0.619}}  & {0.808 } & {0.605 } & {0.772} \\ 
  {\scriptsize Random Walk} & 14 & 0.242 & 0.340 & 0.331 & 0.447 \\ 
  MDS 5 & 19 & 0.087 & 0.110 & 0.112 & 0.150 \\ 
  MDS & 21 & 0.470 & 0.632 & 0.387 & 0.550 \\ \hline
 \end{tabular}
\end{table}

\begin{figure}
\centering
\includegraphics[width=0.8\linewidth]{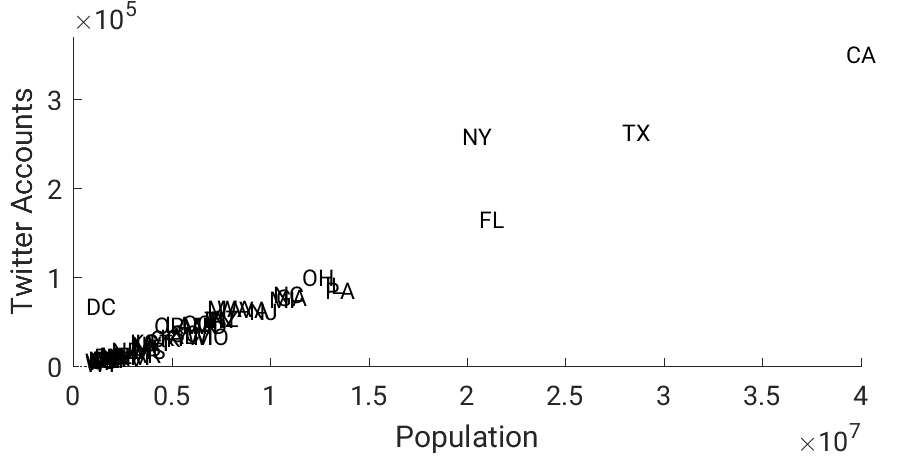}
\caption{Scatter plot between state population and sampled accounts per state. No particular state is over- or under-represented}
\label{fig:population2accts}
\end{figure}

\newpage
\section{Summary of notations} 

\label{app:not}

\begin{itemize}
\item $A \in R^{n \times n}$: the adjacent matrix of an undirected graph. In the simplified graph model, $A$ is the input graph. In the 
bipartite graph model, $A$ is the regularized matrix over $B^{\transpose}B$. 
\item $B \in R^{m \times n}$: the bipartite graph matrix, \ie $B_{i, j} = 1$ if and only if follower $i$ is connected to influencer $j$. 
\item $\Ball(z, r) = \{z': \|z'- z\| \leq r\}$
\item $C(n)$: the normalization constant in the undirected graph model, \ie $\Pr[\{x_i, x_j\} \in E] = \kappa(x_i, x_j)/C(n)$.
\item $\calc$: a curve in $R^{N_{\calh}}$, defined as $\calc = \{\Phi(x)\}_{x \in [0, 1]}$
\item $\cald$: the distribution in which $x_i$ and $y_I$ come from. 
\item $D$: the distance estimate, \eg $D_{i,j}$ is the estimate of distance between $x_i$ and $x_j$ in our algorithm. 
\item $d$: the number of eigenvalues to keep; in the section for isomap technique, it sometimes is used to refer to the length of the shortest path. 
\item $\calh$: the reproducing kernel Hilbert space of $\kappa$. 
\item $K \in R^{n \times n}$: the kernel/Gram matrix associated with $\kappa$, \ie $K_{i, j} = \kappa(x_i, x_j)$. 
\item $\calk$: an integral operator defined as $\calk f(x) = \int \kappa(x, x') f(x') dF(x')$. 
\item $F(n)$: the cdf of $\cald$.
\item $\calm$: an integral operator defined as $\calm f(x) = \int \mu(x, y)f(y)dF(y)$.
\item $M \in R^{n \times n}$: the kernel/Gram matrix associated with $\mu$, \ie $M_{i, j} = \mu(x_i, x_j)$. 
\item $m$: the number of followers in the bipartite graph model. 
\item $N_{\calh}$: the number of eigenvalues in $\calk$, which could be countably infinite. 
\item $n$: the number of nodes in the simplified model and the number of influencers in the bipartite grpah model.
\item $\calp$: projection operators. 
\item $[\tilde U_X, \tilde S_X, \tilde V_S]$ where $X \in \{A, K, M\}$: the SVD of $X$. 
\item $[U_X, S_X, V_S]$, where $X \in \{A, K, M\}$: the first $d$ singular vectors/values of $X$. Note here they implicitly depend on $d$.  
\item $\bfx = \{x_1, ..., x_n\}$: the set of nodes in the simplified model and the set of influencers in the bipartite graph model. 
\item $\hat x_i$: our algorithm's estimate of $x_i$. 
\item $\bfy = \{y_1, ..., y_n\}$: the set of followers in the bipartite graph model.
\item $\hat y_i$: our algorithm's estimate of $y_i$.
\item $z_i$: the feature of $x_i$, \ie $\Phi(z_i)$. 
\item $\hat z_i$: our algorithm's estimate of $z_i$. 
\item $\delta_i$: the eigengap, defined as $\delta_i = \lambda_i - \lambda_{i + 1}$. 
\item $\lambda_i$: the eigenvalues of $\calk$ unless otherwise specified. 
\item $\rho(n)$: we also re-parametrize $C(n) = n/\rho(n)$, \ie $\rho(n) = n/C(n)$.
\item $\kappa: [0, 1]\times[0,1] \rightarrow (0, 1]$: the kernel function.
\item $\Delta$: the exponent in the small-world kernel, \ie $\kappa(x_i, x_j) = c_0/(|x_i - x_j|^{\Delta} + c_0$. 
\item $\Phi: [0, 1]\rightarrow R^{N_{\calh}}$ the feature map associated with $\calk$.
\item $\hat \Phi$: our algorithm's estimate of $\Phi$.
\item $\Phi^{\calm}$: the feature map associated with $\calm$.
\item $\Phi_d$: the first $d$ coordinates of the feature map. It is also overloaded to be in $R^{N_{\calh}}$ by padding $0$'s after the $d$-th coordinate.  
\item $\hat \Phi_d$: our algorithm's estimate of $\Phi_d$. 
\item $\Phi^{\calm}_d$: the first $d$ coordinates of $\Phi^{\calm}$.
\item $\psi_i$: the $i$-th eigenfunction of $\calk$. 
\item $\mu$: a kernel used for the analysis of the bipartite graph model, \ie $\mu(x, x') = \int \kappa(x, z)\kappa(z, x')dF(z)$ (see Appendix~\ref{sec:phiest}).
\item Norms: $\|A\|$ is the opereator/spectral norm of $A$. $\|A\|_F$ is the Frobenius norm. $\|A\|_{\hsn}$ is the Hilbert-Schmidt norm (see Definition~\ref{def:hsn}).  
\end{itemize}

\end{document}